\documentclass[runningheads]{llncs}
\usepackage{graphicx}

\usepackage{tikz}
\usepackage{comment}
\usepackage{amsmath,amssymb} %
\usepackage{bbm}  %
\usepackage{color}
\usepackage{caption}
\usepackage{subcaption}
\usepackage{calc}
\usepackage{colortbl}

\usepackage{overpic}

\usepackage[font=small]{caption} %

\usepackage{cuted} %

\usepackage{array} %
\newcolumntype{x}[1]{>{\centering\let\newline\\\arraybackslash\hspace{0pt}}p{#1}}
\usepackage{makecell,booktabs,multirow}

\newcommand{\sref}[1]{(\ref{#1})}
\newcommand{\refmainpaper}[2]{\autoref{#2}}

\newcommand{\eg}{\textit{e.g.},\ }

\newcommand{\R}{\ensuremath{\mathbb{R}}}

\newcommand{\T}{^\top}
\newcommand*\dif{\mathop{}\!\mathrm{d}}

\newcommand{\gauss}{\mathcal{N}}

\newcommand{\man}{\mathcal{M}}

\newcommand{\ef}{\phi}
\newcommand{\ev}{\lambda}

\newcommand{\inner}[1]{\langle #1 \rangle}

\newcommand{\fat}[1]{\mathbf{#1}}
\newcommand{\pp}{\fat{p}}
\newcommand{\qq}{\fat{q}}

\newcommand{\xx}{\fat{x}}
\newcommand{\yy}{\fat{y}}
\newcommand{\dd}{\fat{d}}
\newcommand{\uu}{\fat{u}}

\DeclareMathOperator{\hntk}{h_{\text{NTK}}}
\DeclareMathOperator{\kntk}{k_{\text{NTK}}}
\DeclareMathOperator{\hg}{h_{\gamma}}

\DeclareMathOperator{\expec}{\mathbb{E}}

\newcommand{\sphere}{\mathbb{S}}
\newcommand{\shspace}[2]{\mathbf{H}^{#1}_{#2}}

\newsavebox\CBox
\def\textBF#1{\sbox\CBox{#1}\resizebox{\wd\CBox}{\ht\CBox}{\textbf{#1}}}

\makeatletter
\newcommand\notsotiny{\@setfontsize\notsotiny\@vipt\@viipt}
\makeatother

\newcommand{\pseudoparagraph}[1]{\paragraph{#1. }}

\definecolor{tscolor}{RGB}{117,112,179}

\definecolor{TableColor}{RGB}{204,229,255}
\newcolumntype{a}{>{\columncolor{TableColor}}c}

\DeclareMathOperator{\nn}{\mathcal{F}}
\DeclareMathOperator{\nnT}{\nn_{\theta}}
\DeclareMathOperator{\nnE}{\nn_{\theta}^E}

\newcommand*\Laplace{\mathop{}\!\mathbin\bigtriangleup}

\DeclareMathOperator{\comp}{\circ}

\usepackage[breaklinks,colorlinks]{hyperref}

\begin{document}

\title{Intrinsic Neural Fields: Learning Functions\texorpdfstring{\\}{ }on Manifolds}

\titlerunning{Intrinsic Neural Fields: Learning Functions on Manifolds}
\author{Lukas Koestler \inst{1\ast} \and
Daniel Grittner\inst{1\ast} \and
Michael Moeller\inst{2} \and
Daniel Cremers \inst{1} \and
Zorah L\"ahner\inst{2}}
\authorrunning{L.~Koestler, D.~Grittner et al.}
\institute{Technical University of Munich\\
\email{\{lukas.koestler,daniel.grittner,cremers\}@tum.de} \and
University of Siegen\\
\email{\{michael.moeller,zorah.laehner\}@uni-siegen.de}\\
$^\ast$ Equal contribution}
\maketitle

\begin{abstract}
Neural fields have gained significant attention in the computer vision community due to their excellent performance in novel view synthesis, geometry reconstruction, and generative modeling. Some of their advantages are a sound theoretic foundation and an easy implementation in current deep learning frameworks. 
While neural fields have been applied to signals on manifolds, \eg for texture reconstruction, their representation has been limited to extrinsically embedding the shape into Euclidean space. The extrinsic embedding ignores known intrinsic manifold properties and is inflexible wrt. transfer of the learned function.
To overcome these limitations, this work introduces intrinsic neural fields, a novel and versatile representation for neural fields on manifolds. Intrinsic neural fields combine the advantages of neural fields with the spectral properties of the Laplace-Beltrami operator.
We show theoretically that intrinsic neural fields inherit many desirable properties of the extrinsic neural field framework but exhibit additional intrinsic qualities, like isometry invariance. 
In experiments, we show intrinsic neural fields can reconstruct high-fidelity textures from images with state-of-the-art quality and are robust to the discretization of the underlying manifold.
We demonstrate the versatility of intrinsic neural fields by tackling various applications: texture transfer between deformed shapes \& different shapes, texture reconstruction from real-world images with view dependence, and discretization-agnostic learning on meshes and point clouds.
\end{abstract}

\begin{figure}
    \centering
    \begin{overpic}[trim={0cm, 1cm, 0cm, 2cm},clip,width=\linewidth]{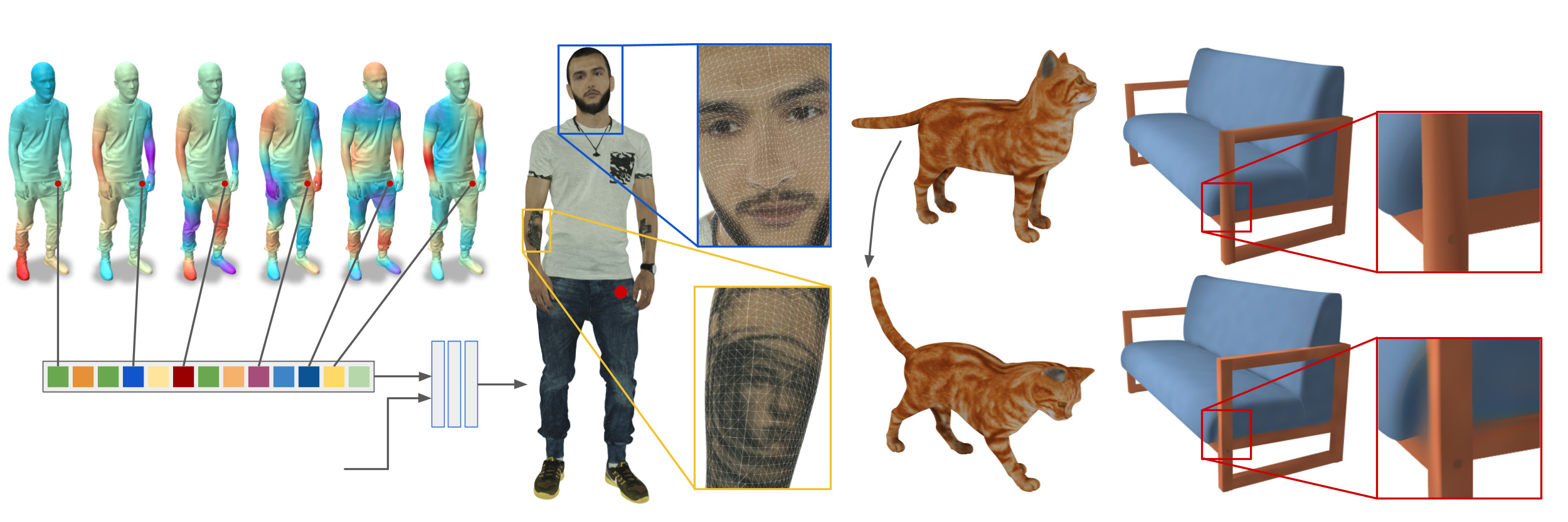}
    \put(6,5){\tiny{embedding $\gamma(\pp)$}}
    \put(3,2){\tiny{optional parameters,}}
    \put(3,0){\tiny{like view direction}}
    \put(27.5,2.5){\scriptsize{$f_\theta$}}
    \put(56.5,15){\tiny{texture transfer}}
    \put(56.5,13.5){\tiny{without retraining}}
    \put(2,30){\tiny{Laplace-Beltrami eigenfunctions }}
    \put(4.5,20){\scriptsize{$\pp$}}
    \put(4,27){\tiny{$\phi_0$}}
    \put(9,27){\tiny{$\phi_3$}}
    \put(14.5,27){\tiny{$\phi_5$}}
    \put(20,27){\tiny{$\phi_8$}}
    \put(25,27){\tiny{$\phi_{10}$}}
    \put(30,27){\tiny{$\phi_{11}$}}
    \put(85,29){\tiny{Ours}}
    \put(75,-1.5){\tiny{Random Fourier Features}}
    \put(0,28){\scriptsize{a)}}
    \put(54,28){\scriptsize{b)}}
    \put(71,28){\scriptsize{c)}}
    \put(35.5,15){\tiny{$\mathcal{F}_\theta(\pp)$}}
    \end{overpic}
    \caption{(a) Overview of our method. We use the eigenfunctions $\ef_i$ of the Laplace-Beltrami operator (LBO) at each point as a point embedding $\gamma(\pp)$. This overcomes the spectral bias of the multilayer perceptron (MLP) $f_{\theta}$, and hence the combined \emph{intrinsic neural field} $\nnT$ can represent a high-frequency function on the surface. Notice that $\pp$ can be inside a triangle, and the function is clearly more detailed than the discretization (\textit{insets}). (b) An intrinsic neural texture field trained on one shape (\textit{top}) can be transferred to a new shape (\textit{bottom}) without retraining. (c) Due to our intrinsic approach (LBO eigenfunctions) local geometry is maintained in close but separate parts, whereas an extrinsic approach (Random Fourier Features~\cite{DBLP:conf/nips/TancikSMFRSRBN20}) shows bleeding artifacts when trained with sparse supervision.}
    \label{fig:teaser}
\end{figure}
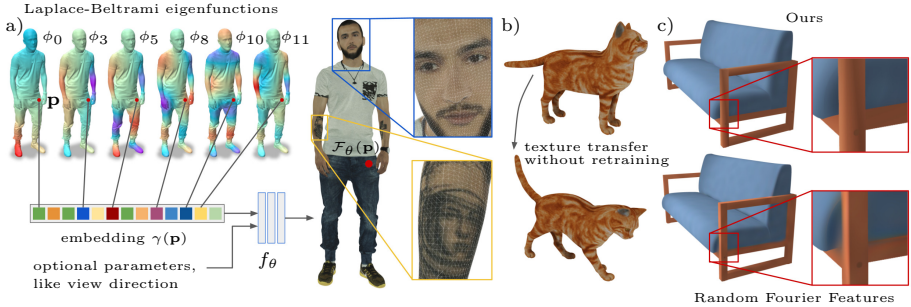

\section{Introduction}

Neural fields have grown incredibly popular for novel view synthesis since the breakthrough work by Mildenhall et al.~\cite{DBLP:conf/eccv/MildenhallSTBRN20}.
They showed that neural radiance fields together with differentiable volume rendering can be used to reconstruct scenes and often yield photorealistic renderings from novel viewpoints.
This inspired work in related fields, \eg human shape modeling \cite{palafox2021npm}, shape and texture generation from text \cite{DBLP:journals/corr/abs-2112-03221}, and texture representation on shapes \cite{DBLP:conf/iccv/OechsleMNSG19,DBLP:conf/rt/BaatzGPRN21}, where neural fields are able to generate a wide variety of functions with high fidelity.

These methods use neural fields as functions from a point in Euclidean space to the quantity of interest. 
While this is valid for many applications, for others, the output actually lives on a general manifold. 
For example, texture mappings define a high-frequency color function on the surface of a 3D object. 
TextureFields~\cite{DBLP:conf/iccv/OechsleMNSG19} and Text2Mesh~\cite{DBLP:journals/corr/abs-2112-03221} solve this discrepancy by defining a mapping of each surface point to its Euclidean embedding and then learning the neural field there.
Both show that they can achieve detail preservation above the discretization level, but the detour to Euclidean space has drawbacks.
The Euclidean and geodesic distance between points can differ significantly. 
This is important on intricate shapes with fine geometric details that overlap because the local geometry prior is lost.
Further, extrinsic representations cannot be used in the presence of surface deformations without retraining or applying heuristics.

Similar challenges have been solved in geometry processing by using purely intrinsic representations, most famously properties derived from the Laplace-Beltrami operator (LBO).
Some of the main advantages of the LBO are its invariance under rigid and isometric deformations and reparametrization. 
We follow this direction by defining intrinsic neural fields on manifolds independent of the extrinsic Euclidean embedding and thus inherit the favorable properties of intrinsic representations. This is enabled by the fact that random Fourier features~\cite{DBLP:conf/nips/TancikSMFRSRBN20}, an embedding technique that enabled the recent success of Euclidean neural fields, have an intrinsic analog based on the LBO.
The result is a fully differentiable method that can learn high-frequency information on any 3D geometry representation that admits the computation of the LBO.
A schematic overview of our method can be found in Figure~\ref{fig:teaser}.
Our main theoretical and experimental \textbf{contributions} are:

\begin{itemize}
    \item We introduce \textbf{intrinsic neural fields}, a novel and versatile representation for neural fields on manifolds. Intrinsic neural fields combine the advantages of neural fields with the spectral properties of the Laplace-Beltrami operator.
    \item We extend the \textbf{neural tangent kernel analysis} of \cite{DBLP:conf/nips/TancikSMFRSRBN20} to the manifold setting. This yields a proof characterizing the stationarity of the kernel induced by intrinsic neural fields and insight into their spectral properties.
    \item We show that intrinsic neural fields can \textbf{reconstruct high-fidelity textures} from images with state-of-the-art quality. 
    \item We demonstrate the versatility of intrinsic neural fields by tackling \textbf{various applications}: texture transfer between isometric and non-isometric shapes, texture reconstruction from real-world images with view dependence, and discretization-agnostic learning on meshes and point clouds.
\end{itemize}

\noindent We will release the full source code for all experiments and the associated data along with the final publication.

This work studies how a \emph{neural field} can be defined on a manifold. Current approaches use the \emph{extrinsic Euclidean embedding} and define the neural field on the manifold as a Euclidean neural field in the extrinsic embedding space -- we describe this approach in \autoref{sec:euclidean_neural_fields}. In contrast, our approach uses the well-known Laplace-Beltrami Operator (LBO), which we briefly introduce in \autoref{sec:lbo}. The final definition of \emph{intrinsic neural fields} is given in \autoref{sec:intrinsic_neural_fields}. The experimental results are presented in \autoref{sec:applications}.

\section{Related Work}
This work investigates neural fields for learning on manifolds, and we will only consider directly related work in this section. We point interested readers to the following overview articles: 
neural fields in visual computing \cite{xie2021neuralfield}, advances in neural rendering \cite{DBLP:journals/corr/abs-2111-05849}, and an introduction into spectral shape processing \cite{spectralgeometrytut}.

\pseudoparagraph{Neural Fields}
While representing 3D objects and scenes with coordinate-based neural networks, or neural fields, has already been studied more than two decades ago \cite{DBLP:conf/rt/GarganN98,DBLP:conf/3dim/PiperakisK01,DBLP:conf/graphite/PengS04}, the topic has gained increased interest following the breakthrough work by Mildenhall et al.~\cite{DBLP:conf/eccv/MildenhallSTBRN20}. They show that a Neural Radiance Field (NeRF) %
often yields photorealistic renderings from novel viewpoints. One key technique underlying this success is positional encoding, which transforms the three-dimensional input coordinates into a higher dimensional space using sines and cosines with varying frequencies. This encoding overcomes the low-frequency bias of neural networks \cite{DBLP:conf/icml/RahamanBADLHBC19,DBLP:conf/icml/BasriGGJKK20} and thus enables high-fidelity reconstructions. The aforementioned phenomenon is analyzed using the neural tangent kernel~\cite{DBLP:conf/nips/JacotHG18} by Tancik et al.~\cite{DBLP:conf/nips/TancikSMFRSRBN20}, and our analysis extends theirs from Euclidean space to manifolds. Simultaneously to Tancik et al., Sitzmann et al.~\cite{DBLP:conf/nips/SitzmannMBLW20} use periodic activation functions for neural scene representation, which is similar to the above-mentioned positional encoding \cite{DBLP:conf/wacv/BenbarkaHRZ22}.
Additionally, many other works \cite{DBLP:journals/corr/abs-2112-01917,DBLP:journals/corr/abs-2112-11577,DBLP:journals/corr/abs-2111-15135,DBLP:journals/corr/abs-2110-13572,DBLP:journals/corr/abs-2111-08918,hertz2021sape,DBLP:journals/corr/abs-2107-02561,DBLP:conf/ijcai/WangLYT21,DBLP:journals/corr/abs-2104-03960} offer insights into neural fields and their embedding functions. However, none of these works considers neural fields on manifolds.

\pseudoparagraph{Neural Fields on Manifolds}
Prior works \cite{DBLP:conf/iccv/OechsleMNSG19,DBLP:conf/eccv/ChibaneP20,DBLP:conf/rt/BaatzGPRN21,DBLP:conf/cvpr/XiangXHHS021,palafox2021npm,DBLP:conf/cvpr/MorrealeAKM21,yifan2022geometryconsistent,DBLP:journals/corr/abs-2112-03221,DBLP:journals/corr/abs-2110-05433} use neural fields to represent a wide variety of quantities on manifolds.
Oechsle et al.~\cite{DBLP:conf/iccv/OechsleMNSG19} use the extrinsic embedding of the manifold to learn textures as multilayer perceptrons. Their Texture Fields serve as an important baseline for this work.
NeuTex by Xiang et al.~\cite{DBLP:conf/cvpr/XiangXHHS021} combines neural, volumetric scene representations with a 2D texture network to facilitate interpretable and editable texture learning. To enable this disentanglement, their method uses mapping networks from the 3D space of the object to the 2D space of the texture and back. We compare with an adapted version of their method that utilizes the known geometry of the object.
Baatz et al.~\cite{DBLP:conf/rt/BaatzGPRN21} introduce NeRF-Tex, a combination of neural radiance fields (NeRFs) and classical texture maps. Their method uses multiple small-scale NeRFs to cover the surface of a shape and represent mesoscale structures, such as fur, fabric, and grass. Because their method focuses on mesoscale structures and artistic editing, we believe that extending the current work to their setting is an interesting direction for future research.

Additionally, neural fields have been used to represent quantities other than texture on manifolds.
Palafox et al.~\cite{palafox2021npm} define a neural deformation field that maps points on a canonical shape to their location after the shape is deformed. This model is applied to generate neural parametric models, which can be used similarly to traditional parametric models like SMPL~\cite{DBLP:journals/tog/LoperM0PB15}.
Yifan et al.~\cite{yifan2022geometryconsistent} decompose a neural signed distance function (SDF) into a coarse SDF and a high-frequency implicit displacement field.
Morreale et al.~\cite{DBLP:conf/cvpr/MorrealeAKM21} define neural surface maps, which can be used to define surface-to-surface correspondences among other applications.
Text2Mesh~\cite{DBLP:journals/corr/abs-2112-03221} uses a coarse mesh and a textual description to generate a detailed mesh and associated texture as neural fields.

\pseudoparagraph{Intrinsic Geometry Processing}
Intrinsic properties are a popular tool in geometry processing, especially in the analysis of deformable objects.
Two of the most basic intrinsic features are Gauss curvature and intrinsic point descriptors based on the Laplace-Beltrami operator (LBO). They have been heavily used since the introduction of the global point signature \cite{rustamov2007laplace} and refined since then \cite{sun09hks,aubry11wks}.
Intrinsic properties are not derived from a manifold's embedding into its embedding space but instead arise from the pairwise geodesic distance on the surface. 
These are directly related to natural kernel functions on manifolds, \eg shown by the efficient approximation of the geodesic distance from the heat kernel \cite{crane17geoinheat}. 
Kernel functions as a measure of similarity between points are very popular in geometry processing. They have been used in various applications, \eg in shape matching \cite{vestner17kernel,burghard17embedding,liu08kernel}, parallel transport \cite{sharp19vectorheat}, and 
robustness wrt.\ discretization \cite{vaxman10heat,sharp2021diffusion}.
Manifold kernels naturally consider the local and global geometry \cite{boscaini16anisotropic}, and our approach follows in this direction by showing a natural extension of neural fields on manifolds.

\section{Background}
Differential geometry offers two viewpoints onto manifolds: intrinsic and extrinsic. 
The extrinsic viewpoint studies the manifold $\man$ through its \emph{Euclidean embedding}
where each point $\pp \in \man$ is associated with its corresponding point in Euclidean space.
In contrast, the intrinsic viewpoint considers only properties of points independent of the extrinsic embedding, such as, the geodesic distance between a point pair.
Both can have advantages depending on the method and application. 
An intrinsic viewpoint is by design invariant against certain deformations in the Euclidean embedding, like rigid transformations but also pose variations that are hard to characterize in the extrinsic view.

\subsection{Neural Fields for Euclidean Space} \label{sec:euclidean_neural_fields}
A Euclidean neural field $\nnE : \R^m \to \R^o$ is a neural network that maps points in Euclidean space to vectors and is parametrized by weights $\theta \in \R^p$. The network is commonly chosen to be a multilayer perceptron (MLP).
Let $\man \subset \R^m$ be a manifold with a Euclidean embedding into $\R^m$.
Naturally, the restriction of $\nnE$ to $\man$ leads to a neural field on a manifold: 
$\nnT: \man \to \R^o \,, \nnT(x) = \nnE(x) \,$.

Natural signals, such as images and scenes, are usually quite complex and contain high-frequency variations. Due to spectral bias, standard neural fields fail to learn high-frequency functions from low dimensional data \cite{DBLP:conf/nips/TancikSMFRSRBN20,DBLP:conf/nips/SitzmannMBLW20} and generate blurry reconstructions. With the help of the neural tangent kernel, it was proven that the composition $\nnE \comp \gamma$ of a higher dimensional Euclidean neural field and a random Fourier feature (RFF) encoding $\gamma$ helps to overcome the spectral bias and, consequently, enables the neural field to better represent high-frequency signals. The RFF encoding $\gamma: \mathbb{R}^m \rightarrow \mathbb{R}^d$ with $d \gg m$ is defined as 
\begin{equation} \label{eqn:general_fourier_feature_encoding}
    \gamma(\xx) = [a_1\cos(\fat{b}_1^\top \xx),\, a_1\sin(\fat{b}_1^\top \xx), \dots,\, a_{d/2}\cos(\fat{b}_{d/2}^\top \xx), a_{d/2}\sin(\fat{b}_{d/2}^\top \xx)],
\end{equation}
where the coefficients $\fat{b}_i \in \R^m$ are randomly drawn from the multivariate normal distribution $\gauss(\fat{0}, (2 \pi \sigma)^2 \fat{I})$. The factors $a_i$ are often set to one for all $i$ and $\sigma > 0$ is a hyperparameter that offers a trade-off between reconstruction fidelity and overfitting of the training data. 

\subsection{The Laplace-Beltrami Operator} \label{sec:lbo}

In the following, we briefly introduce the Laplace-Beltrami operator (LBO) and refer the interested reader to \cite{rustamov2007laplace} for more details.
The LBO $\Laplace_\man$ is the generalization of the Euclidean Laplace operator on general closed compact manifolds. 
Its eigenfunctions $\ef_i: \man \to \R$ and eigenvalues $\ev_i \in \R$ are the non-trivial solutions of the equation
$\Laplace_\man \ef_i = \ev_i \ef_i \,$.
The eigenvalues are non-negative and induce a natural ordering which we will use for the rest of the paper.
The eigenfunctions are orthonormal to each other, build an optimal basis for the space of square-integrable functions \cite{parlett98eigenvalue}, and are frequency ordered allowing a low-pass filtering by projecting onto the first $k$ eigenfunctions.
Hence, a function $f: \man \to \R \in L^2(\man)$ can be expanded in this basis:
\begin{equation} \label{eqn:laplace_beltrami_expansion}
    f = \sum_{i=0}^\infty c_i \ef_i = \sum_{i=0}^\infty \inner{f, \ef_i} \ef_i \approx \sum_{i=0}^k \inner{f, \ef_i} \ef_i\,,
\end{equation}
where the quality of the last $\approx$ depends on the amount of high-frequency information in $f$.
The projection onto the LBO basis is similar to the Fourier transform, allowing the same operations, and thus we use the LBO basis as the replacement for Fourier features. %
In fact, if $[0,1]^2$ is considered as a manifold, its LBO eigenfunctions with different boundary conditions are exactly combinations of sines and cosines.
Furthermore, the eigenfunctions of the LBO are identical up to sign ambiguity for isometric shapes since the LBO is entirely intrinsic.

\section{Intrinsic Neural Fields}\label{sec:intrinsic_neural_fields}
We introduce \emph{Intrinsic Neural Fields} based on the eigenfunctions of the Laplace-Beltrami operator (LBO) which can represent detailed surface information, like texture, directly on the manifold.
In the presence of prior geometric surface information, it is more efficient than using the extrinsic Euclidean embedding space which is often mainly empty. 
Additionally, this representation is naturally translation and rotation invariant, surface representation invariant, as well as invariant to isometric deformations.

\begin{definition}[Intrinsic Neural Field]\label{def:intrinsicneuralfield}
Let $\man \subset \R^m$ be a closed, compact manifold and $\ef_1, \dots, \ef_d$ be the first $d$ Laplace-Beltrami eigenfunctions of $\man$. 
We define an \textbf{intrinsic neural field} $\nnT : \man \to \R^o$ as
\begin{equation} \label{eqn:intrinsic_neural_fields}
    \nnT(\pp) = (f_\theta \comp \gamma)(\pp) = f_\theta(a_1 \ef_1(\pp), \dots, a_d \ef_d(\pp)) \,.
\end{equation} 
where $\gamma: \man \to \R^d, \gamma(\pp) = (a_1 \ef_1(\pp), \dots, a_d \ef_d(\pp))$, with $a_i \geq 0$ and $\lambda_i = \lambda_j \Rightarrow a_i = a_j$, is our embedding function and $f_\theta : \R^d \to \R^o$ represents a neural network with weights $\theta \in \R^p$.
\end{definition}

Within this work, we will use $a_i = 1$, which has proven sufficient in praxis, and multilayer perceptrons (MLPs) for $f_\theta$, as this architectural choice is common for Euclidean neural fields \cite{xie2021neuralfield}. 
A detailed description of the architecture can be found in \autoref{fig:intrinsic_network_arch}.
It is possible to choose different embedding functions $\gamma$ but we choose the LBO eigenfunctions as they have nice theoretical properties (see Section~\ref{sec:stationary}) and are directly related to Fourier features.

In \autoref{fig:theory_1d}, we apply intrinsic neural fields to the task of signal reconstruction on a 1D manifold to give an intuition about how it works and what its advantages are. 
The results show that the neural tangent kernel (NTK) for intrinsic neural fields exhibits favorable properties, which we prove in \autoref{sec:stationary}.
We show that we can represent high-frequency signals on manifold surfaces that go far beyond the discretization level. 
In \autoref{sec:applications}, we apply the proposed intrinsic neural fields to a variety of tasks including texture reconstruction from images, texture transfer between shapes without retraining, and view-dependent appearance modeling. 

\begin{figure}
\centering
\begin{subfigure}[b]{0.2057142857142857\textwidth}
\centering
\includegraphics[width=\textwidth]{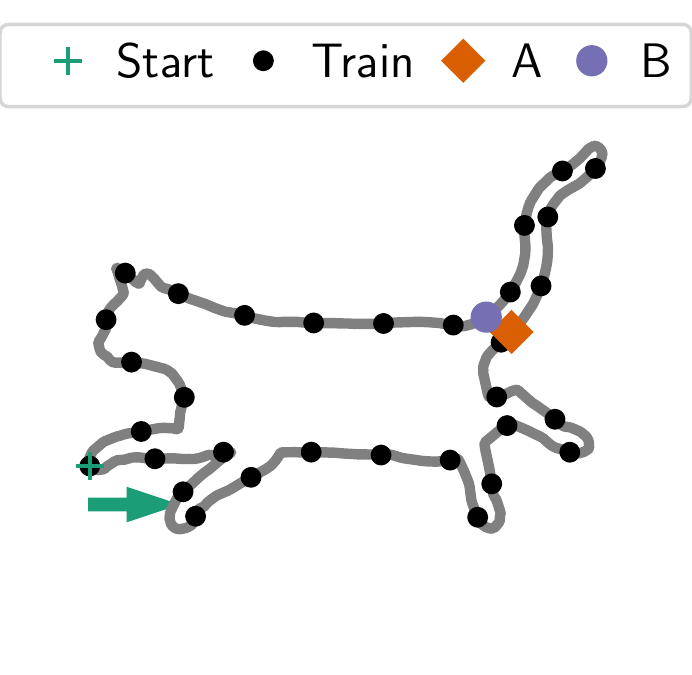}
\caption{Manifold}
\label{fig:theory_1d_sketch}
\end{subfigure}
\hfill
\begin{subfigure}[b]{0.3771428571428571\textwidth}
\centering
\includegraphics[width=\textwidth]{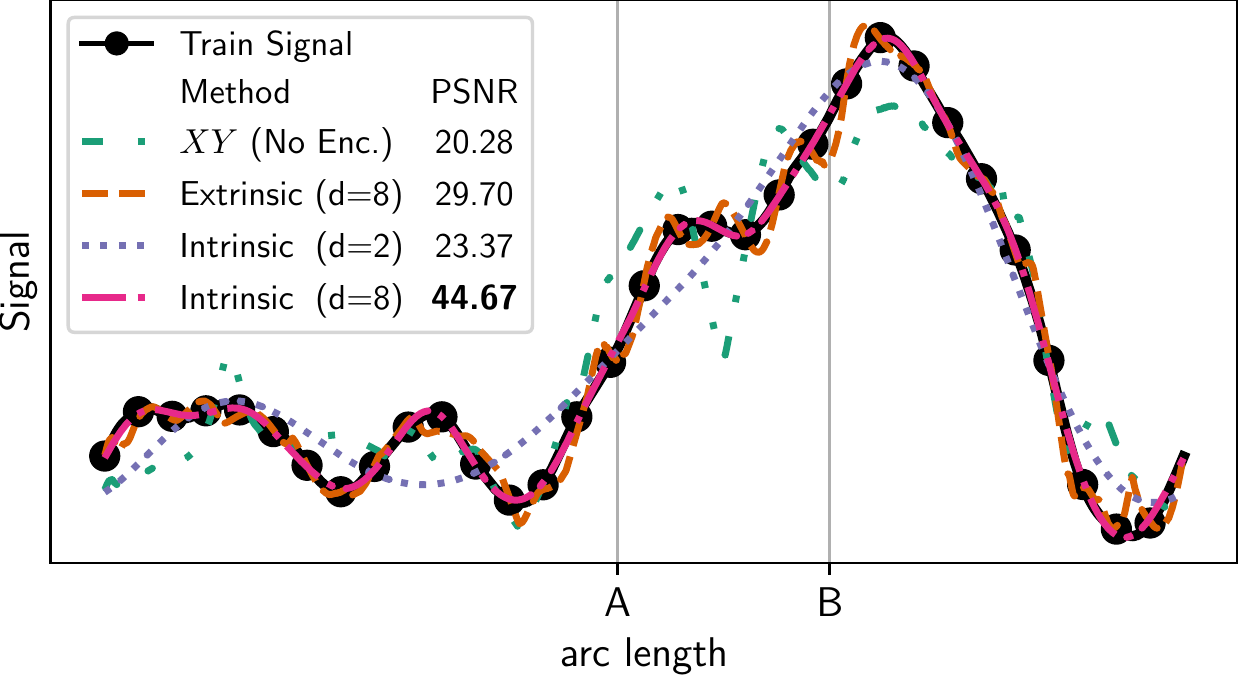}
\caption{Reconstructed Signals}
\label{fig:theory_1d_signal}
\end{subfigure}
\hfill
\begin{subfigure}[b]{0.3771428571428571\textwidth}
\centering
\includegraphics[width=\textwidth]{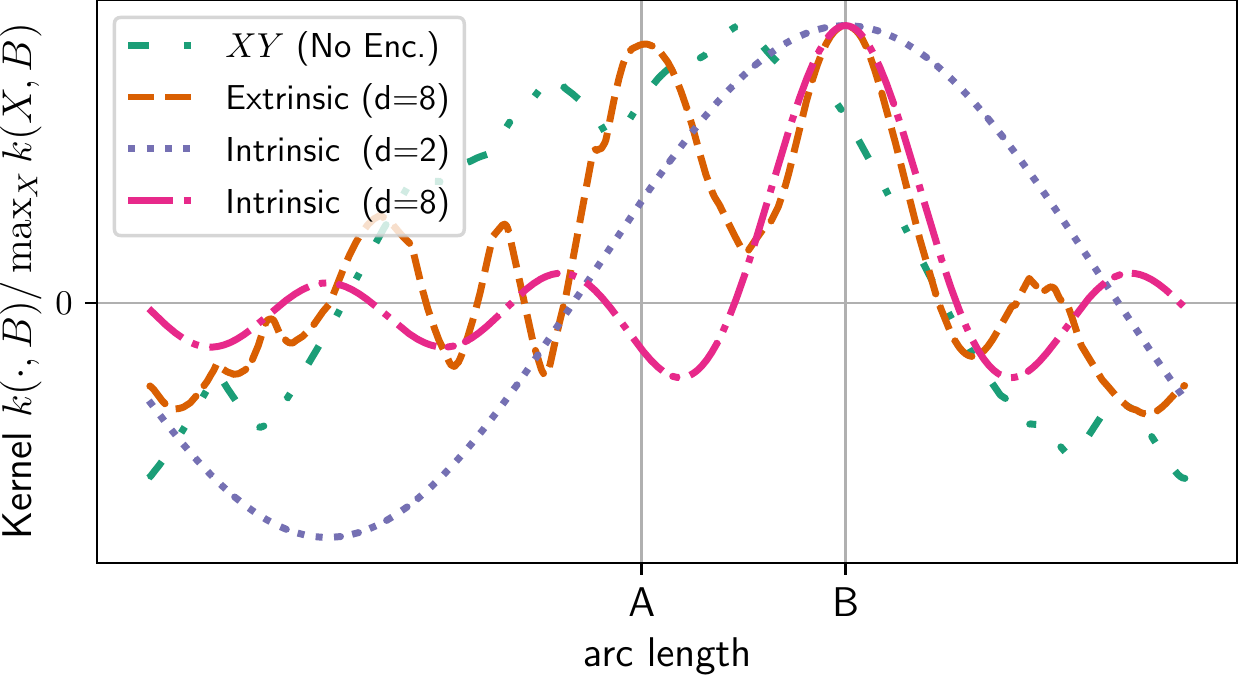}
\caption{Normalized Kernels around B}
\label{fig:theory_1d_kernel}
\end{subfigure}
\\[1.1ex]
\begin{subfigure}[b]{0.22\textwidth}
\centering
\includegraphics[width=\textwidth]{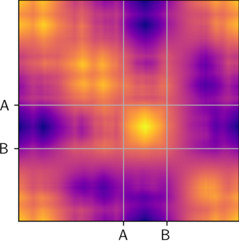}
\caption{\vspace{-0.2cm} XY}
\label{fig:cat_2d_kernels_xy}
\end{subfigure}
\hfill
\begin{subfigure}[b]{0.22\textwidth}
\centering
\includegraphics[width=\textwidth]{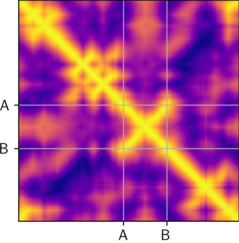}
\caption{\vspace{-0.2cm}RFF ($\sigma{=}\frac{1}{2},d{=}8$)}
\label{fig:cat_2d_kernels_ff}
\end{subfigure}
\hfill
\begin{subfigure}[b]{0.22\textwidth}
\centering
\includegraphics[width=\textwidth]{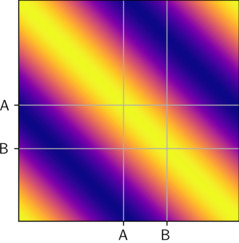}
\caption{\vspace{-0.2cm} Ours ($d{=}2$)}
\label{fig:cat_2d_kernels_intr_d2}
\end{subfigure}
\hfill
\begin{subfigure}[b]{0.22\textwidth}
\centering
\includegraphics[width=\textwidth]{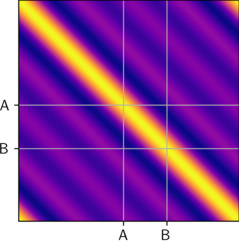}
\caption{\vspace{-0.2cm} Ours ($d{=}8$)}
\label{fig:cat_2d_kernels_intr_d8}
\end{subfigure}
\caption{Signal reconstruction. \sref{fig:theory_1d_signal} The target is sampled at $32$ points and MLPs with three layers, $1024$ channels, and different embeddings are trained using $L^2$ loss.
The intrinsic neural field with $d{=}8$ eigenfunctions performs best. Using only two eigenfunctions leads to oversmoothing. The reconstruction with the extrinsic embedding and random Fourier features (RFF) \cite{DBLP:conf/nips/TancikSMFRSRBN20} can capture the high-frequency details, but introduces artifacts when the Euclidean distance is not a good approximation of the geodesic distance, for example, at points A \& B.
(\ref{fig:cat_2d_kernels_xy}-\ref{fig:cat_2d_kernels_intr_d8}) The second row of subfigures shows the pairwise neural tangent kernel (NTK)~\cite{DBLP:conf/nips/JacotHG18,neuraltangents2020} between all points on the manifold.
\sref{fig:cat_2d_kernels_xy} The NTK using the extrinsic Euclidean embedding is not maximal along the diagonal. 
\sref{fig:cat_2d_kernels_ff} For the NTK with RFF embedding the maximum is at the diagonal because each point's influence is maximal onto itself. However, it has many spurious correlations between points that are close in Euclidean space but not along the manifold, for example, around B. (\ref{fig:cat_2d_kernels_intr_d2},\ref{fig:cat_2d_kernels_intr_d8}) The NTK with our intrinsic embedding is localized correctly and is stationary (c.f.~\autoref{thm:stationarity}), which makes it most suitable for interpolation.
}
\label{fig:theory_1d}
\end{figure}

\begin{figure}
\centering
\begin{subfigure}[b]{0.19\textwidth}
\centering
\includegraphics[width=\textwidth]{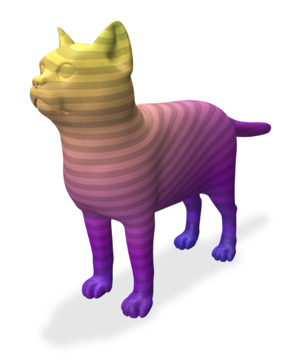}
\caption{XYZ: $S_1$}
\label{fig:cat_3d_kernels_xyz}
\end{subfigure}
\hfill
\begin{subfigure}[b]{0.19\textwidth}
\centering
\includegraphics[width=\textwidth]{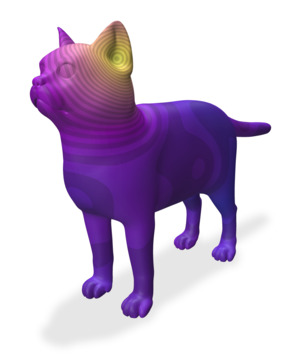}
\caption{RFF: $S_1$}
\label{fig:cat_3d_kernels_ff}
\end{subfigure}
\hfill
\begin{subfigure}[b]{0.19\textwidth}
\centering
\includegraphics[width=\textwidth]{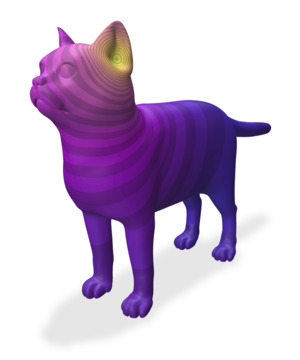}
\caption{Ours: $S_1$}
\label{fig:cat_3d_kernels_intr_A}
\end{subfigure}
\hfill
\begin{subfigure}[b]{0.19\textwidth}
\centering
\includegraphics[width=\textwidth]{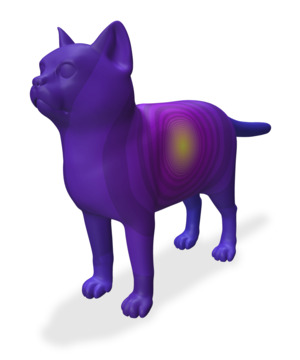}
\caption{Ours: $S_2$}
\label{fig:cat_3d_kernels_intr_B}
\end{subfigure}
\hfill
\begin{subfigure}[b]{0.19\textwidth}
\centering
\includegraphics[width=\textwidth]{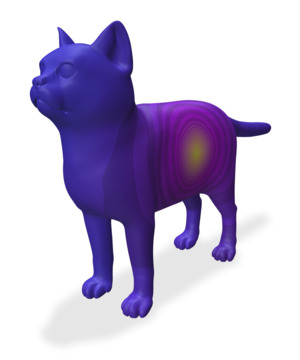}
\caption{Ours: $S_3$}
\label{fig:cat_3d_kernels_intr_C}
\end{subfigure}
\caption{Neural tangent kernels (NTKs)~\cite{DBLP:conf/nips/JacotHG18,neuraltangents2020} with different embedding functions. The source $S_1$ lies directly inside the ear of the cat. \sref{fig:cat_3d_kernels_xyz} The NTK using the extrinsic Euclidean embedding is not maximal at the source. \sref{fig:cat_3d_kernels_ff} The NTK using random Fourier features (RFF)~\cite{DBLP:conf/nips/TancikSMFRSRBN20} is localized correctly, but shows wrong behavior on the cat's body. \sref{fig:cat_3d_kernels_intr_A} The NTK with our intrinsic embedding is localized correctly and adapts to the local and global geometry. (\ref{fig:cat_3d_kernels_intr_B},\ref{fig:cat_3d_kernels_intr_C}) Additionally, the NTK with our intrinsic embedding is nearly shift-invariant, if the local geometry is approximately Euclidean: When the source is shifted from $S_2$ to $S_3$ the kernel is approximately shifted as well.}
\label{fig:cat_3d_kernels}
\end{figure}

\subsection{Theory} \label{sec:theory}\label{sec:stationary}
In this section, we will prove that the embedding function $\gamma$ proposed in \autoref{def:intrinsicneuralfield} generalizes the stationarity result of \cite{DBLP:conf/nips/TancikSMFRSRBN20} to certain manifolds. 
Stationarity is a desirable property if the kernel is used for interpolation, for example, in novel view synthesis \cite[App. C]{DBLP:conf/nips/TancikSMFRSRBN20}.
Fourier features induce a stationary (shift-invariant) neural tangent kernel (NTK). Namely, the composed NTK for two points in Euclidean space $\xx, \yy \in \R^m$ is given by
$\kntk(\xx, \yy) = (\hntk \comp \hg) (\xx - \yy) \,$
where $\hntk: \R \to \R$ is a scalar function related to the NTK of the MLP and $\hg: \R^m \to \R$ is a scalar function related to the Fourier feature embedding \cite[Eqn.~7,8]{DBLP:conf/nips/TancikSMFRSRBN20}. 
Extending this result to cover inputs $\pp, \qq \in \man$ on a manifold is challenging because the point difference $\pp - \qq$ and, therefore, the concept of stationary is not defined intrinsically. 
\pseudoparagraph{Stationarity on Manifolds} While one could use the Euclidean embedding of the manifold to define the difference $\pp - \qq$, this would ignore the local connectivity and can change under extrinsic deformations.
Instead, we make use of an equivalent definition from Bochner's theorem which implies that for Euclidean space any continuous, \emph{stationary} kernel is the Fourier transform of a non-negative measure \cite[Thm. 1]{DBLP:conf/nips/RahimiR07}. 
This definition can be directly used on manifolds, and we define a kernel $k: \man \times \man \to \R$ to be \textbf{stationary} if it can be written as
\begin{equation} \label{def:stationary}
    k(\pp, \qq) = \sum_i \hat{k}(\ev_i) \ef_i(\pp) \ef_i(\qq)\,, \qquad\qquad \hat{k}(\ev_i)  \geq 0 \;\; \forall i \,,
\end{equation}
where the function $\hat{k}: \R \to \R^{+}_0$ is akin to the Fourier transform.
This implies that $\hat{k}(\ev_i)$ and $\hat{k}(\ev_j)$ for identical eigenvalues $\ev_i = \ev_j$ must be identical.

First, we want to point out that for inputs with $\| \xx \| = \| \yy \|= r$ the result of $\kntk(\xx, \yy) = \hntk(\inner{\xx, \yy})$ shown by \cite{DBLP:conf/nips/JacotHG18} for $r = 1$ and used in \cite{DBLP:conf/nips/TancikSMFRSRBN20} still holds. 
This slight extension is given as \autoref{thm:ntk_of_normed}.
It is a prerequisite for the following theorem which requires the same setting as used in \cite{DBLP:conf/nips/JacotHG18}.

\begin{theorem} \label{thm:stationarity}
Let $\man$ be $\mathbb{S}^n$ or a closed 1-manifold. Let $(\ev_i, \ef_i)_{i = 1, \dots, d}$ be the positive, non-decreasing eigenvalues with associated eigenfunctions of the Laplace-Beltrami operator on $\man$. Let $a_i \geq 0$ be coefficients s.t.\ ${\ev_i = \ev_j \Rightarrow a_i = a_j}$, which define the embedding function $\gamma: \man \to \R^d$ with ${\gamma(\pp) = (a_1 \ef_1(\pp), \dots, a_d \ef_d(\pp))}$. Then, the composed neural tangent kernel $\kntk: \man \times \man \to \R$ of an MLP with the embedding $\gamma$ is \emph{stationary} as defined in \autoref{def:stationary}. 
\end{theorem}
\begin{proof}
Let $\man = \sphere^n$ and let $\shspace nl$ be the space of degree $l$ spherical harmonics on $\sphere^n$. Let $Y_{lm} \in \shspace nl$ be the $m$-th real spherical harmonic of degree $l$ with ${m=1, \dots, \dim\shspace nl}$. Notice that the spherical harmonics are the eigenfunctions of the LBO. We will use $j$ to linearly index the spherical harmonics and $l(j)$ for the degree.
Spherical harmonics of the same degree have the same eigenvalues, thus we use $c_{l(j)} = a_j = a_i$ for $\lambda_i = \lambda_j$ to denote the equal coefficients for same degree harmonics.
First, the norm of the embedding function is constant:
\begin{align}\label{eqn:stat_proof_norm}
\|  \gamma(\qq)\|^2
= \sum_j c_{l(j)}^2 \ef_j^2(\qq)
= \sum_l c_{l}^2 \sum_{m = 1}^{\dim \shspace nl} Y_{lm}^2(\qq) \overset{(a)}{=} \sum_l c_{l}^2 Z_l(\qq,\qq) \overset{(b)}{=} \text{const} \,.
\end{align}
Here, $Z_l(\qq,\qq)$ is the degree $l$ zonal harmonic and (a), (b) are properties of zonal harmonics \cite[Lem.~1.2.3, Lem.~1.2.7]{dai2013approximation}. Due to \autoref{eqn:stat_proof_norm} and \autoref{thm:ntk_of_normed} $\kntk(\gamma(\pp), \gamma(\qq)) = \hntk(\inner{\gamma(\pp), \gamma(\qq)}) \; \forall \pp, \qq \in \man$ holds. We can rewrite the scalar product as follows
\begin{align}
\inner{\gamma(\pp), \gamma(\qq)} =& 
\sum_j c_{l(j)}^2 \ef_j(\pp) \ef_j(\qq) 
= \sum_l  c_{l}^2 \sum_{m = 1}^{\dim \shspace nl} Y_{lm}(\pp) Y_{lm}(\qq) \\
\overset{(c)}{=}& \sum_l c_{l}^2 Z^l_{\pp}(\qq)
\overset{(d)}{=} \sum_l c_{l}^2 \, (1 + l/\alpha) \, C^{\alpha}_l(\inner{\pp, \qq}) \,,
\end{align}
where $C^\alpha_l: [-1, 1] \to \R$ are the Gegenbauer polynomials which are orthogonal on $[-1,1]$ for the weighting function $w_{\alpha}(z) = (1-z^2)^{\alpha - 1/2}$ with $\alpha = (n-1)/2$  \cite[B.2]{dai2013approximation}. Equality (c) holds again due to \cite[Lem.~1.2.3]{dai2013approximation}. 
Equality (d) holds due to a property of Gegenbauer polynomials \cite[Thm.~1.2.6]{dai2013approximation}, here $\inner{\pp, \qq}$ denotes the extrinsic Euclidean inner product. For the composed NTK we obtain
\begin{equation}
\kntk(\gamma(\pp), \gamma(\qq))
= \hntk \left( {\textstyle \sum_l} c_{l}^2 \, (1 + l/\alpha) \, C^{\alpha}_l(\inner{\pp, \qq}) \right)\,.
\end{equation}
We see that $\kntk(\gamma(\pp), \gamma(\qq))$ is a function depending only on $\inner{\pp, \qq}$. 
Because the Gegenbauer polynomials are orthogonal on $[-1,1]$, this function can be expanded %
with coefficients $\hat{c}_l \in \R$%
, which yields
\begin{align}
\kntk(\gamma(\pp), \gamma(\qq)) =& \sum_l \hat{c}_l \,  (1 + l/\alpha) \, C^{\alpha}_l(\inner{\pp, \qq}) = \sum_l \hat{c}_l \, Z^l(\pp,\qq) \\
=& \sum_l \hat{c}_l \sum_{m = 1}^{\dim \shspace nl} Y_{lm}(\pp) Y_{lm}(\qq)
= \sum_j \hat{c}_{l(j)} \ef_j(\pp) \ef_j(\qq) \,.
\end{align}
The coefficients $\hat{c}_{l(j)}$ are non-negative as a consequence of the positive definiteness of the NTK~\cite[Prop.~2]{DBLP:conf/nips/JacotHG18} and a classic result by Schoenberg~\cite[Thm.~14.3.3]{dai2013approximation}. This shows that $\kntk(\gamma(\pp), \gamma(\qq))$ is stationary as defined in Equation \ref{def:stationary}. \qed
\end{proof}
The adapted proof for 1-manifolds can be found in \autoref{sec:theory_thm_man1}.
A qualitative example of the stationary kernels can be seen in \autoref{fig:theory_1d}.
The theorem does not hold for general manifolds but we do not consider this a shortcoming. The composed kernel adapts to the intrinsic geometry of the underlying manifold and is approximately shift-invariant between points that share a similar local neighborhood, see \autoref{fig:cat_3d_kernels}. Hence, the proposed method naturally and unambiguously integrates the geometry of the manifold based on an intrinsic representation.

\section{Experiments} \label{sec:applications}
Due to the large number of experiments presented within this section, we refer to \autoref{sup:sec:details_experiments}
for all experimental details and hyperparameter settings as well as further results. 
To facilitate fair comparisons, all methods use the same hyperparameters like learning rate, optimizer, number of training epochs, or MLP architecture except when noted otherwise.
For baselines using random Fourier features (RFF), we follow \cite{DBLP:conf/nips/TancikSMFRSRBN20} and tune the standard deviation~$\sigma$ (c.f.~\autoref{eqn:general_fourier_feature_encoding}) of the random frequency matrix to obtain optimal results.

\subsection{Texture Reconstruction from Images} \label{sec:app:texture_representation}
To investigate the representation power of the proposed intrinsic neural fields, we consider the task of texture reconstruction from posed images as proposed by Oechsle et al.~\cite{DBLP:conf/iccv/OechsleMNSG19} in \autoref{tab:texture_reconstruction} and \autoref{fig:texture_representation}. The input to our algorithms is a set of five $512{\times}512$ images with their camera poses and the triangle mesh of the shape. After fitting the intrinsic neural field to the data, we render images from $200$ novel viewpoints and compare them to ground-truth images for evaluation.

For each pixel, we perform ray mesh intersection between the ray through the pixel and the mesh. The eigenfunctions of the Laplace-Beltrami operator are defined only on vertices of the mesh \cite{DBLP:journals/cgf/SharpC20}. Within triangles, we use barycentric interpolation. We employ the mean $L^1$ loss across a batch of rays and the RGB color channels.
The eigenfunction computation and ray-mesh intersection are performed once at the start of training. Hence, our training speed is similar to the baseline method that uses random Fourier features. Training takes approx.\ one hour on an Nvidia Titan X with 12 GB memory.

\pseudoparagraph{Comparison with State of the Art Methods}
\sloppy
We compare against Texture Fields~\cite{DBLP:conf/iccv/OechsleMNSG19} enhanced with random Fourier features (RFF) \cite{DBLP:conf/nips/TancikSMFRSRBN20}. Additionally, we compare against NeuTex~\cite{DBLP:conf/cvpr/XiangXHHS021}, which uses a network to map a shape to the sphere and represents the texture on this sphere. We adapt NeuTex s.t.\ it takes advantage of the given geometry, see \autoref{sup:sec:neutex_mod}.
\autoref{tab:texture_reconstruction} and \autoref{fig:texture_representation} show that intrinsic neural fields can reconstruct texture with state-of-the-art quality. This is also true if the number of training epochs is decreased from $1000$ to $200$.

\pseudoparagraph{Ablation Study}
We investigate the effect of different hyperparameters on the quality of the intrinsic neural texture field. The results in \autoref{tab:ablation_study} show that the number of eigenfunctions is more important than the size of the MLP, which is promising for real-time applications. A model using only $64$ eigenfunctions and $17k$ parameters\footnote{For reference: a $80\times80$ 3-channel color texture image has over $17k$ pixel values.} still achieves a PSNR of $29.20$ for the cat showing that intrinsic neural fields can be a promising approach for compressing manifold data.

\begin{table}[tb]
\scriptsize
\caption{Texture reconstruction from images.
Our intrinsic neural fields show state-of-the-art performance (\textit{first row block}), which also holds for fewer training epochs (\textit{Ep.\,$\downarrow$, second row block}).
For a fair comparison, we improve the original Texture Fields by employing the same MLP architecture as our model and by additionally using random Fourier features (\textit{TF+RFF}). NeuTex already has more parameters than our model and we increase the embedding size (\textit{Em.\,$\uparrow$}). We adapt NeuTex s.t.\ it takes advantage of the given geometry, which we detail in \autoref{sup:sec:neutex_mod}.
The methods are evaluated on novel views using the PSNR, DSSIM~\cite{wang2004image}, and LPIPS~\cite{DBLP:conf/cvpr/ZhangIESW18}.
DSSIM and LPIPS are scaled by $100$. For each row block, the best number is in bold font. The intrinsic representation shows better results than the extrinsic representation (\textit{TF+RFF}) as well as when mapping to a sphere and representing the texture there (\textit{NeuTex}). For qualitative results, see \autoref{fig:texture_representation}.
}
\label{tab:texture_reconstruction}

\newlength{\extrinsiclength}
\settowidth{\extrinsiclength}{TF+RFF}

\newlength{\extrinsiclengthparen}
\settowidth{\extrinsiclengthparen}{TF+RFF ($\sigma{=}8$)}

\begin{center}
\begin{tabular}{p{3.3cm} >{\centering\arraybackslash}p{0.8cm} >{\centering\arraybackslash}p{0.8cm} @{\hspace{0.2cm}} c c c c @{\hspace{0.2cm}} c c c}
\toprule
&
\multirow{2}{*}{\makecell{Em.}}&
\multirow{2}{*}{Ep.}&
\multicolumn{3}{c}{cat} & &
\multicolumn{3}{c}{human} \\

\cmidrule{4-6}
\cmidrule{8-10}

&
&
&
\makecell{\notsotiny PSNR\,$\uparrow$}&
\makecell{\notsotiny DSSIM\,$\downarrow$}&
\makecell{\notsotiny LPIPS\,$\downarrow$}& &
\makecell{\notsotiny PSNR\,$\uparrow$}&
\makecell{\notsotiny DSSIM\,$\downarrow$}&
\makecell{\notsotiny LPIPS\,$\downarrow$}\\
\midrule

NeuTex~\cite{DBLP:conf/cvpr/XiangXHHS021}                     & \phantom{10}63                             & 1000                           & 31.60                          & 0.242                          & 0.504                         &  & 29.49                          & 0.329                          & 0.715                         \\
NeuTex Em.\,$\uparrow$                  & 1023                           & 1000                           & 31.96                          & 0.212                          & 0.266                         &  & 29.22                          & 0.306                          & 0.669                         \\
TF+RFF ($\sigma{=}4$)~\cite{DBLP:conf/iccv/OechsleMNSG19,DBLP:conf/nips/TancikSMFRSRBN20}              & 1023                           & 1000                           & 33.86                          & 0.125                          & 0.444                         &  & 32.04                          & 0.130                          & 0.420                         \\
TF+RFF ($\sigma{=}16$)                 & 1023                           & 1000                           & 34.19                          & 0.105                          & 0.167                         &  & 31.53                          & 0.193                          & 0.414                         \\
TF+RFF ($\sigma{=}8$)                       & 1023                           & 1000                           & 34.39                          & 0.097                          & 0.205                         &  & 32.26                          & 0.129                          & 0.336                         \\
\makebox[\extrinsiclength][l]{Intrinsic} (Ours)                & 1023                           & 1000                           & \textBF{34.82}                 & \textBF{0.095}                 & \textBF{0.153}                &  & \textBF{32.48}                 & \textBF{0.121}                 & \textBF{0.306}                \\

\midrule
\makebox[\extrinsiclengthparen][l]{NeuTex} Ep.\,$\downarrow$               & 1023                           & \phantom{1}200                            & 30.96                          & 0.290                          & 0.355                         &  & 28.02                          & 0.418                          & 0.900                         \\
TF+RFF ($\sigma{=}8$) Ep.\,$\downarrow$ & 1023                           & \phantom{1}200                            & 34.07                          & 0.116                          & 0.346                         &  & 31.85                          & 0.142                          & 0.427                         \\
\makebox[\extrinsiclengthparen][l]{Intrinsic (Ours)} Ep.\,$\downarrow$ & 1023                           & \phantom{1}200                            & \textBF{34.79}                 & \textBF{0.100}                 & \textBF{0.196}                &  & \textBF{32.37}                 & \textBF{0.126}                 & \textBF{0.346}                \\

\bottomrule

\end{tabular}
\end{center}
\vspace{-0.5cm}
\end{table}

\begin{figure}
\centering
\begin{subfigure}[b]{0.23\textwidth}
\centering
\includegraphics[width=\textwidth]{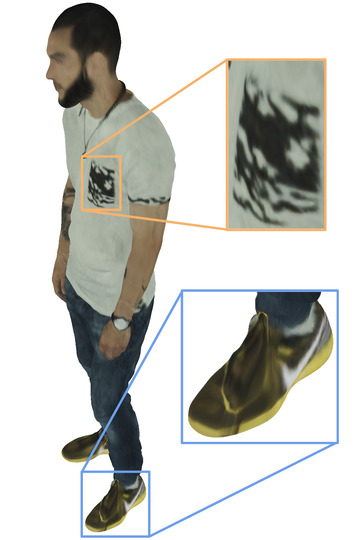}
\caption{\vspace{-0.2cm}NeuTex~\cite{DBLP:conf/cvpr/XiangXHHS021}}
\label{fig:human_texture_recon_neutex}
\end{subfigure}
\hfill
\begin{subfigure}[b]{0.23\textwidth}
\centering
\includegraphics[width=\textwidth]{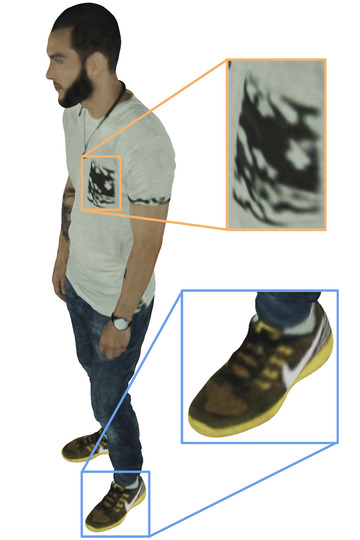}
\caption{\vspace{-0.2cm}TF ($\sigma{=}8$)~\cite{DBLP:conf/iccv/OechsleMNSG19,DBLP:conf/nips/TancikSMFRSRBN20}}
\label{fig:human_texture_recon_rff8}
\end{subfigure}
\hfill
\begin{subfigure}[b]{0.23\textwidth}
\centering
\includegraphics[width=\textwidth]{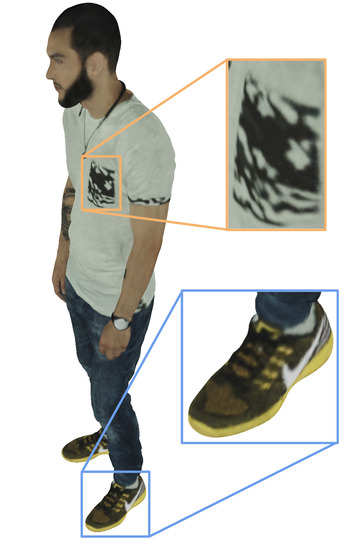}
\caption{\vspace{-0.2cm}Ours}
\label{fig:human_texture_recon_ours}
\end{subfigure}
\hfill
\begin{subfigure}[b]{0.23\textwidth}
\centering
\includegraphics[width=\textwidth]{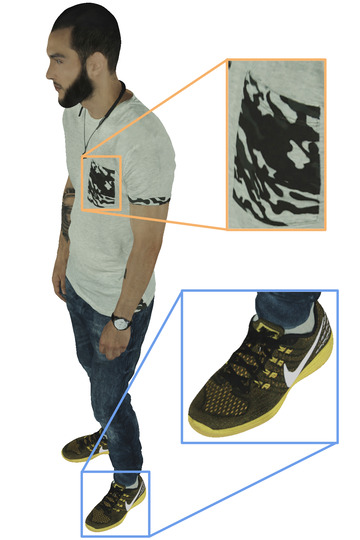}
\caption{\vspace{-0.2cm}GT Image}
\label{fig:human_texture_recon_gt}
\end{subfigure}
\caption{Texture reconstruction from images. \sref{fig:human_texture_recon_neutex} NeuTex uses a network to map from the shape to the sphere and represents the texture on the sphere, which yields distortions around the shoe.
\sref{fig:human_texture_recon_rff8} Texture Fields (TF)~\cite{DBLP:conf/iccv/OechsleMNSG19} with random Fourier Features (RFF)~\cite{DBLP:conf/nips/TancikSMFRSRBN20} learns the texture well and only around the breast pocket our method shows slightly better results. \sref{fig:human_texture_recon_ours} Intrinsic neural fields can reconstruct texture from images with state-of-the-art quality, which we show quantitatively in \autoref{tab:texture_reconstruction}.
}
\label{fig:texture_representation}
\end{figure}

\begin{table}[tb]
\scriptsize
\caption{Ablation study based on the texture reconstruction experiment (c.f.~\autoref{sec:app:texture_representation}). The number of eigenfunctions is more important than the size of the MLP which is promising for real-time applications. A model using only $64$ eigenfunctions and only 17k parameters still achieves a PSNR of $29.20$ for the cat, which shows that intrinsic neural fields can be a promising approach for compressing manifold data.}
\label{tab:ablation_study}

\begin{center}
\begin{tabular}{p{3.5cm} >{\centering\arraybackslash}p{1.2cm} >{\centering\arraybackslash}p{0.8cm} @{\hspace{0.2cm}} c c c c @{\hspace{0.2cm}} c c c}
\toprule
&
\multirow{2}{*}{\#Params}&
\multirow{2}{*}{\#$\phi$}&
\multicolumn{3}{c}{cat} & &
\multicolumn{3}{c}{human} \\

\cmidrule{4-6}
\cmidrule{8-10}

&
&
&
\makecell{\notsotiny PSNR\,$\uparrow$}&
\makecell{\notsotiny DSSIM\,$\downarrow$}&
\makecell{\notsotiny LPIPS\,$\downarrow$}& &
\makecell{\notsotiny PSNR\,$\uparrow$}&
\makecell{\notsotiny DSSIM\,$\downarrow$}&
\makecell{\notsotiny LPIPS\,$\downarrow$}\\
\midrule

Full model & 329k & 1023                 & \textBF{34.82}                 & \textBF{0.095}                 & \textBF{0.153}                &  & \textBF{32.48}                 & \textBF{0.121}                 & \textBF{0.306} \\
Smaller MLP & 140k & 1023     & 34.57                          & 0.108                          & 0.205                         &  & 32.20                          & 0.134                          & 0.379                         \\
Fewer eigenfunctions & 83k & 64 & 31.18                          & 0.284                          & 0.927                         &  & 28.95                          & 0.312                         & 1.090                         \\
Smaller MLP \& fewer efs & 17k & 64 & 29.20                          & 0.473                          & 1.428                         &  & 26.72                          & 0.493                          & 1.766                         \\
Just $4$ eigenfunctions & 68k & 4 & 22.84                          & 1.367                          & 3.299                         &  & 20.60                          & 1.033                          & 2.756                         \\

\bottomrule

\end{tabular}
\end{center}
\vspace{-0.5cm}
\end{table}

\subsection{Discretization-agnostic Intrinsic Neural Fields} \label{sec:app:geometry_representation}
For real-world applications, it is desirable that intrinsic neural fields can be trained for different discretizations of the same manifold. First, the training process of the intrinsic neural field should be robust to the sampling in the discretization. Second, it would be beneficial if an intrinsic neural field trained on one discretization could be transferred to another, which we show in \autoref{sec:app:texture_transfer}. To quantify the discretization dependence of intrinsic neural fields, we follow the procedure proposed by Sharp et al.~\cite[Sec.~5.4]{sharp2021diffusion} and rediscretize the meshes used in \autoref{sec:app:texture_representation}. The qualitative results in \autoref{fig:discretization} and the quantitative results in \autoref{tab:discretization} show that intrinsic neural fields work across various  discretizations. Furthermore, \autoref{fig:texture_transfer} shows that transferring pre-trained intrinsic neural fields across discretizations is possible with minimal loss in visual quality.

\begin{table}[tb]
\scriptsize
\caption{Discretization-agnostic intrinsic neural fields. We employ the procedure proposed by Sharp et al.~\cite[Sec.~5.4]{sharp2021diffusion} to generate different discretizations of the original meshes (\textit{orig}): uniform isotropic remeshing (\textit{iso}), densification around random vertices (\textit{dense}), refinement and subsequent quadric error simplification \cite{DBLP:conf/siggraph/GarlandH97} (\textit{qes}), and point clouds sampled from the surfaces with more points than vertices (\textit{cloud}\,$\uparrow$) and with fewer points (\textit{cloud}\,$\downarrow$). The discretizations are then used for texture reconstruction as in \autoref{sec:app:texture_representation}. For the point clouds, we use local triangulations \cite[Sec.\ 5.7]{DBLP:journals/cgf/SharpC20} for ray-mesh intersection. This table and the qualitative results in \autoref{fig:discretization} show that intrinsic neural fields can be trained for a wide variety of discretizations. Furthermore, pre-trained intrinsic neural fields can be transferred across discretizations as shown in \autoref{fig:texture_transfer}.}
\label{tab:discretization}
\begin{center}
\begin{tabular}{p{1.7cm} c c c c c c c@{\hspace{0.4cm}}  c c c c c c}
\toprule

&
\multicolumn{6}{c}{cat}
&
&
\multicolumn{6}{c}{human} \\

\cmidrule{2-7}
\cmidrule{9-14}

Method &
\makecell{orig} &
\makecell{iso}&
\makecell{dense}&
\makecell{qes}&
\makecell{cloud\,$\uparrow$}&
\makecell{cloud\,$\downarrow$}&
&
\makecell{orig} &
\makecell{iso}&
\makecell{dense}&
\makecell{qes}&
\makecell{cloud\,$\uparrow$}&
\makecell{cloud\,$\downarrow$}\\

\midrule

PSNR $\uparrow$ & 34.82 & 34.85                          & 34.74                          & \textBF{35.07}                 & 34.91                          & 33.17 & & 32.48 & \textBF{32.63}                 & 32.57                          & 32.49                          & 32.45                          & 31.99 \\
DSSIM $\downarrow$ & 0.095 & \textBF{0.093}                 & 0.096                          & 0.096                          & 0.096                          & 0.130 & & 0.121 & \textBF{0.117}                 & 0.120                          & 0.121                          & 0.123                          & 0.135 \\
LPIPS $\downarrow$ & 0.153 & 0.152                          & 0.159                          & \textBF{0.147}                 & 0.152                          & 0.220 & & 0.306 & 0.300                          & 0.301                          & \textBF{0.297}                 & 0.307                          & 0.323 \\

\bottomrule

\end{tabular}
\end{center}
\vspace{-0.5cm}
\end{table}

\begin{figure}
\centering
\begin{subfigure}[b]{0.19\textwidth}
\centering
\includegraphics[width=\textwidth]{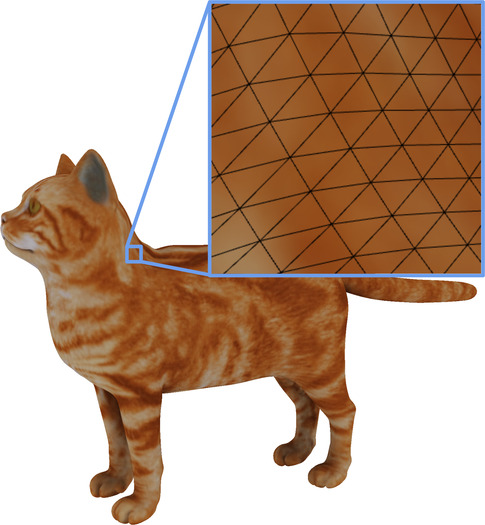}
\caption{\vspace{-0.2cm}orig}
\label{fig:disc_orig}
\end{subfigure}
\hfill
\begin{subfigure}[b]{0.19\textwidth}
\centering
\includegraphics[width=\textwidth]{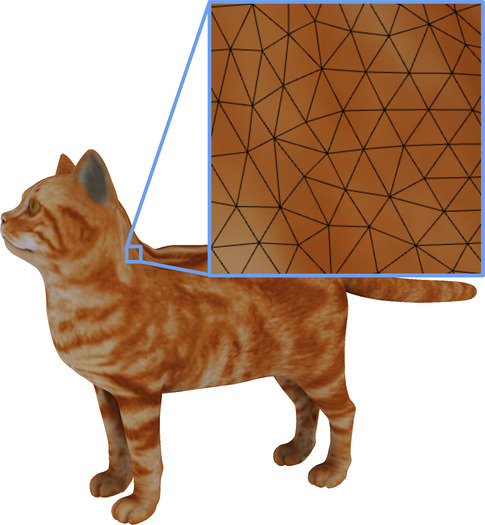}
\caption{\vspace{-0.2cm}iso}
\label{fig:disc_iso}
\end{subfigure}
\hfill
\begin{subfigure}[b]{0.19\textwidth}
\centering
\includegraphics[width=\textwidth]{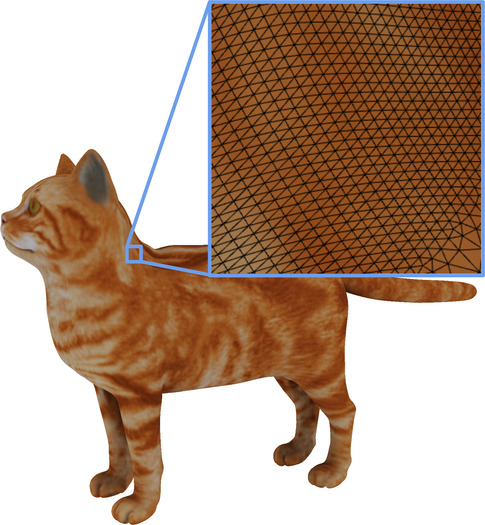}
\caption{\vspace{-0.2cm}dense}
\label{fig:disc_dens}
\end{subfigure}
\hfill
\begin{subfigure}[b]{0.19\textwidth}
\centering
\includegraphics[width=\textwidth]{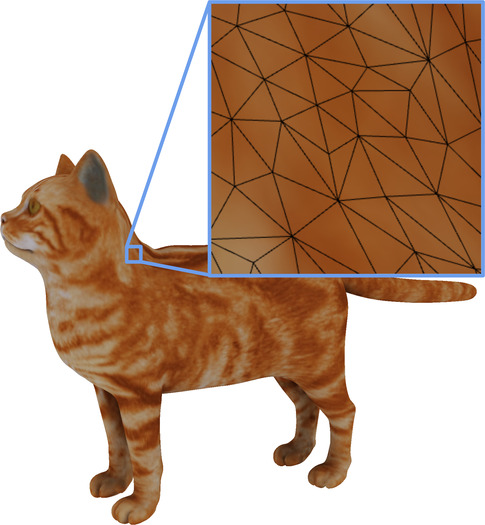}
\caption{\vspace{-0.2cm}qes}
\label{fig:disc_qes}
\end{subfigure}
\hfill
\begin{subfigure}[b]{0.19\textwidth}
\centering
\includegraphics[width=\textwidth]{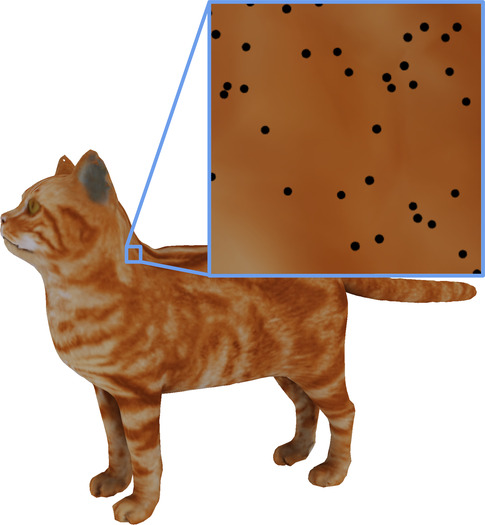}
\caption{\vspace{-0.2cm}cloud\,$\downarrow$}
\label{fig:disc_cloud10}
\end{subfigure}
\caption{
Discretization-agnostic intrinsic neural fields. Our method produces identical results for a wide variety of triangular meshings and even point cloud data.
For the point cloud, we use local triangulations \cite[Sec.\ 5.7]{DBLP:journals/cgf/SharpC20} for ray-mesh intersection. Pre-trained intrinsic neural fields can be transferred across discretizations as shown in \autoref{fig:texture_transfer}.
}
\label{fig:discretization}
\end{figure}

\begin{figure}
\centering
\begin{subfigure}[b]{0.19\textwidth}
\centering
\includegraphics[width=\textwidth]{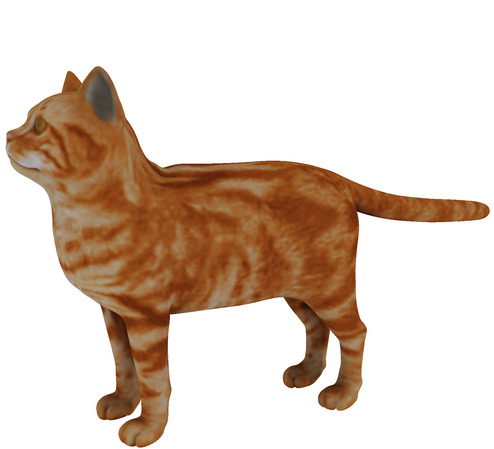}
\caption{\vspace{-0.2cm}source}
\label{fig:texture_transfer_source}
\end{subfigure}
\hfill
\begin{subfigure}[b]{0.19\textwidth}
\centering
\includegraphics[width=\textwidth]{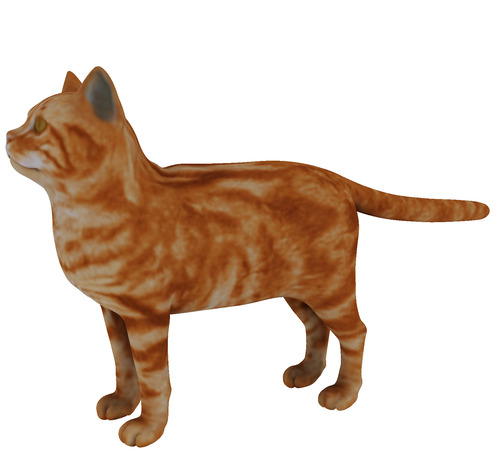}
\caption{\vspace{-0.2cm}dense}
\label{fig:texture_transfer_disc}
\end{subfigure}
\hfill
\begin{subfigure}[b]{0.19\textwidth}
\centering
\includegraphics[width=\textwidth]{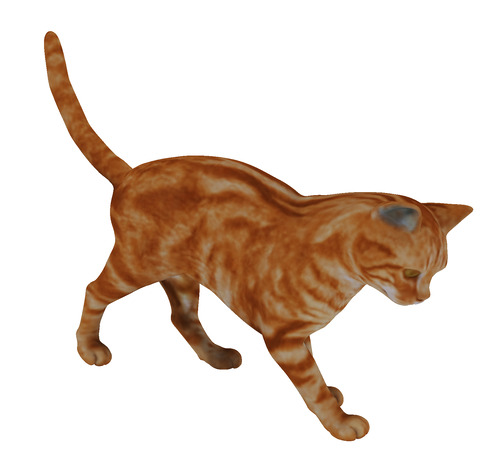}
\caption{\vspace{-0.2cm}ARAP~\cite{DBLP:conf/sgp/SorkineA07}}
\label{fig:texture_transfer_arap}
\end{subfigure}
\hfill
\begin{subfigure}[b]{0.19\textwidth}
\centering
\includegraphics[width=\textwidth]{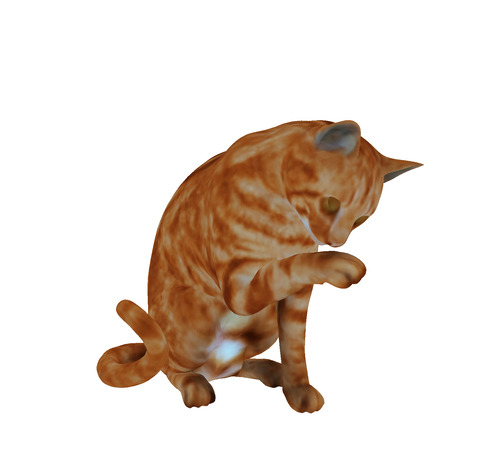}
\caption{\vspace{-0.2cm}TOSCA cat 2}
\label{fig:texture_transfer_tosca_cat}
\end{subfigure}
\hfill
\begin{subfigure}[b]{0.19\textwidth}
\centering
\includegraphics[width=\textwidth]{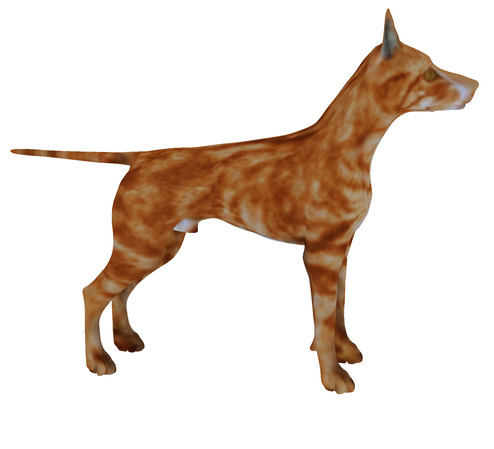}
\caption{\vspace{-0.2cm}TOSCA dog 0}
\label{fig:texture_transfer_tosca_dog}
\end{subfigure}
\caption{Intrinsic neural field transfer. \sref{fig:texture_transfer_source} The pre-trained intrinsic neural texture field from the source mesh is transferred to the target shapes using functional maps \cite{ovsjanikov12funmaps,DBLP:conf/cvpr/EisenbergerLC20}. (\ref{fig:texture_transfer_disc},\ref{fig:texture_transfer_arap}) The transfer across rediscretization (c.f.~\autoref{fig:discretization}) and deformation gives nearly perfect visual quality. (\ref{fig:texture_transfer_tosca_cat},\ref{fig:texture_transfer_tosca_dog}) As a proof of concept, we show artistic transfer to a different cat shape and a dog shape from the TOSCA dataset~\cite{bronstein2008numerical}. Both transfers work well but the transfer to the dog shows small visual artifacts in the snout area due to locally different geometry. Overall, the experiment shows the advantage of the intrinsic formulation which naturally incorporates field transfer through functional maps.}
\label{fig:texture_transfer}
\end{figure}

\subsection{Intrinsic Neural Field Transfer} \label{sec:app:texture_transfer}
One advantage of the Laplace-Beltrami operator is its invariance under isometries which allows for transferring a pre-trained intrinsic neural field from one manifold to another. However, this theoretic invariance does not always perfectly hold in practice, for example, due to discretization artifacts as discussed by Kovnatsky et al.~\cite{kovnatsky2013coupled}. 
Hence, we employ functional maps \cite{ovsjanikov12funmaps} to transfer the eigenfunctions of the source to the target shape, computed with the method proposed by Eisenberger et al.~\cite{DBLP:conf/cvpr/EisenbergerLC20}.
\autoref{fig:texture_transfer} shows that intrinsic neural fields can be transferred between shapes. Specifically, the transfer is possible between different discretizations and deformations~\cite{DBLP:conf/sgp/SorkineA07} of the same shape. As a proof of concept, we also show artistic transfer, which yields satisfying results for shapes from the same category, but small visual artifacts for shapes from different categories. 
It is, of course, possible to generate similar results with extrinsic fields by calculating a point-to-point correspondence and mapping the coordinate values. However, functional maps are naturally low-dimensional, continuous, and differentiable. This makes them a beneficial choice in many applications, especially related to learning.

\subsection{Real-world Data \& View Dependence} \label{sec:app:real_world_view_dep}
We validate the effectiveness of intrinsic neural fields under real-world settings on the BigBIRD dataset \cite{DBLP:conf/icra/SinghSNAA14}. The dataset provides posed images and reconstructed meshes, and we apply a similar pipeline as in \autoref{sec:app:texture_representation}. However, the objects under consideration are not perfectly Lambertian, and thus, view dependence must be considered. 
This is achieved by using the viewing direction as an additional input to the network, as done in \cite{DBLP:conf/eccv/MildenhallSTBRN20}. At first glance, using the viewing direction in its extrinsic representation is in contrast to our intrinsic definition of neural fields. However, view dependence arises from the extrinsic scene of the object, such as the lighting, which cannot be represented purely intrinsically. 
Decomposing the scene into the intrinsic properties, like the BRDF, of the object and the influence of the environment, like light sources, is an interesting future application for intrinsic neural fields. Such decomposition has recently been studied in the context of neural radiance fields \cite{DBLP:conf/eccv/ChenNN20,DBLP:conf/cvpr/ZhangLWBS21,boss2021nerd,DBLP:journals/corr/abs-2201-02279,DBLP:journals/corr/abs-2110-14373}. Intrinsic neural fields can reconstruct high-quality textures from real-world data with imprecise calibration and imprecise meshes, as shown in \autoref{fig:bigbird}.

\begin{figure}
\centering
\begin{subfigure}[b]{0.16\textwidth}
\centering
\includegraphics[width=\textwidth]{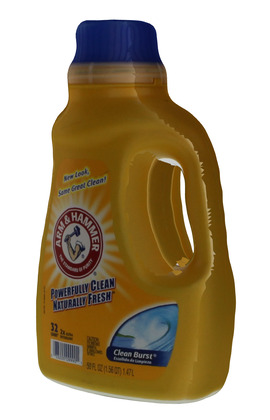}
\caption{\vspace{-0.2cm}Baseline~\cite{DBLP:conf/icra/SinghSNAA14}}
\label{fig:bigbird_det_theirs}
\end{subfigure}
\hfill
\begin{subfigure}[b]{0.16\textwidth}
\centering
\includegraphics[width=\textwidth]{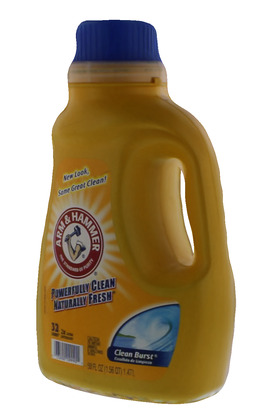}
\caption{\vspace{-0.2cm}Ours}
\label{fig:bigbird_det_ours}
\end{subfigure}
\hfill
\begin{subfigure}[b]{0.16\textwidth}
\centering
\includegraphics[width=\textwidth]{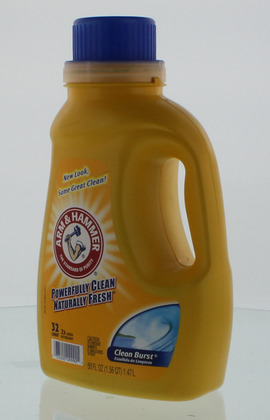}
\caption{\vspace{-0.2cm}GT Image}
\label{fig:bigbird_det_gt}
\end{subfigure}
\hfill
\begin{subfigure}[b]{0.16\textwidth}
\centering
\includegraphics[width=\textwidth]{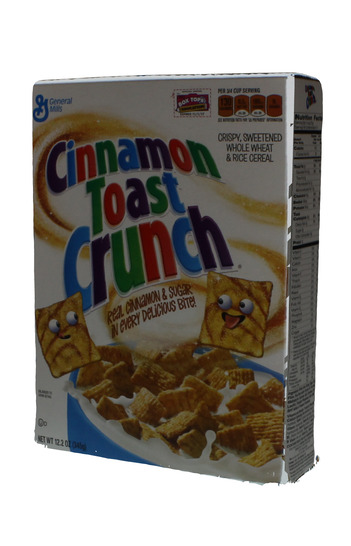}
\caption{\vspace{-0.2cm}Baseline}
\label{fig:bigbird_cin_theirs}
\end{subfigure}
\hfill
\begin{subfigure}[b]{0.16\textwidth}
\centering
\includegraphics[width=\textwidth]{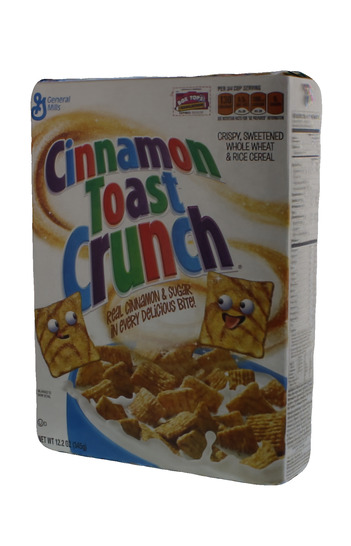}
\caption{\vspace{-0.2cm}Ours}
\label{fig:bigbird_cin_ours}
\end{subfigure}
\hfill
\begin{subfigure}[b]{0.16\textwidth}
\centering
\includegraphics[width=\textwidth]{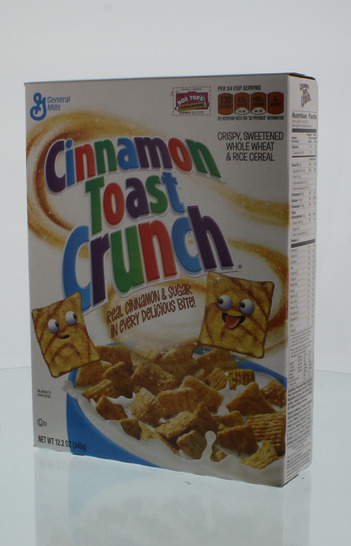}
\caption{\vspace{-0.2cm}GT Image}
\label{fig:bigbird_cin_gt}
\end{subfigure}
\caption{Texture reconstruction from real-world data. (\ref{fig:bigbird_det_ours},\ref{fig:bigbird_cin_ours}) Intrinsic neural fields can reconstruct high quality textures from the real-world BigBIRD dataset \cite{DBLP:conf/icra/SinghSNAA14} with imprecise calibration and imprecise meshes. (\ref{fig:bigbird_det_theirs},\ref{fig:bigbird_cin_theirs}) The baseline texture mapped meshes provided in the dataset show notable seams due to the non-Lambertian material, which are not present in our reconstruction that utilizes view dependence as proposed by \cite{DBLP:conf/eccv/MildenhallSTBRN20}.}
\label{fig:bigbird}
\end{figure}

\section{Conclusion}

\pseudoparagraph{Discussion}
The proposed intrinsic formulation of neural fields outperforms the extrinsic formulation in the presented experiments. However, if the data is very densely sampled from the manifold, and the kernel is thus locally limited, the extrinsic method can overcome many of its weaknesses shown before. In practice, dense sampling often leads to an increase in runtime of further processing steps, and therefore, we consider our intrinsic approach still to be superior. 
Further, we provided the proof for a stationary NTK on n-spheres. 
Our good results and intuition imply that for general manifolds, it is advantageous how the NTK takes local geometry into account. We believe this opens up an interesting direction for further theoretical analysis.

\pseudoparagraph{Conclusion}
We present intrinsic neural fields, an elegant and direct generalization of neural fields for manifolds. 
Intrinsic neural fields can represent high-frequency functions on the manifold surface independent of discretization by making use of the Laplace-Beltrami eigenfunctions. 
As a result, they also inherit beneficial properties of the LBO, like isometry invariance, a natural frequency filter, and are directly compatible with the popular functional map framework.
We introduce a new definition for stationary kernels on manifolds, and our theoretic analysis shows that the derived neural tangent kernel is stationary under specific conditions. 

We conduct experiments to investigate the capabilities of our framework on the application of texture reconstruction from a limited number of views. 
Our results show that intrinsic neural fields can represent high-frequency functions on the surface independent of sampling density and surface discretization.
Furthermore, the learned functions can be transferred to new examples using functional maps without any retraining, and view-dependent changes can be incorporated. 
Intrinsic neural fields outperform competing methods in all settings. Additionally, they add flexibility, especially in settings with deformable objects due to the intrinsic nature of our approach.

\noindent \textbf{Acknowledgements}\\

\noindent We express our appreciation to our colleagues who have supported us. Specifically we thank Simon Klenk, Tarun Yenamandra, Bj\"orn H\"afner, and Dominik Muhle for proofreading and helpful discussions. We want to thank the research community at large and contributors to open source projects for openly and freely sharing their knowledge and work.

\bibliographystyle{splncs04}
\bibliography{egbib}

\clearpage

\appendix

\section*{\Large Supplementary Material}

\noindent In this supplementary, we elaborate on the implementation details for our intrinsic neural fields (\autoref{sup:sec:impl_details}), discuss the details of our experiments (\autoref{sup:sec:details_experiments}), show that our method is not overly initialization-dependent (\autoref{sup:sec:init_dependence}), and, finally, provide further theoretical results (\autoref{sup:sec:theory}).

The high-resolution intrinsic neural texture field on the human model that we showcase in \refmainpaper{Fig.~1}{fig:teaser} is available as a textured mesh on \href{https://skfb.ly/otvFK}{sketchfab}\footnote{https://skfb.ly/otvFK}. We would like to note that the small texture seams are not due to our method, but due to the conversion from our network to a uv texture map. Intrinsic neural fields do not possess the discontinuities which are present in the uv map. The texture is created by an inverse uv lookup of each texel and an evaluation of the intrinsic neural field at the corresponding point on the manifold. We provide the textured mesh as a convenient possibility for qualitative inspection with current tools but it is never used to evaluate the proposed method qualitatively or quantitatively in the paper.

\section{Implementation Details}\label{sup:sec:impl_details}

\begin{figure}[tb]
    \centering
    \begin{overpic}[width=\linewidth]{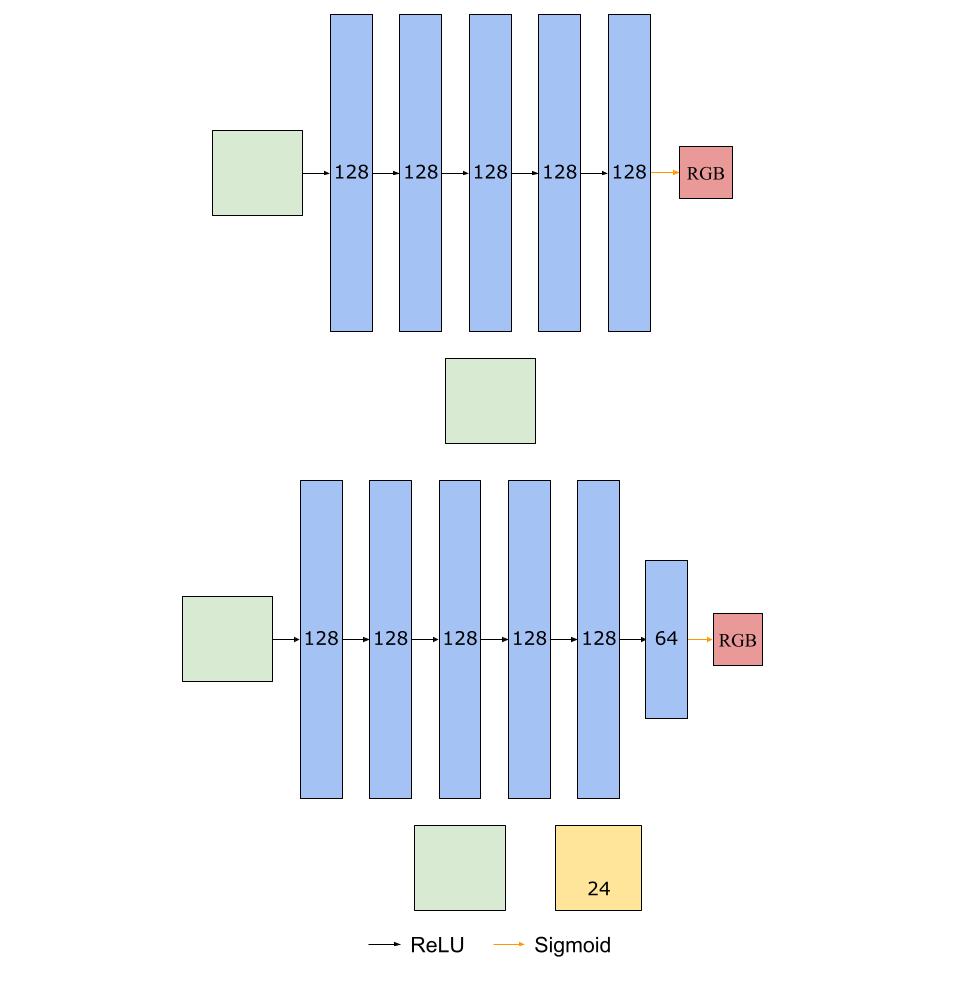}
    \put(21.95,83.5){\scriptsize{$\gamma_{\text{Intr}}(\pp)$}}
    \put(25,79.35){$d$}
    \put(45.3,61){\scriptsize{$\gamma_{\text{Intr}}(\pp)$}}
    \put(48.35,56.85){$d$}
    \put(48,64.9){$+$}
    \put(18.9,37){\scriptsize{$\gamma_{\text{Intr}}(\pp)$}}
    \put(21.95,32.85){$d$}
    \put(42.3,14.5){\scriptsize{$\gamma_{\text{Intr}}(\pp)$}}
    \put(45.35,10.35){$d$}
    \put(45, 18.2){$+$}
    \put(56.5,14.5){\scriptsize{$\gamma_{\text{FF}}(\mathbf{d})$}}
    \put(58.75, 18.2){$+$}
    \end{overpic}
    \caption{Network architecture. We use a similar network architecture as used for NeRF~\cite{DBLP:conf/eccv/MildenhallSTBRN20}. A point on the 2-manifold is described by $\pp$. The unit ray direction is represented by $\dd$. The notations $\gamma_{\text{Intr}}$ and $\gamma_{\text{FF}}$ represent the proposed eigenfunction embedding and the sine/cosine positional encoding~\cite{DBLP:conf/eccv/MildenhallSTBRN20} respectively. The $+$ sign denotes concatenation. The second architecture is used in the experiments of \refmainpaper{Sec. 5.4}{sec:app:real_world_view_dep} while we use the first architecture in all other experiments.}
    \label{fig:intrinsic_network_arch}
\end{figure}

For our method, we calculate the eigenfunctions of the Laplace-Beltrami operator of a given triangle mesh once by 
solving the generalized eigenvalue problem for the first $d$ eigenvalues using the robust Laplacian by Sharp and Crane~\cite{DBLP:journals/cgf/SharpC20}.
If the given geometry is a pointcloud, we create a local triangulation around each point which lets us perform a normal ray-mesh intersection. Additionally, the robust Laplacian~\cite{DBLP:journals/cgf/SharpC20} supports calculating eigenfunctions on pointclouds.

Depending on whether the viewing direction is taken into account, we use the respective network architectures shown in \autoref{fig:intrinsic_network_arch} for our experiments.
Both networks take as input a point $\pp$ from the surface of the discrete 2-manifold embedded into its $d$ eigenfunctions. 
Since $\pp$ might not be a vertex, we linearly interpolate the eigenfunctions of the vertices $v_i$, $v_j$, and $v_k$, which span the triangle face where $\pp$ is located, using the barycentric coordinates.
For embedding the unit viewing direction $\dd \in \mathbb{R}^3$, we use the sine/cosine positional encoding~\cite{DBLP:conf/eccv/MildenhallSTBRN20}.

During the training, we randomly sample preprocessed rays from the whole training split. 
We use a batch size of 4096 across all our experiments. 
As optimizer, we use Adam with the default parameters ($\beta_1 = 0.9$, $\beta_2 = 0.999$, $\epsilon = 10^{-8}$). 
Our loss function is the mean L1 loss over a batch of randomly sampled rays $\mathcal{B}$
\begin{align} \label{eqn:mean_L1_loss}
\mathcal{L}_1 = \frac{1}{|\mathcal{B}|} \sum_{\pp \in \mathcal{B}} \| F_{\theta}(\pp) - c_{\text{gt}}(\pp) \|_1
\end{align}
where $F_{\theta}$ is an intrinsic neural field and $c_{\text{gt}}(\pp)$ is the ground-truth RGB color.

\section{Experimental Details} \label{sup:sec:details_experiments}

\subsection{Modification of NeuTex} \label{sup:sec:neutex_mod}

\begin{figure}[tb]
    \centering
    \begin{overpic}[width=\linewidth]{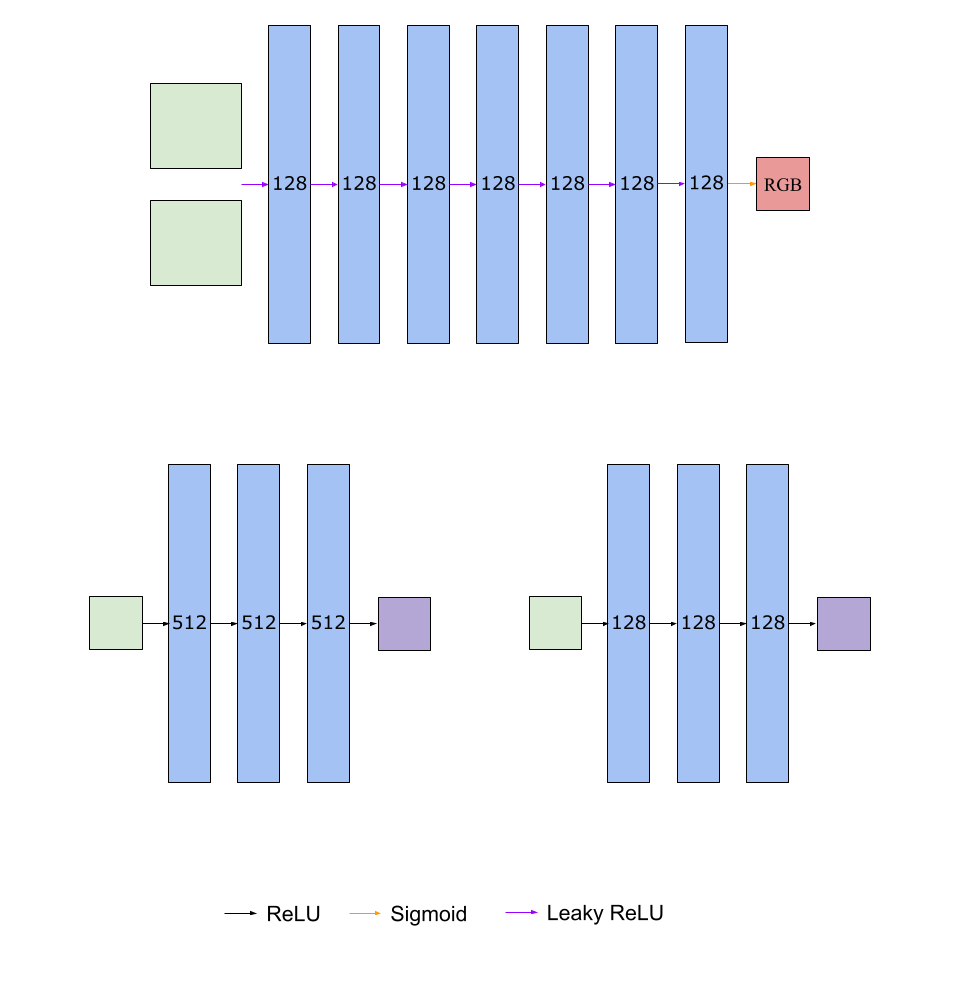}
    \put(19, 88.5){\scriptsize{$\uu$}}
    \put(19.1, 85){\scriptsize{$3$}}
    \put(19, 81){$+$}
    \put(16, 77){\scriptsize{$\gamma_{FF}(\uu)$}}
    \put(18.5, 73){\scriptsize{$60$}}
    \put(45, 61){$F_{\text{tex}}$}
    \put(24, 17){$F_{\text{uv}}^{-1}$}
    \put(68, 17){$F_{\text{uv}}$}
    \put(11, 38.5){\scriptsize{$\uu$}}
    \put(11.25, 36){\scriptsize{$3$}}
    \put(39.75, 38.5){\scriptsize{$\xx$}}
    \put(39.9, 36){\scriptsize{$3$}}
    \put(55, 38.5){\scriptsize{$\xx$}}
    \put(55.1, 36){\scriptsize{$3$}}
    \put(83.8, 38.5){\scriptsize{$\uu$}}
    \put(84, 36){\scriptsize{$3$}}
    \end{overpic}
    \caption{Modified network architecture for NeuTex~\cite{DBLP:conf/cvpr/XiangXHHS021}. We remove the volume density as well as the view dependence enabling NeuTex to learn a simpler setting because the geometry is known and the textures are diffuse in the experiment of \refmainpaper{Sec. 5.1}{sec:app:texture_representation}. A 3D coordinate on the sphere representing the uv-space is described by $\uu$. The symbol $\xx$ is a 3D world coordinate on the surface of the given 2-manifold. The notation for the sine/cosine positional encoding~\cite{DBLP:conf/eccv/MildenhallSTBRN20} is $\gamma_{\text{FF}}$. The $+$ sign denotes concatenation.}
    \label{fig:neutex_network_arch}
\end{figure}

In order to make the comparison between our method and NeuTex~\cite{DBLP:conf/cvpr/XiangXHHS021} fair, we adapt the latter one to the setting of a given geometry.
Since the geometry is known in our experimental setup, we remove the latent vector used for learning a representation of the geometry.
Additionally, we remove the volume density from NeuTex because our experiments are focused on learning a function on the surface of a given 2-manifold.
Furthermore, we do not incorporate view dependence in the experiments of \refmainpaper{Sec. 5.1}{sec:app:texture_representation}.
Hence, we use the the provided, non-view dependent MLP architecture for $F_{\text{tex}}$ from the official \href{https://github.com/fbxiang/NeuTex}{Github repository}\footnote{https://github.com/fbxiang/NeuTex}.
We, additionally, add a sigmoid non-linearity to the last linear layer to ensure that the RGB color values are in $[0,1]$.
The overall architecture is shown in \autoref{fig:neutex_network_arch}.

Due to the higher complexity of additionally learning a mapping between the surface of the manifold and the uv-space, we pretrain the mapping networks $F_{\text{uv}}$ and $F_{\text{uv}}^{-1}$.
In each training iteration, we randomly sample $N = 25,000$ points from the sphere and map them into the 3D world coordinate space of the geometry using $F_{\text{uv}}^{-1}$.
Then, we map the predicted 3D world points back onto the sphere using $F_{\text{uv}}$.
We train for $200,000$ iterations using the Adam optimizer with default parameters ($\beta_1 = 0.9$, $\beta_2 = 0.999$, $\epsilon = 10^{-8}$) and a learning rate of $0.0001$.
The loss function is a combination of the mean Chamfer distance $\mathcal{L}_{\text{chamfer}}$ between the predicted 3D world points and the vertices of the mesh and the mean 2D-3D-2D cycle loss between the sampled and predicted uv-points $u_i$
\begin{align} \label{eqn:2d-3d-2d-cycle-loss}
\mathcal{L}_{\text{cycle}} = \frac{1}{N}\sum_{i=0}^{N-1} \| F_{\text{uv}} ( F_{\text{uv}}^{-1} ( u_i ) ) - u_i \|^2_2.
\end{align}

After the pretraining, we employ a combination of the rendering loss and the 3D-2D-3D cycle loss as the loss function for learning a surface function on a given geometry:
\begin{align} \label{eqn:neutex-train-loss}
\mathcal{L}_{\text{neutex}} = \frac{1}{|\mathcal{B}|}\sum_{\mathbf{p} \in \mathcal{B}} \| F_{\text{tex}} ( F_{\text{uv}} ( \mathbf{p} ) ) - c_{gt} \|_2^2 + \|F_{\text{uv}}^{-1} ( F_{\text{uv}} ( \mathbf{p} ) ) - \mathbf{p} \|^2_2.
\end{align}

For the experiment in \refmainpaper{Sec. 5.1}{sec:app:texture_representation}, we, additionally, increase the embedding size of the sphere coordinates to $1023$ for NeuTex, so that it is similar to the other methods.
For the sine/cosine positional encoding, we select the frequency bands from $[0, 6]$ linearly spaced because it covers a range that is similar to RFF with $\sigma = 8$.

\subsection{Texture Reconstruction}

\begin{figure}[tb]
\centering
\begin{subfigure}[b]{0.19\textwidth}
\centering
\includegraphics[width=\textwidth]{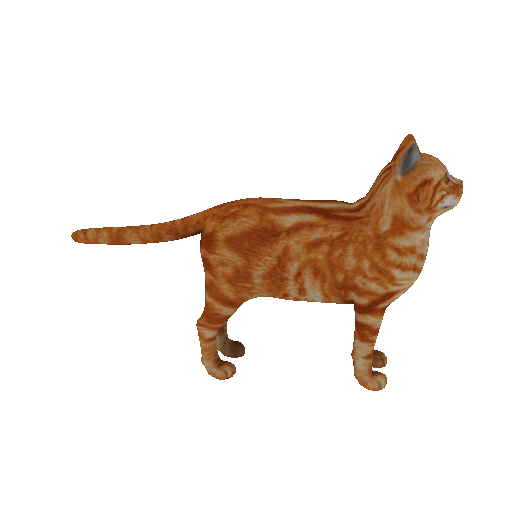}
\end{subfigure}
\hfill
\begin{subfigure}[b]{0.19\textwidth}
\centering
\includegraphics[width=\textwidth]{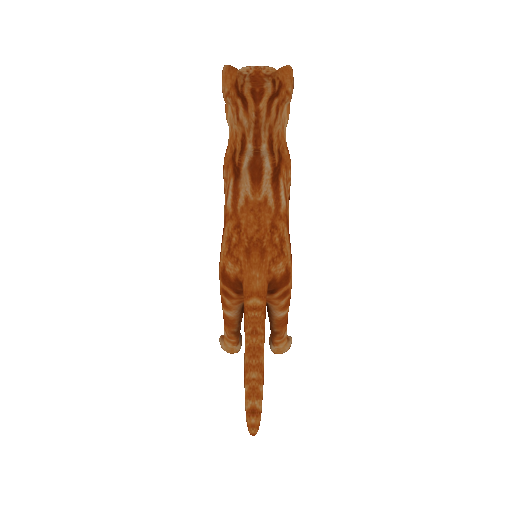}
\end{subfigure}
\hfill
\begin{subfigure}[b]{0.19\textwidth}
\centering
\includegraphics[width=\textwidth]{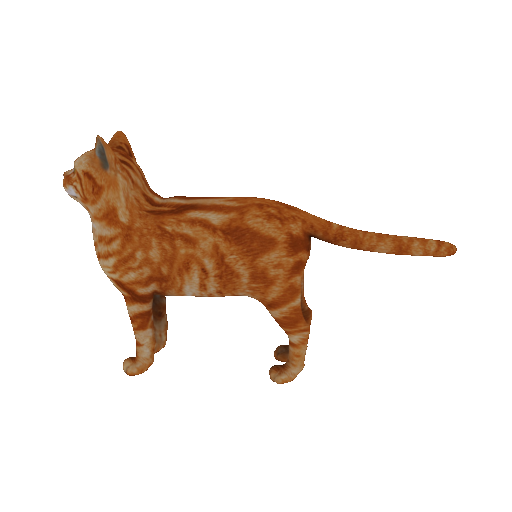}
\end{subfigure}
\hfill
\begin{subfigure}[b]{0.19\textwidth}
\centering
\includegraphics[width=\textwidth]{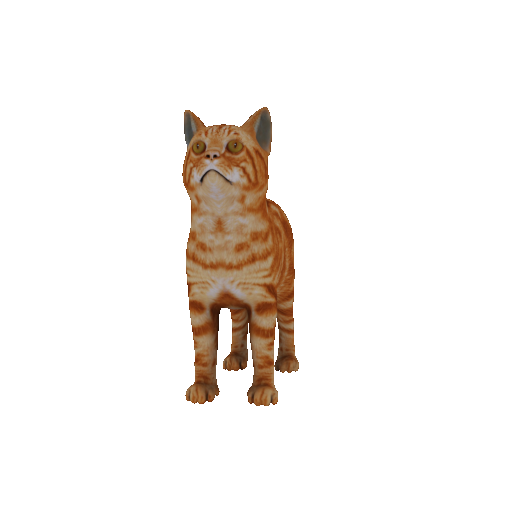}
\end{subfigure}
\hfill
\begin{subfigure}[b]{0.19\textwidth}
\centering
\includegraphics[width=\textwidth]{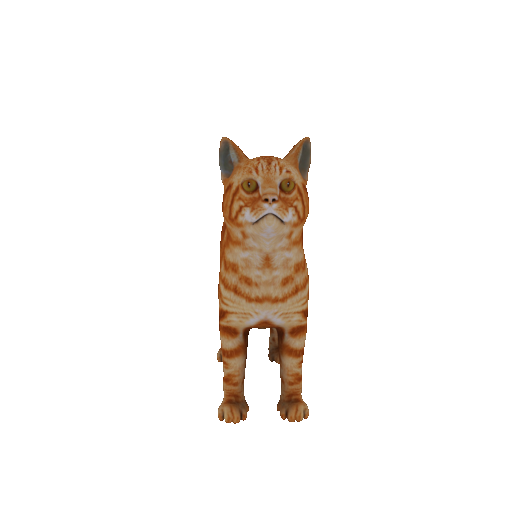}
\end{subfigure}
\\[1.1ex]
\begin{subfigure}[b]{0.19\textwidth}
\centering
\includegraphics[width=\textwidth]{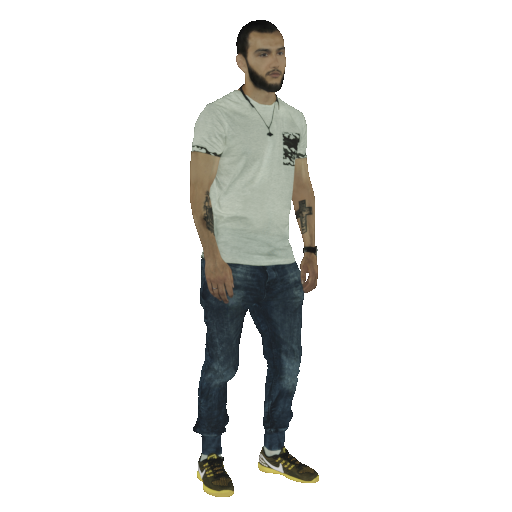}
\end{subfigure}
\hfill
\begin{subfigure}[b]{0.19\textwidth}
\centering
\includegraphics[width=\textwidth]{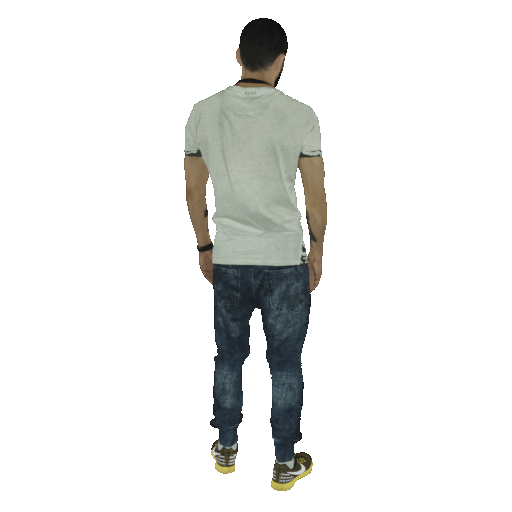}
\end{subfigure}
\hfill
\begin{subfigure}[b]{0.19\textwidth}
\centering
\includegraphics[width=\textwidth]{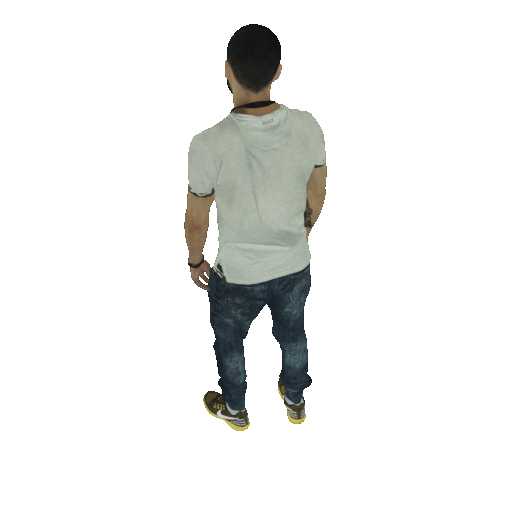}
\end{subfigure}
\hfill
\begin{subfigure}[b]{0.19\textwidth}
\centering
\includegraphics[width=\textwidth]{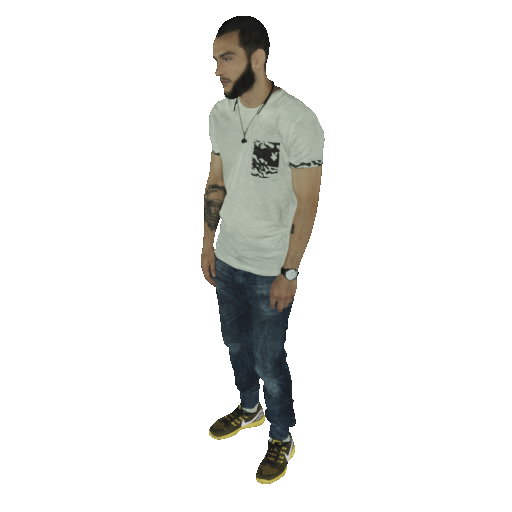}
\end{subfigure}
\hfill
\begin{subfigure}[b]{0.19\textwidth}
\centering
\includegraphics[width=\textwidth]{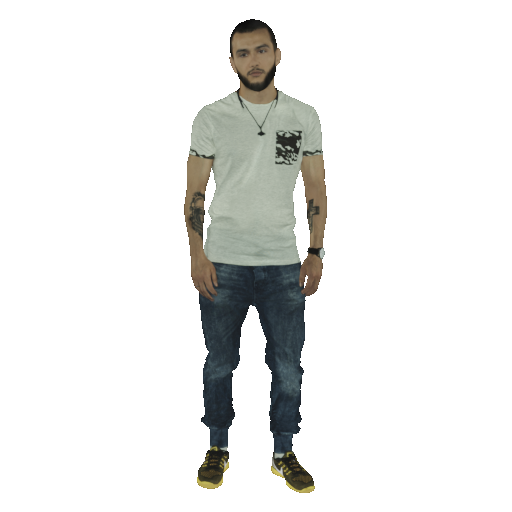}
\end{subfigure}
\caption{Training views for the cat and human datasets with 5 views used in \refmainpaper{Sec. 5.1}{sec:app:texture_representation}, \refmainpaper{Sec. 5.2}{sec:app:geometry_representation}, and \refmainpaper{Sec. 5.3}{sec:app:texture_transfer}.}
\label{fig:training_views_5_views}
\end{figure}

\begin{table}[tb]
\small
\caption{Hyperparameter table: Texture reconstruction.}
\label{tab:hyperparameters_texture_recon}
\begin{center}
\renewcommand{\arraystretch}{1.5}
\begin{tabular}{p{3.2cm} p{7cm}}
    \toprule
    Hyperparameter &
    Value \\
    
    \midrule
    \texttt{optimizer} & Adam with default parameters ($\beta_1 = 0.9$, $\beta_2 = 0.999$, $\epsilon= 10^{-8}$) \\
    \texttt{learning rate} & 0.0001 \\
    \texttt{batch size} & 4096 \\
    \texttt{image size} & $512 {\times} 512$ \\
    \texttt{random seed} & 0 \\
    \texttt{eigenfunctions} & 1 to 256, 1794 to 2304, 3841 to 4096 where 0 is the constant eigenfunction \\
    \texttt{epochs} & 1000 \\
    \bottomrule
\end{tabular}
\end{center}
\end{table}

\begin{figure}[tb]
\centering
\begin{subfigure}[b]{0.19\textwidth}
\centering
\includegraphics[width=\textwidth]{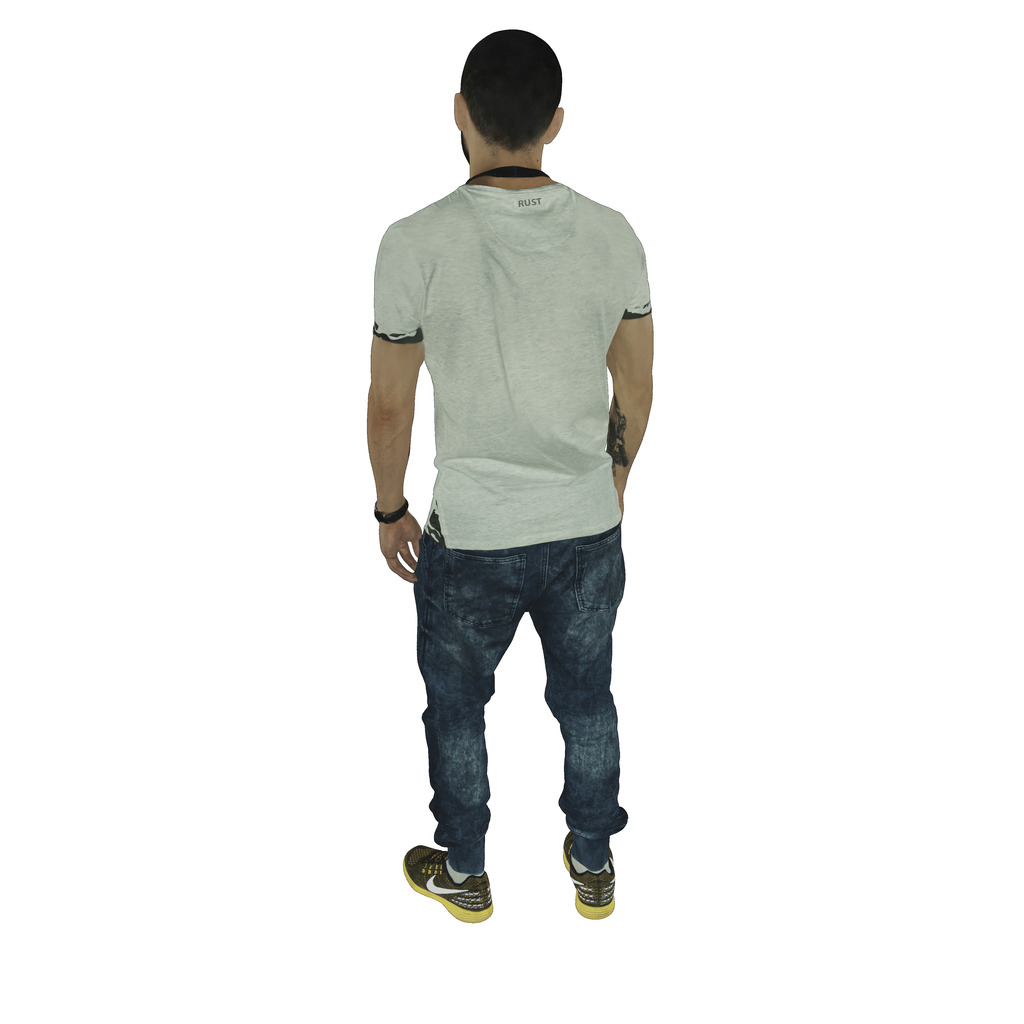}
\end{subfigure}
\hfill
\begin{subfigure}[b]{0.19\textwidth}
\centering
\includegraphics[width=\textwidth]{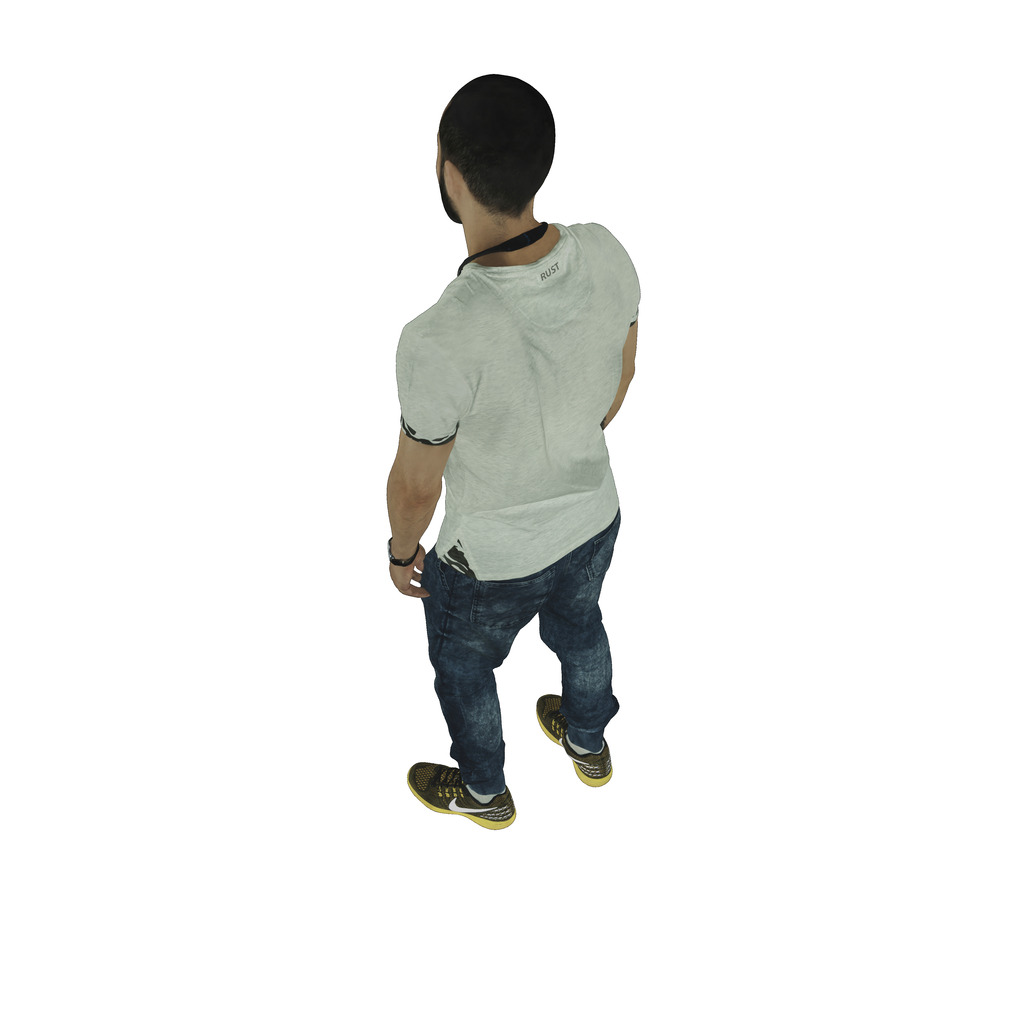}
\end{subfigure}
\hfill
\begin{subfigure}[b]{0.19\textwidth}
\centering
\includegraphics[width=\textwidth]{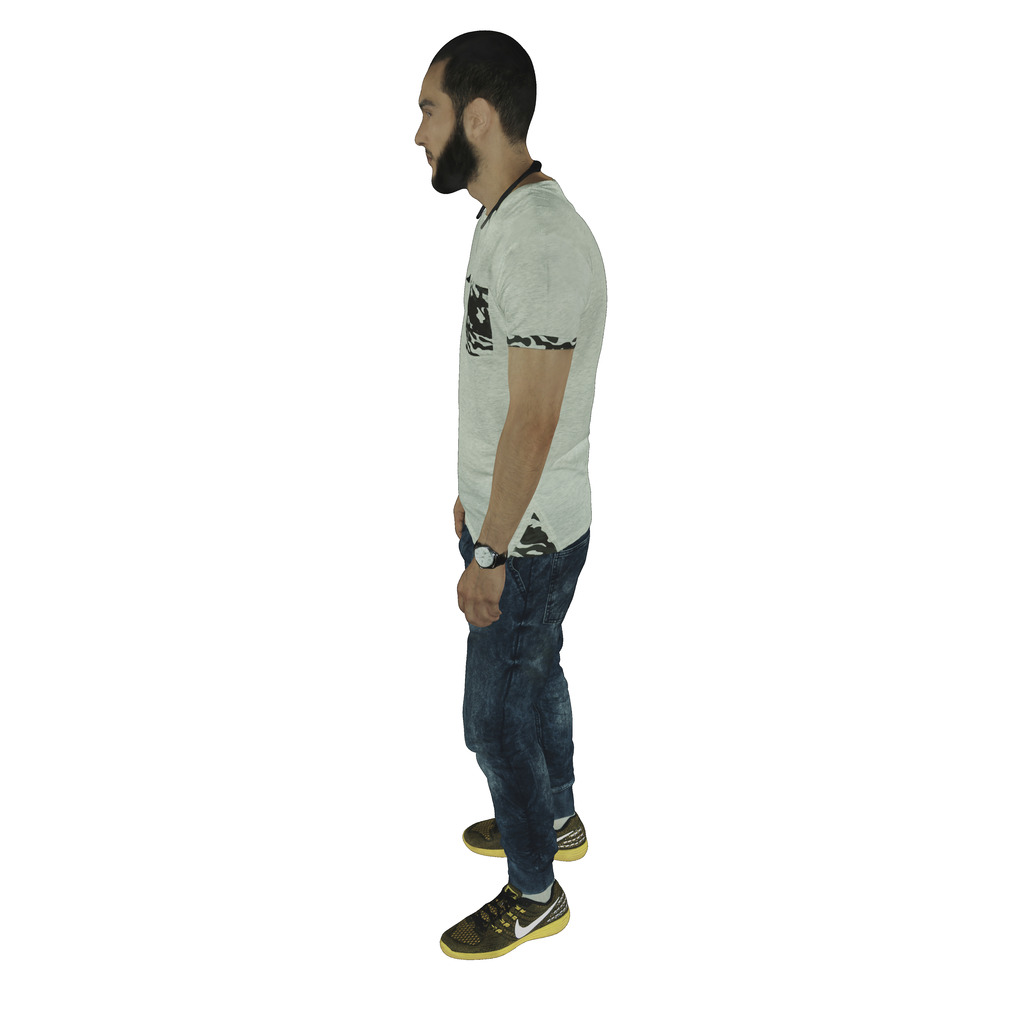}
\end{subfigure}
\hfill
\begin{subfigure}[b]{0.19\textwidth}
\centering
\includegraphics[width=\textwidth]{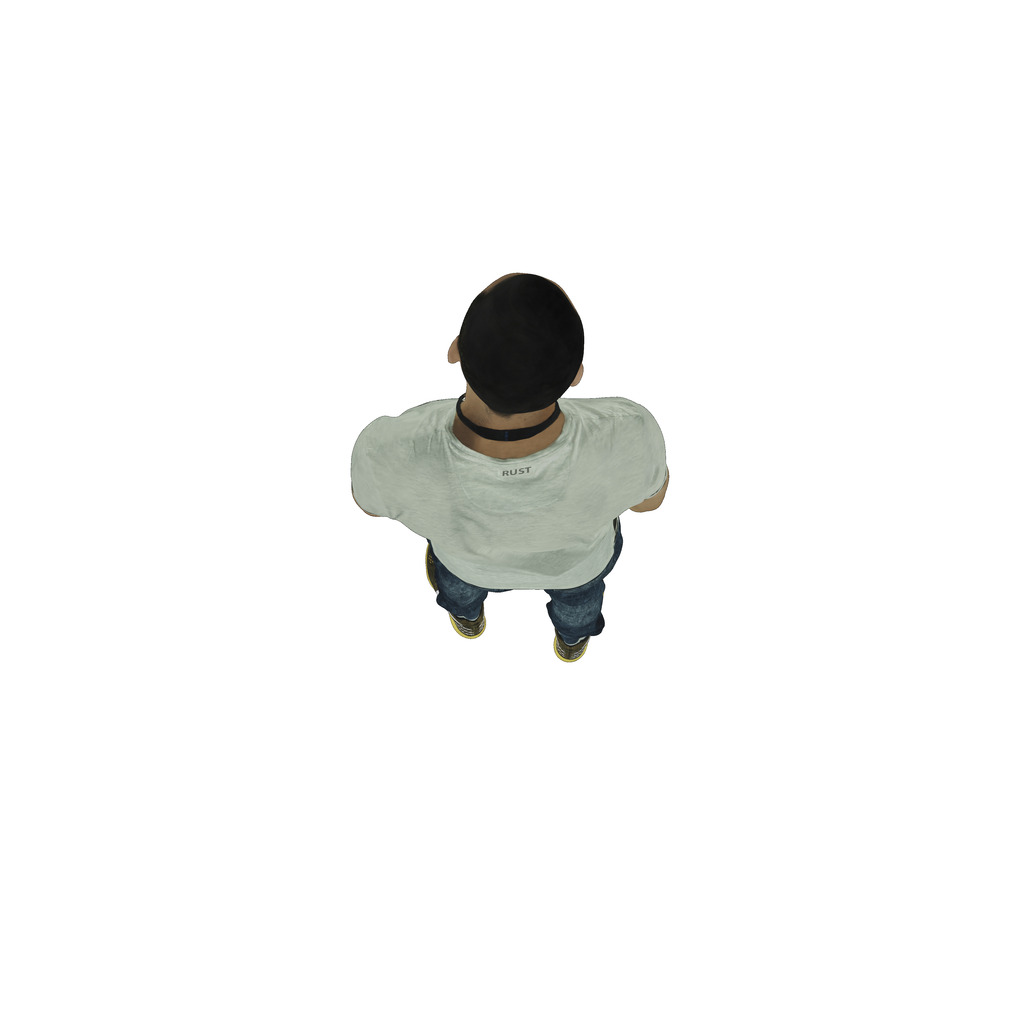}
\end{subfigure}
\hfill
\begin{subfigure}[b]{0.19\textwidth}
\centering
\includegraphics[width=\textwidth]{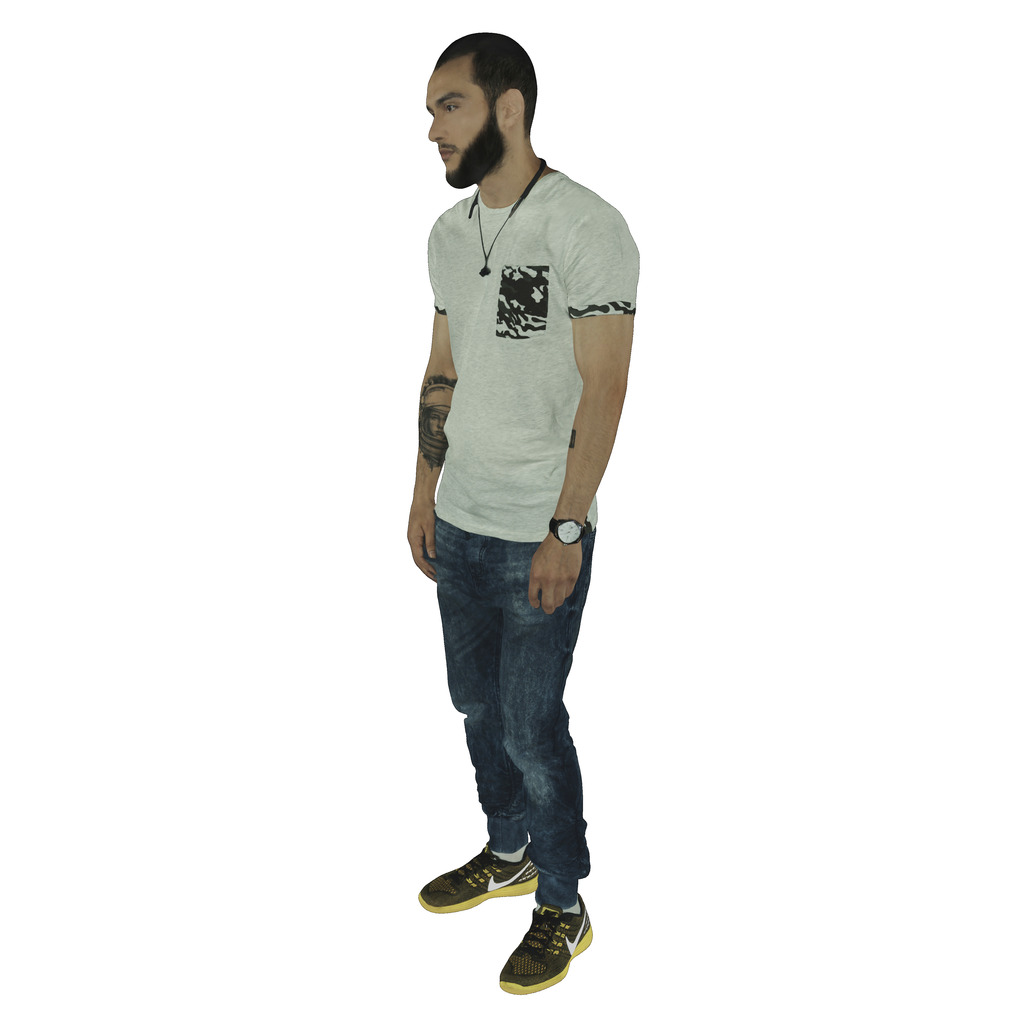}
\end{subfigure}
\\[1.1ex]
\begin{subfigure}[b]{0.19\textwidth}
\centering
\includegraphics[width=\textwidth]{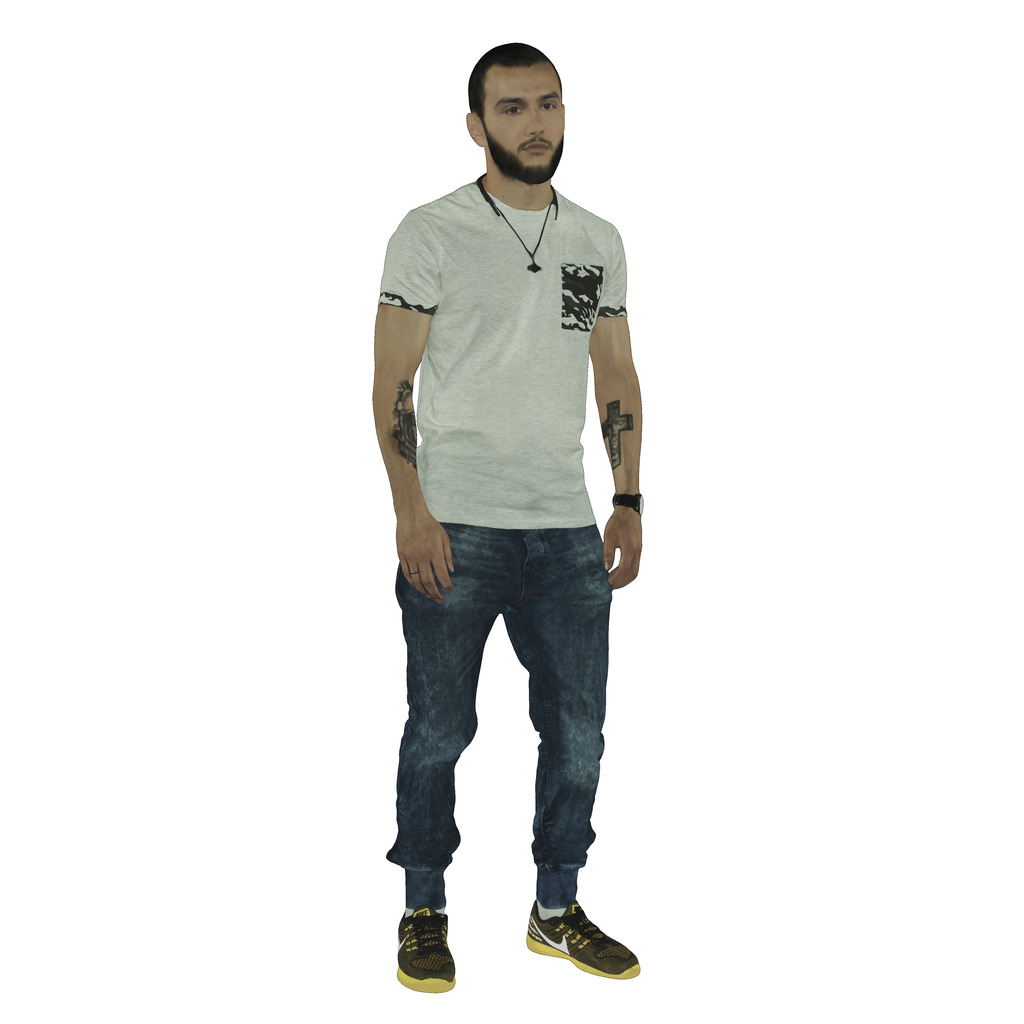}
\end{subfigure}
\hfill
\begin{subfigure}[b]{0.19\textwidth}
\centering
\includegraphics[width=\textwidth]{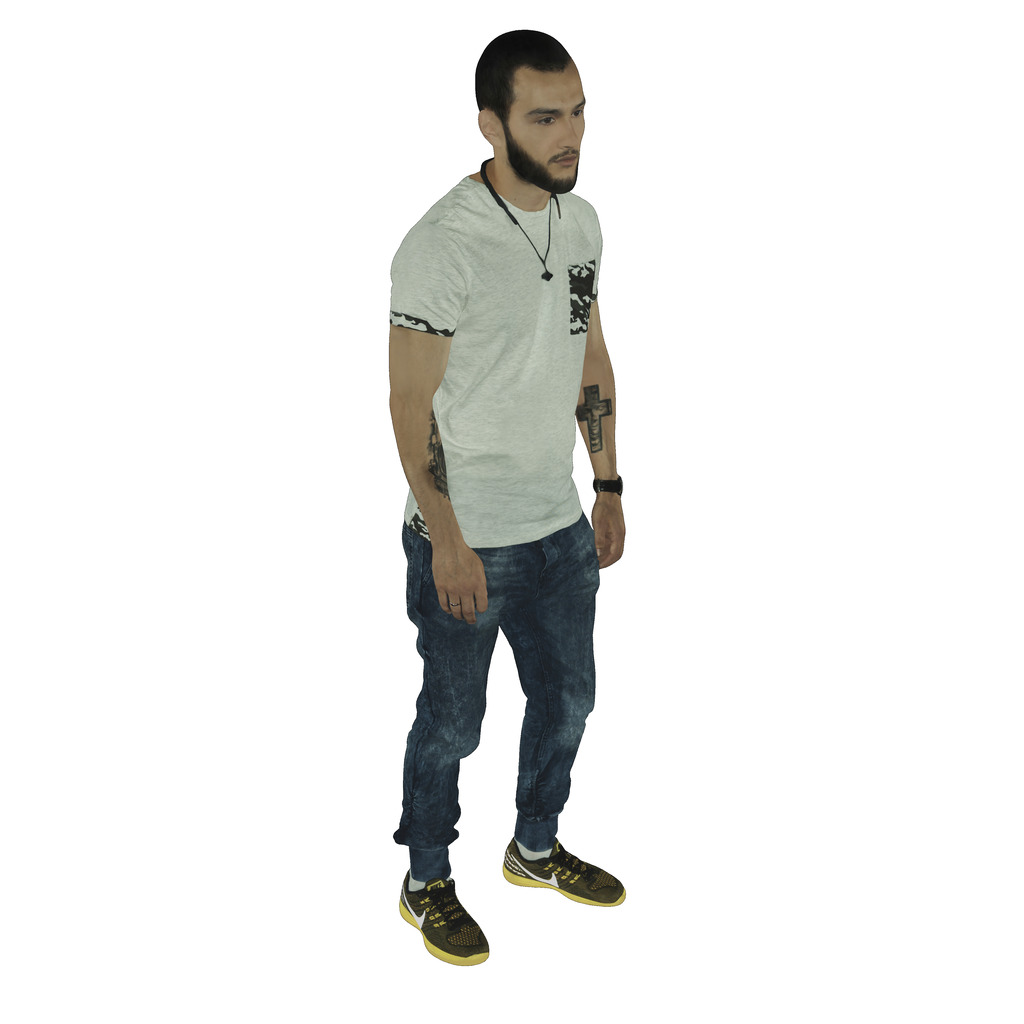}
\end{subfigure}
\hfill
\begin{subfigure}[b]{0.19\textwidth}
\centering
\includegraphics[width=\textwidth]{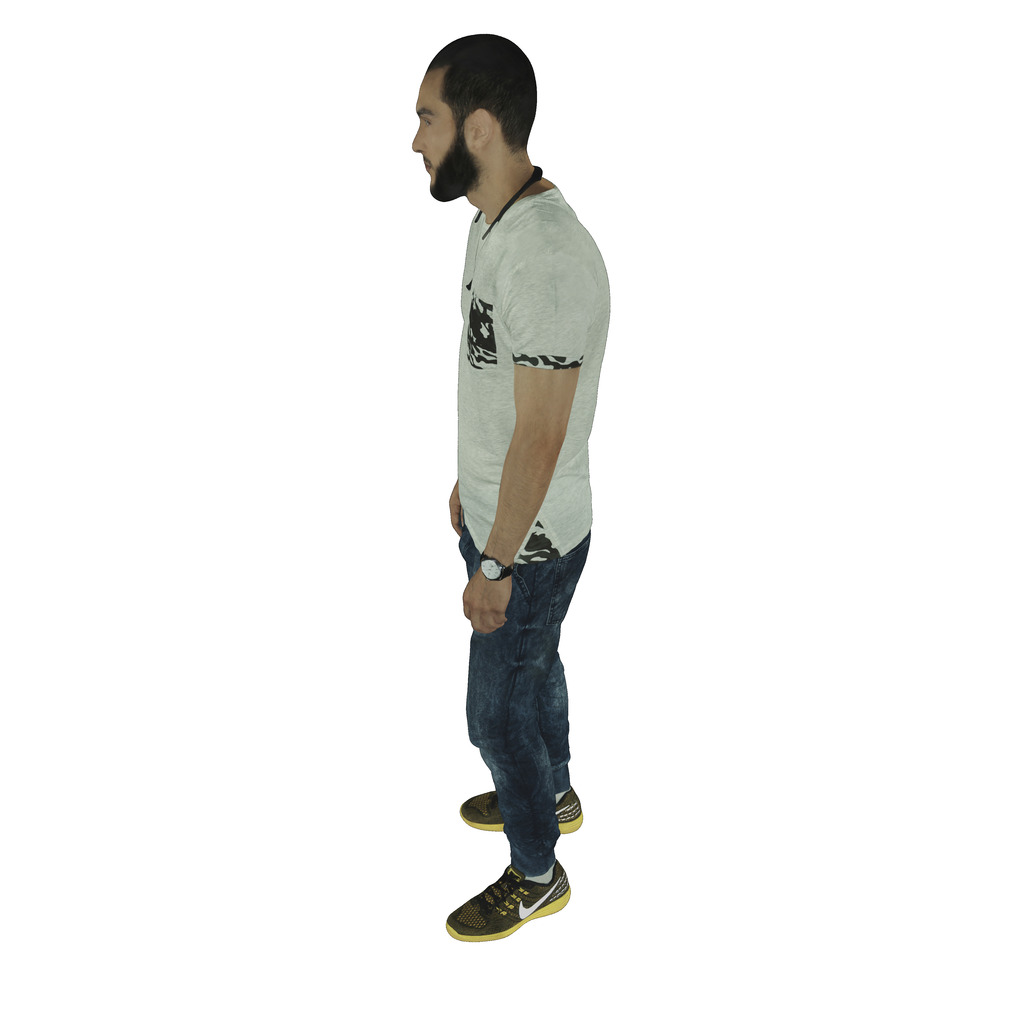}
\end{subfigure}
\hfill
\begin{subfigure}[b]{0.19\textwidth}
\centering
\includegraphics[width=\textwidth]{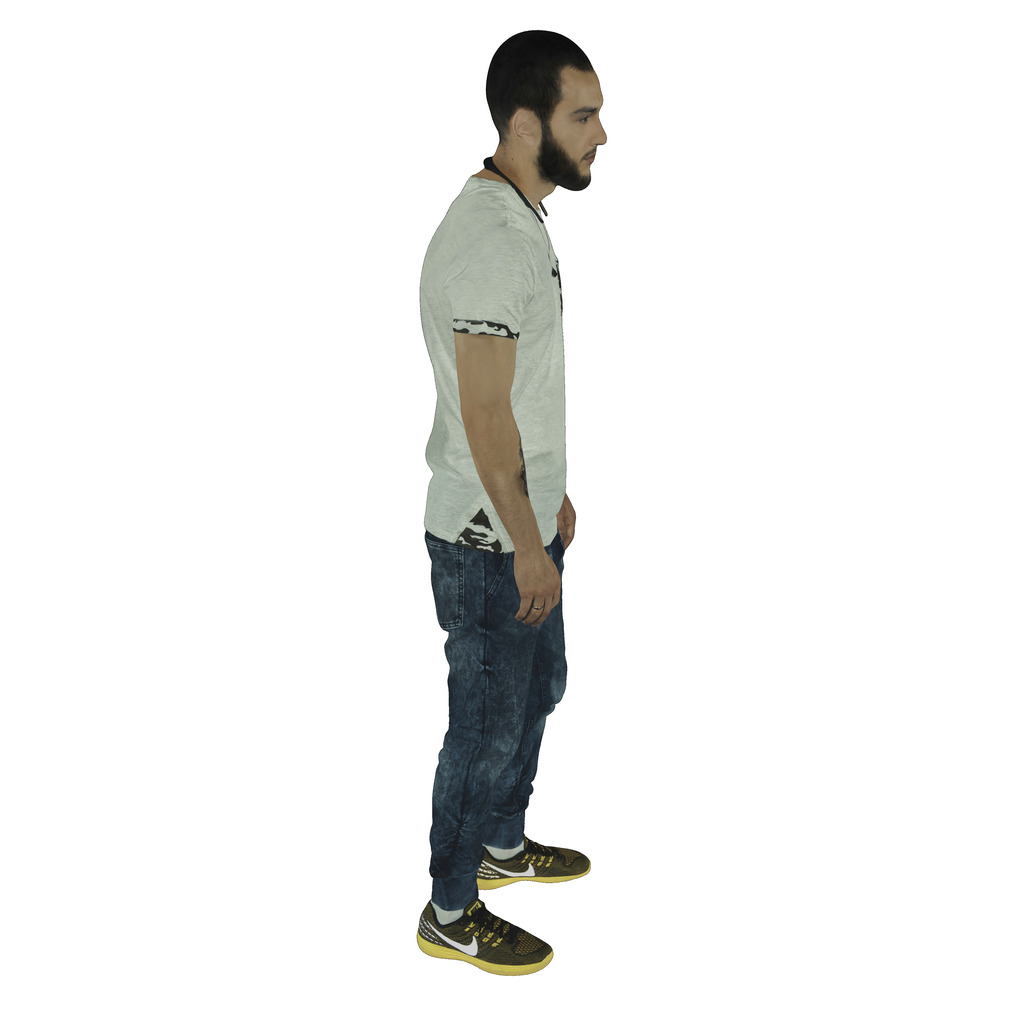}
\end{subfigure}
\hfill
\begin{subfigure}[b]{0.19\textwidth}
\centering
\includegraphics[width=\textwidth]{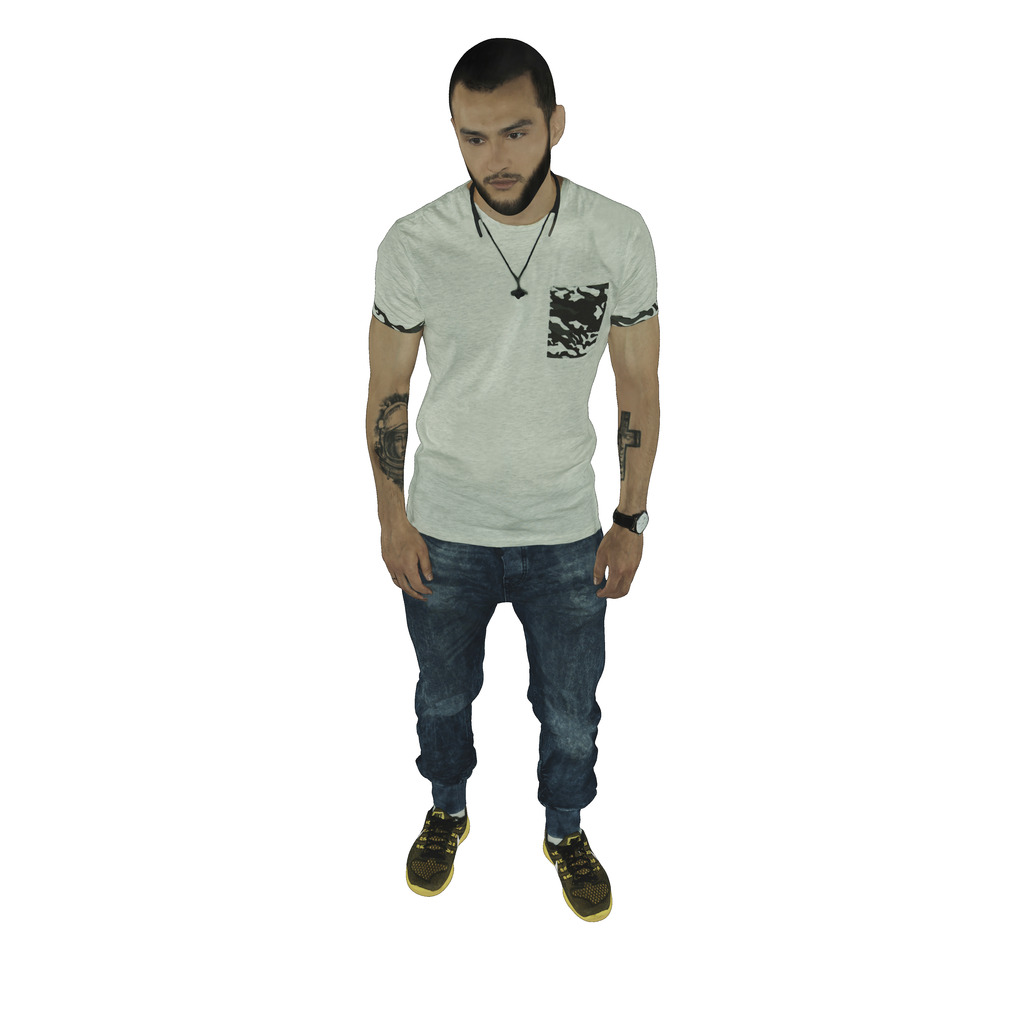}
\end{subfigure}
\\[1.1ex]
\begin{subfigure}[b]{0.19\textwidth}
\centering
\includegraphics[width=\textwidth]{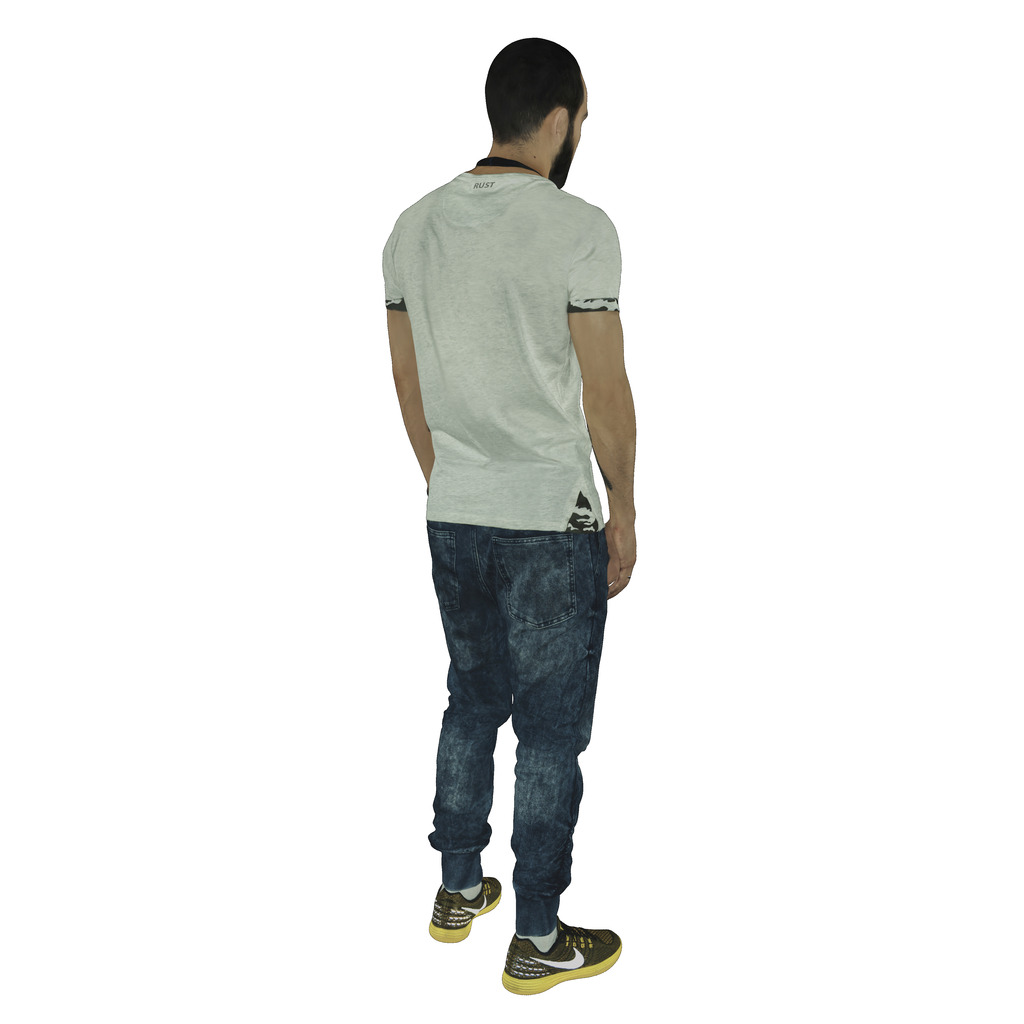}
\end{subfigure}
\hfill
\begin{subfigure}[b]{0.19\textwidth}
\centering
\includegraphics[width=\textwidth]{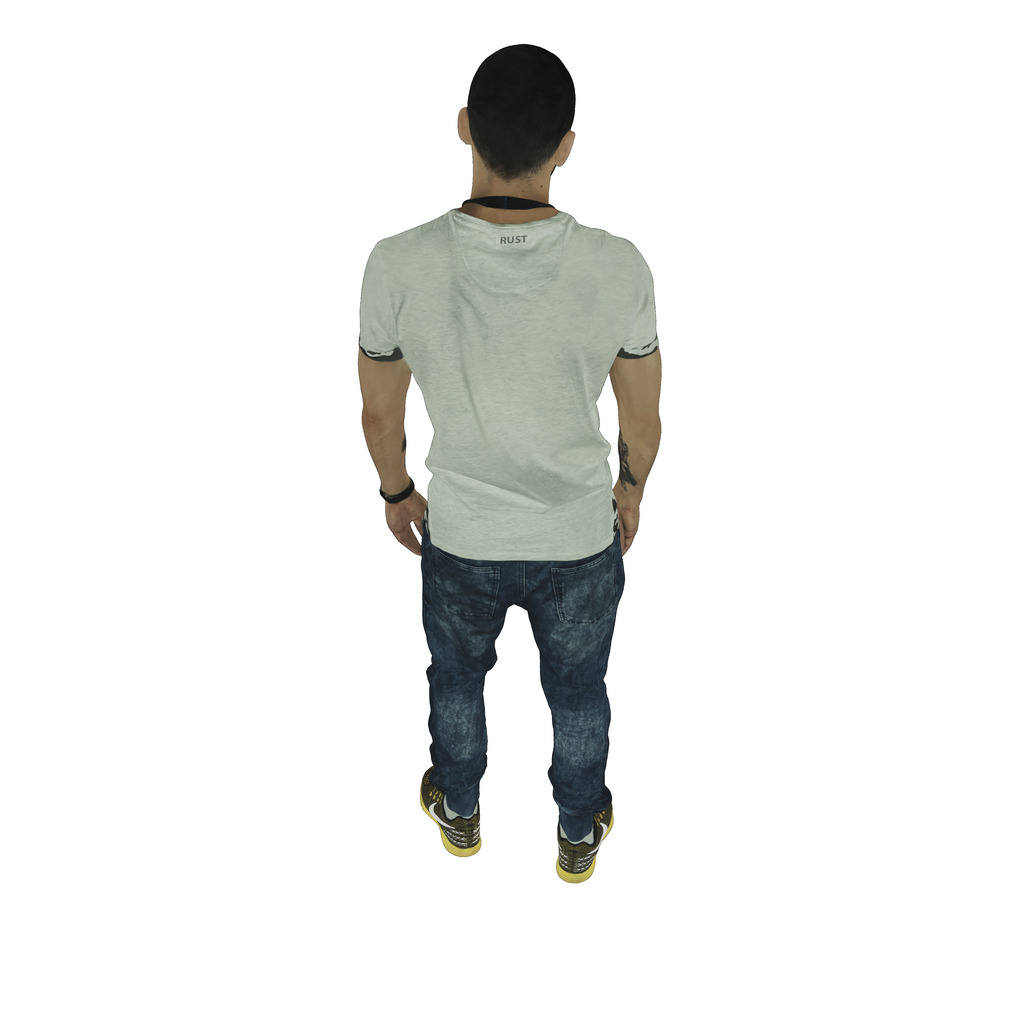}
\end{subfigure}
\hfill
\begin{subfigure}[b]{0.19\textwidth}
\centering
\includegraphics[width=\textwidth]{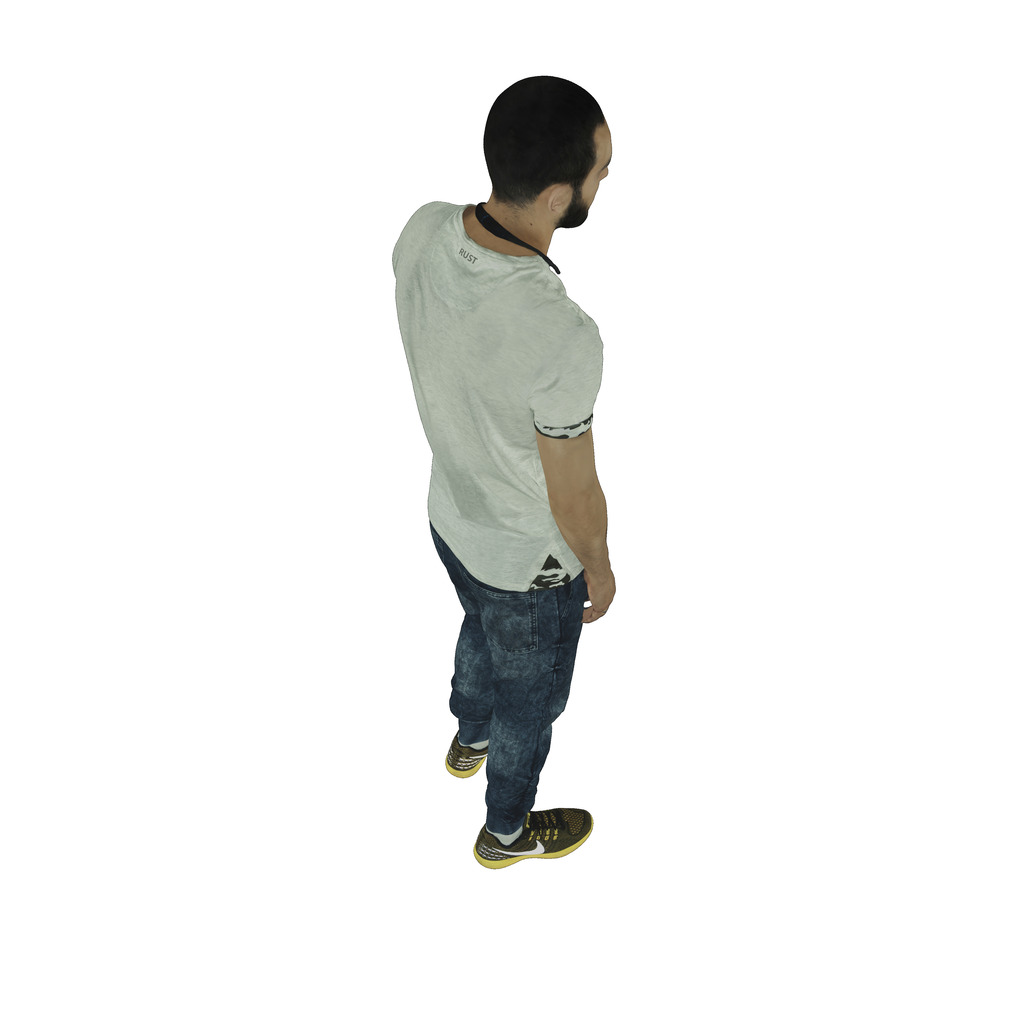}
\end{subfigure}
\hfill
\begin{subfigure}[b]{0.19\textwidth}
\centering
\includegraphics[width=\textwidth]{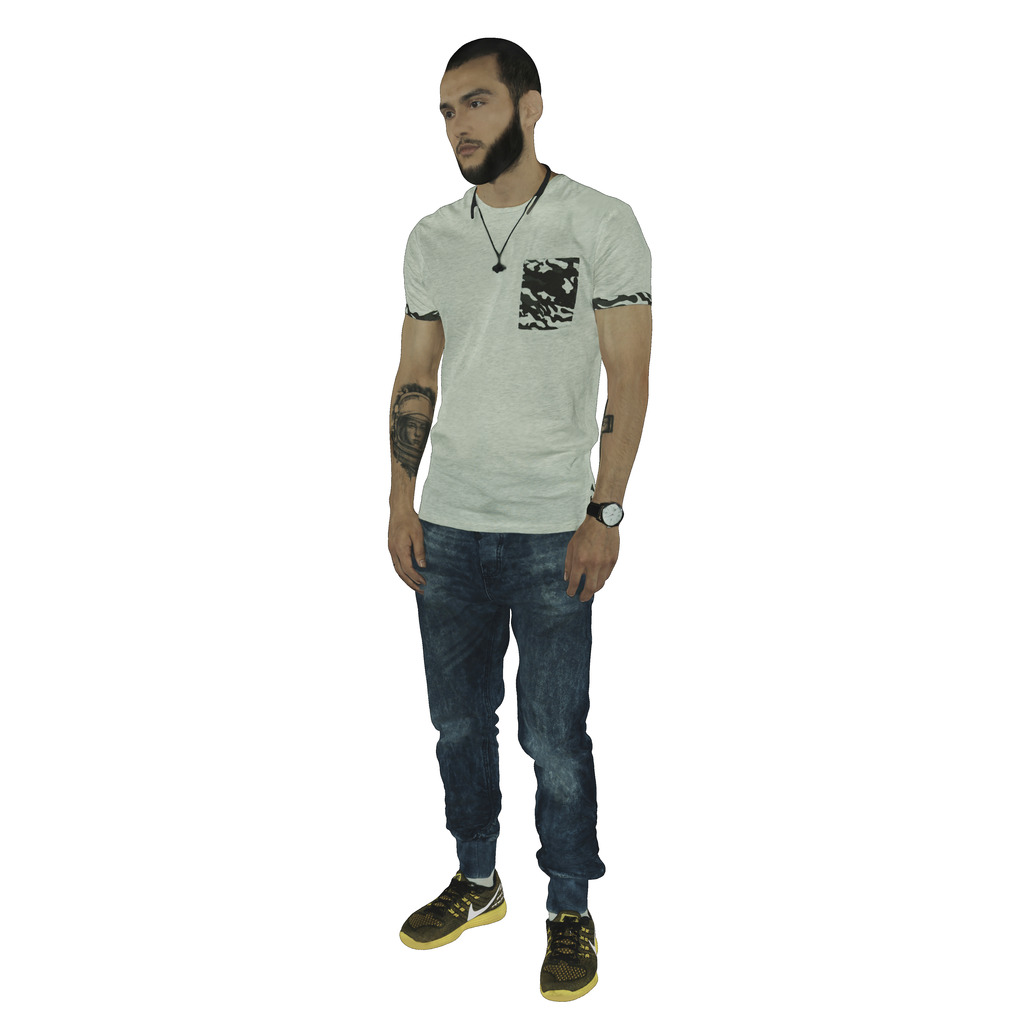}
\end{subfigure}
\hfill
\begin{subfigure}[b]{0.19\textwidth}
\centering
\includegraphics[width=\textwidth]{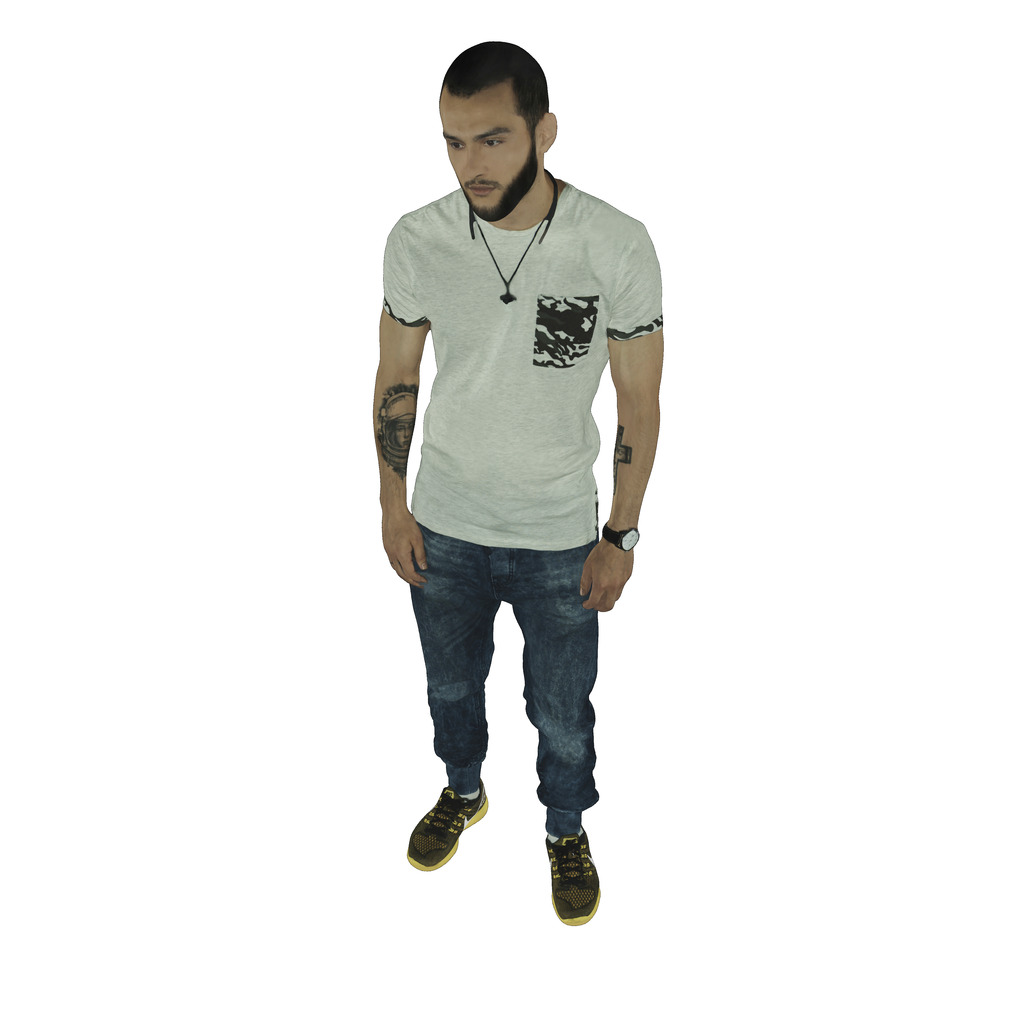}
\end{subfigure}
\\[1.1ex]
\begin{subfigure}[b]{0.19\textwidth}
\centering
\includegraphics[width=\textwidth]{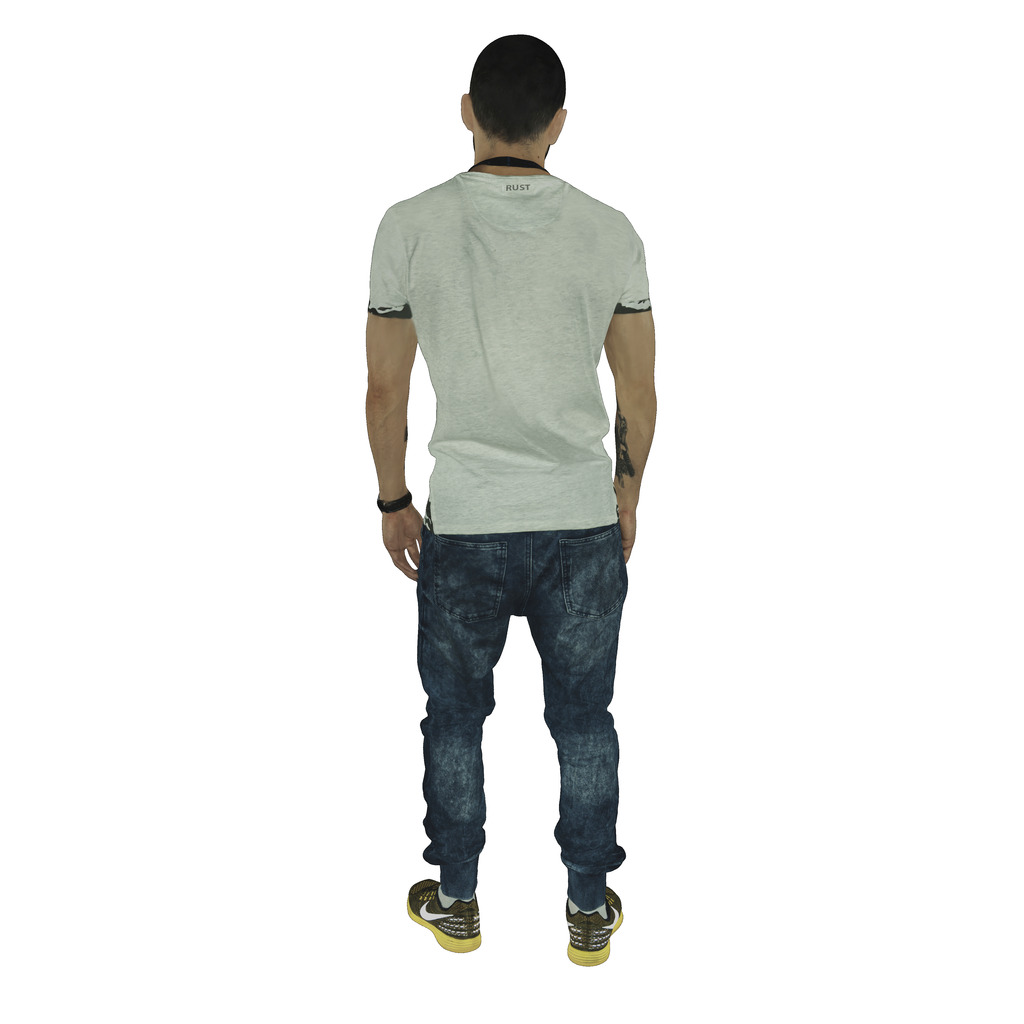}
\end{subfigure}
\hfill
\begin{subfigure}[b]{0.19\textwidth}
\centering
\includegraphics[width=\textwidth]{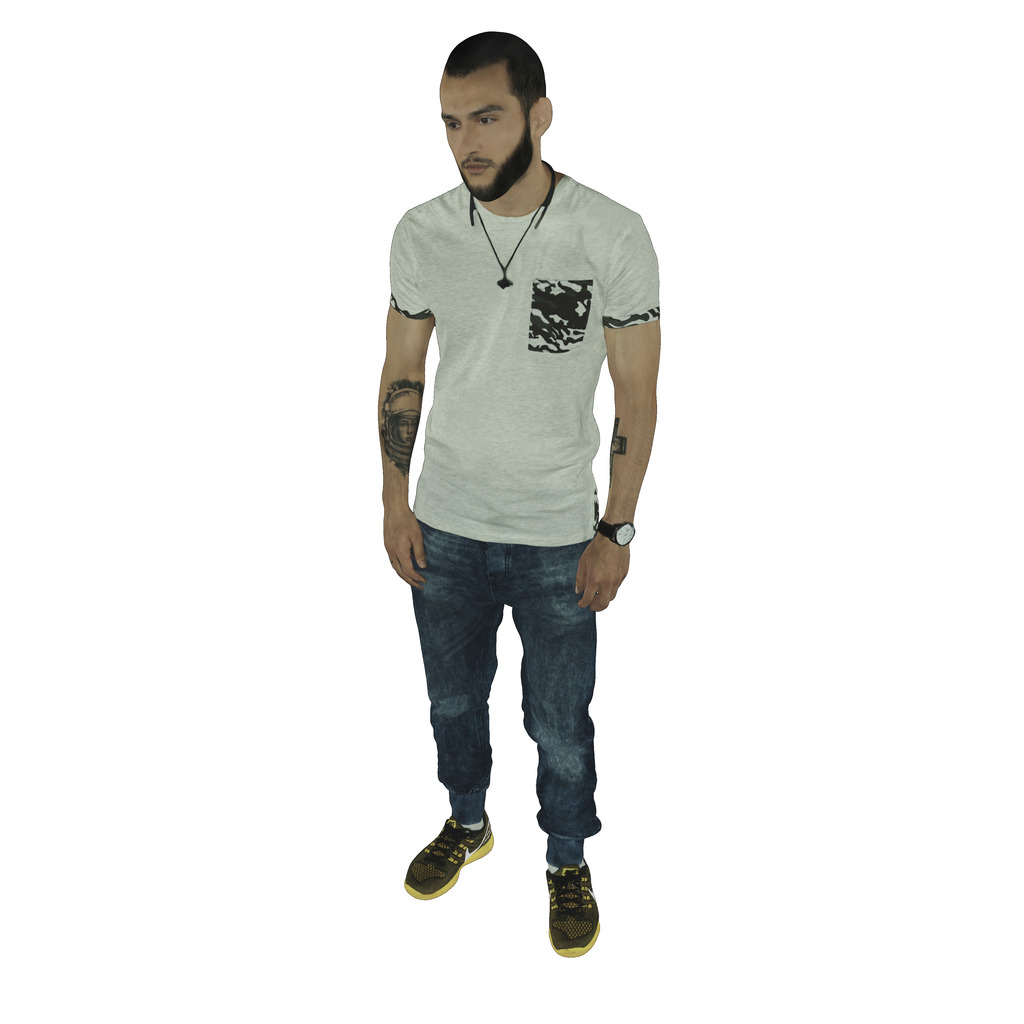}
\end{subfigure}
\hfill
\begin{subfigure}[b]{0.19\textwidth}
\centering
\includegraphics[width=\textwidth]{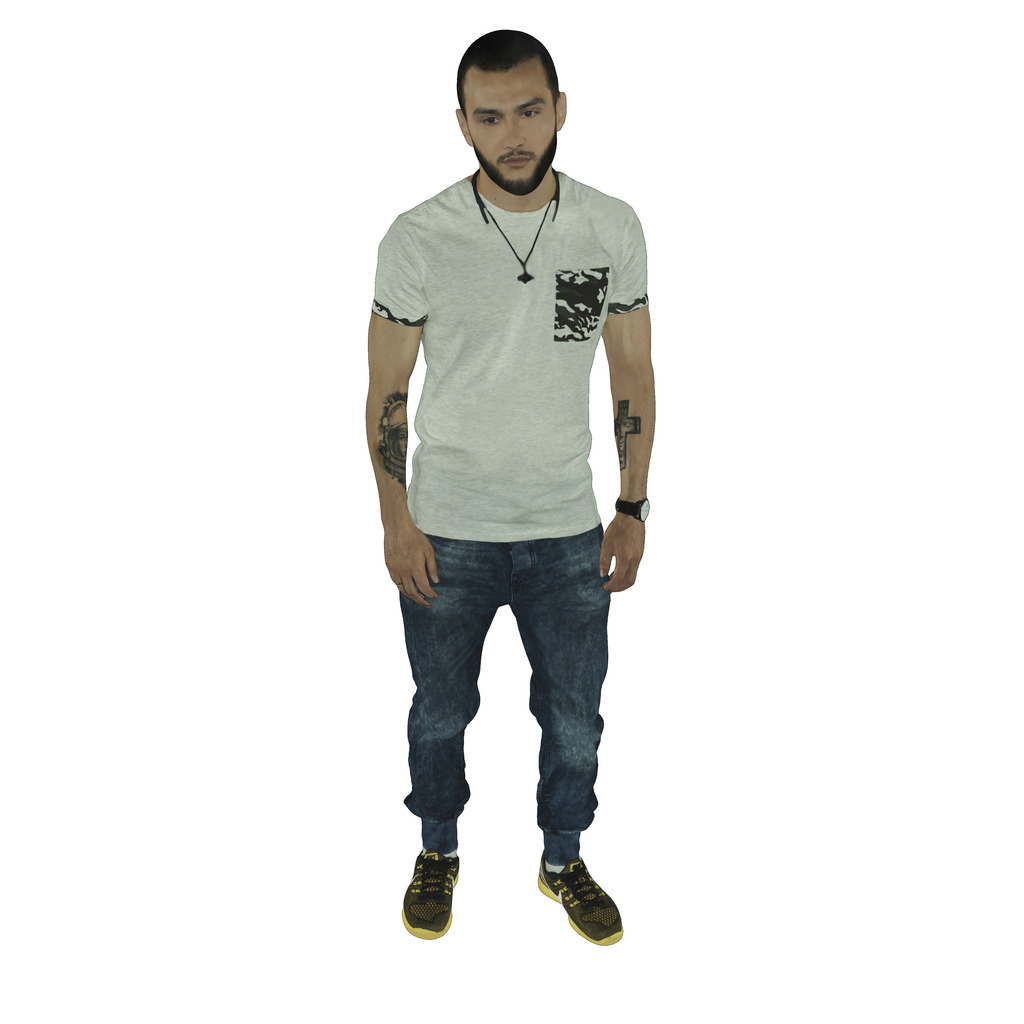}
\end{subfigure}
\hfill
\begin{subfigure}[b]{0.19\textwidth}
\centering
\includegraphics[width=\textwidth]{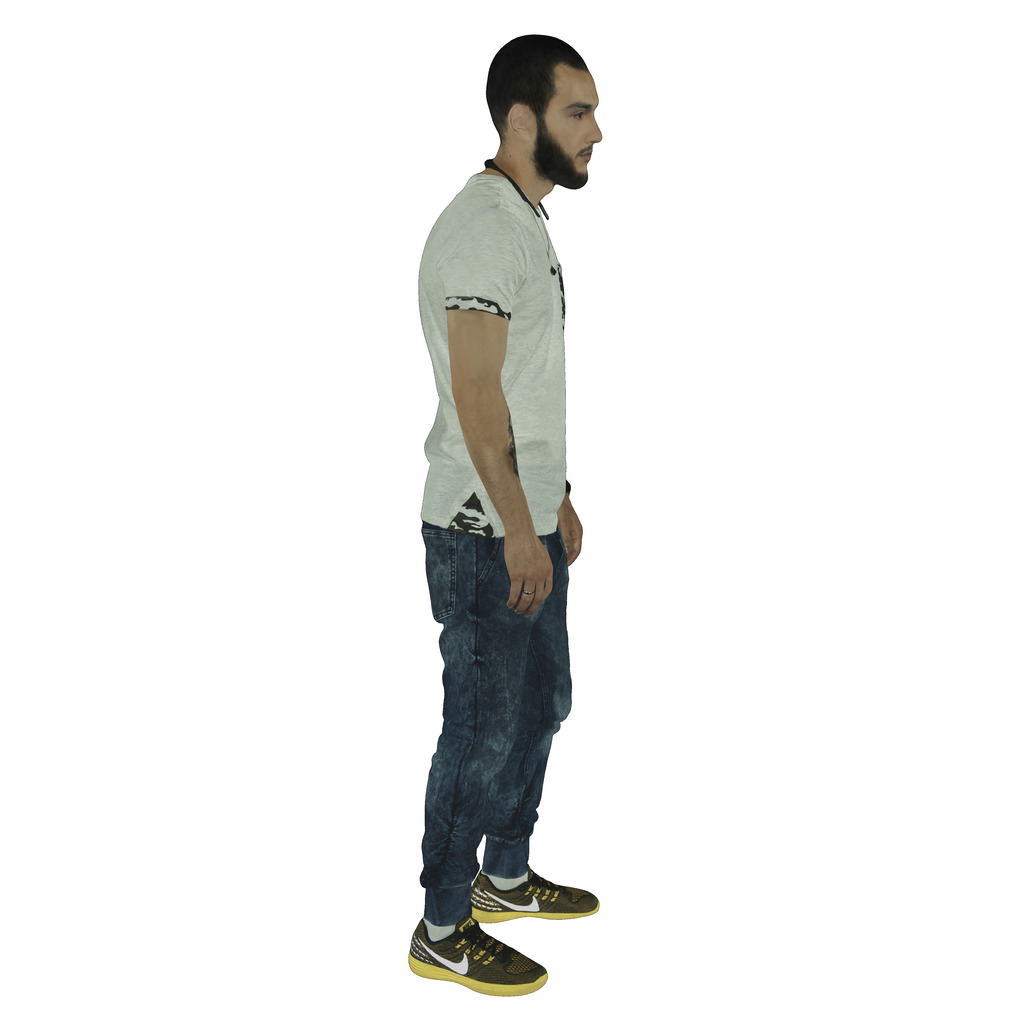}
\end{subfigure}
\hfill
\begin{subfigure}[b]{0.19\textwidth}
\centering
\includegraphics[width=\textwidth]{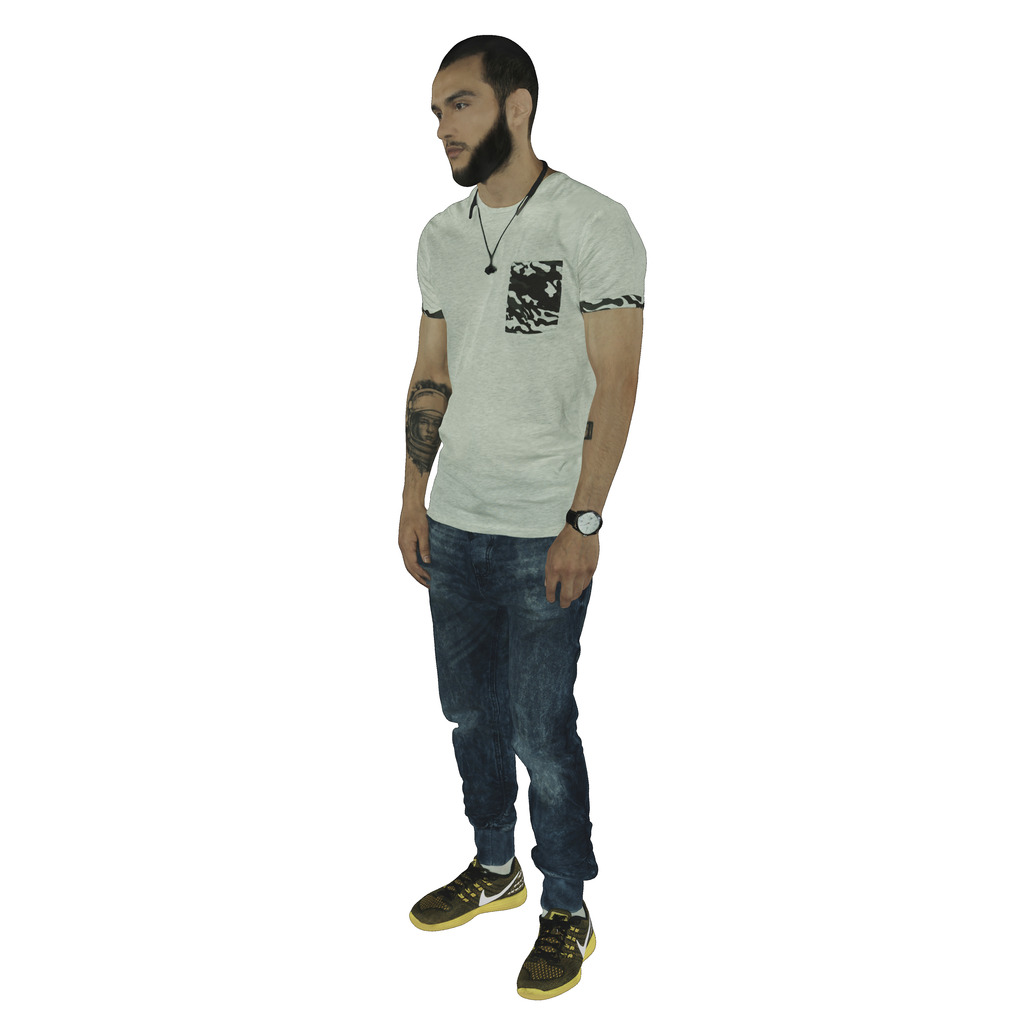}
\end{subfigure}
\caption{Training views for the human high resolution dataset. Due to the large size of the $4096 {\times} 4096$ png images, we converted them to jpg and scaled them down to $1024 {\times} 1024$ for this figure.}
\label{fig:training_views_human_high_res}
\end{figure}

\begin{table}[tb]
\small
\caption{Hyperparameter table: High-resolution texture reconstruction.}
\label{tab:hyperparameters_texture_recon_high_res}
\begin{center}
\renewcommand{\arraystretch}{1.5}
\begin{tabular}{p{3.2cm} p{7cm}}
    \toprule
    Hyperparameter &
    Value \\
    
    \midrule
    \texttt{optimizer} & Adam with default parameters ($\beta_1 = 0.9$, ${\beta_2 = 0.999}$, $\epsilon= 10^{-8}$) \\
    \texttt{learning rate} & 0.0001 \\
    \texttt{batch size} & 4096 \\
    \texttt{image size} & $4096 {\times} 4096$ \\
    \texttt{random seed} & 0 \\
    \texttt{eigenfunctions} & 1 to 4096 where 0 is the constant eigenfunction \\
    \texttt{epochs} & 500 \\
    \bottomrule
\end{tabular}
\end{center}
\end{table}

\begin{figure}[tb]
\centering
\begin{subfigure}[b]{0.24\textwidth}
\centering
\includegraphics[width=0.63\textwidth]{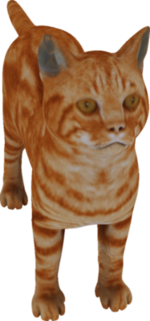}
\caption{\vspace{-0.2cm}NeuTex~\cite{DBLP:conf/cvpr/XiangXHHS021}}
\label{fig:supp:cat_texture_recon_neutex}
\end{subfigure}
\hfill
\begin{subfigure}[b]{0.24\textwidth}
\centering
\includegraphics[width=0.63\textwidth]{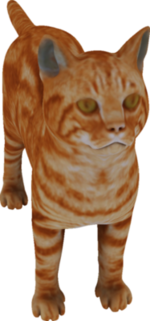}
\caption{\vspace{-0.2cm}TF ($\sigma{=}8$)~\cite{DBLP:conf/iccv/OechsleMNSG19,DBLP:conf/nips/TancikSMFRSRBN20}}
\label{fig:supp:cat_texture_recon_rff8}
\end{subfigure}
\hfill
\begin{subfigure}[b]{0.24\textwidth}
\centering
\includegraphics[width=0.63\textwidth]{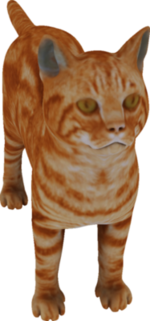}
\caption{\vspace{-0.2cm}Ours}
\label{fig:supp:cat_texture_recon_intrinsic}
\end{subfigure}
\hfill
\begin{subfigure}[b]{0.24\textwidth}
\centering
\includegraphics[width=0.63\textwidth]{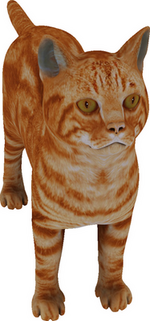}
\caption{\vspace{-0.2cm}GT Image}
\label{fig:supp:cat_texture_recon_gt}
\end{subfigure}
\\[3.0ex]
\begin{subfigure}[b]{0.24\textwidth}
\centering
\includegraphics[width=0.63\textwidth]{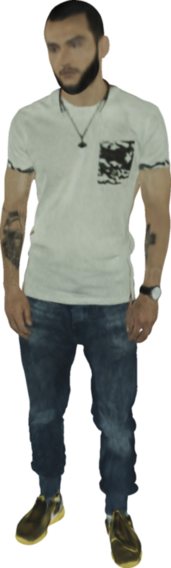}
\caption{\vspace{-0.2cm}NeuTex~\cite{DBLP:conf/cvpr/XiangXHHS021}}
\label{fig:supp:human_texture_recon_neutex}
\end{subfigure}
\hfill
\begin{subfigure}[b]{0.24\textwidth}
\centering
\includegraphics[width=0.63\textwidth]{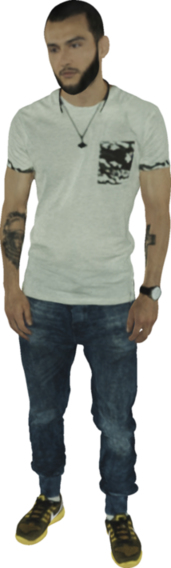}
\caption{\vspace{-0.2cm}TF ($\sigma{=}8$)~\cite{DBLP:conf/iccv/OechsleMNSG19,DBLP:conf/nips/TancikSMFRSRBN20}}
\label{fig:supp:human_texture_recon_rff8}
\end{subfigure}
\hfill
\begin{subfigure}[b]{0.24\textwidth}
\centering
\includegraphics[width=0.63\textwidth]{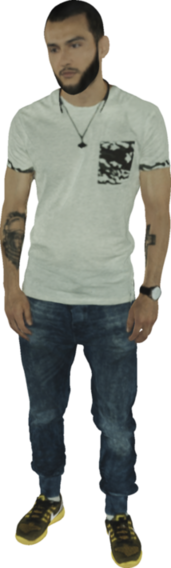}
\caption{\vspace{-0.2cm}Ours}
\label{fig:supp:human_texture_recon_intrinsic}
\end{subfigure}
\hfill
\begin{subfigure}[b]{0.24\textwidth}
\centering
\includegraphics[width=0.63\textwidth]{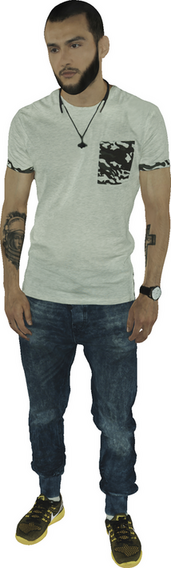}
\caption{\vspace{-0.2cm}GT Image}
\label{fig:supp:human_texture_recon_gt}
\end{subfigure}
\caption{Further qualitative comparisons of unseen views from the texture reconstruction experiment (\refmainpaper{Sec. 5.1}{sec:app:texture_representation}). All the methods from this figure use an embedding size of 1023. We want to point out that these renderings are not high-quality due to the low resolution of the training views.}
\label{fig:supp:texture_recon}
\end{figure}

In this experiment, we use the same \href{https://free3d.com/3d-model/cat-v1--522281.html}{cat}\footnote{https://free3d.com/3d-model/cat-v1--522281.html} and \href{https://www.turbosquid.com/3d-models/water-park-slides-3d-max/1093267}{human}\footnote{https://www.turbosquid.com/3d-models/water-park-slides-3d-max/1093267} mesh as in \cite{DBLP:conf/iccv/OechsleMNSG19}.
The 2D views of the meshes are rendered using Blender and blenderproc \cite{denninger2019blenderproc}. 
For the experiments in \refmainpaper{Sec. 5.1}{sec:app:texture_representation} and \refmainpaper{Sec. 5.2}{sec:app:geometry_representation}, we randomly render 5 training, 100 validation, and 200 test $512{\times}512$ views. 
The training views are visualized in \autoref{fig:training_views_5_views}.
The hyperparameters are given in \autoref{tab:hyperparameters_texture_recon}.
Further qualitative results for \refmainpaper{Sec. 5.1}{sec:app:texture_representation} can be found in \autoref{fig:supp:texture_recon}.
The human shown in \refmainpaper{Fig.~1}{fig:teaser} was trained using a $4096{\times}4096$ high-resolution dataset. 
It consists of 20 training, 20 validation, and 20 test views that were randomly generated.
All training views are visualized in \autoref{fig:training_views_human_high_res}.
In \autoref{tab:hyperparameters_texture_recon_high_res} the hyperparameters for the high-resolution human are shown.

\subsection{Discretization-agnostic Intrinsic Neural Fields}

We use the same datasets for the cat and human as in \refmainpaper{Sec. 5.1}{sec:app:texture_representation} but generate different discretizations of the meshes with the scripts from the \href{https://github.com/nmwsharp/discretization-robust-correspondence-benchmark}{Github project}\footnote{https://github.com/nmwsharp/discretization-robust-correspondence-benchmark} by Sharp et al.~\cite{sharp2021diffusion}.
The hyperparameters can be found in \autoref{tab:hyperparameters_texture_recon}.
The scripts and the meshes will be released together with the code.

\subsection{Intrinsic Neural Field Transfer}

\begin{table}[tb]
\small
\caption{Hyperparameter table: Texture transfer.}
\label{tab:hyperparameters_texture_transfer}
\begin{center}
\renewcommand{\arraystretch}{1.5}
\begin{tabular}{p{3.2cm} p{7cm}}
    \toprule
    Hyperparameter &
    Value \\
    
    \midrule
    \texttt{optimizer} & Adam with default parameters ($\beta_1 = 0.9$, $\beta_2 = 0.999$, $\epsilon= 10^{-8}$) \\
    \texttt{learning rate} & 0.0001 \\
    \texttt{batch size} & 4096 \\
    \texttt{image size} & $512 {\times} 512$ \\
    \texttt{random seed} & 0 \\
    \texttt{eigenfunctions} & 1 to 512 where 0 is the constant eigenfunction \\
    \texttt{epochs} & 500 \\
    \bottomrule
\end{tabular}
\end{center}
\end{table}

For the neural texture transfer, we train an intrinsic neural field on the cat from \refmainpaper{Sec. 5.1}{sec:app:texture_representation} with the hyperparameters shown in \autoref{tab:hyperparameters_texture_transfer}. 
Since the input to our method is only the first $d$ LBO eigenfunctions, we can reuse the network on the cat on other shapes without retraining as long as we know how to transfer the eigenfunctions.
This is exactly what functional maps \cite{ovsjanikov12funmaps} do. 
We use the method of \cite{DBLP:conf/cvpr/EisenbergerLC20}, which works with both isometric and non-isometric pairs, to calculate a correspondence $P$ between the cat $\mathcal{C}$ and a target $\mathcal{T}$, and obtain the functional map by projecting $C = \Phi_{\mathcal{T}}^\top P\Phi_{\mathcal{C}}$. 
Instead of using the eigenfunctions $\Phi_{\mathcal{T}}$ of the target shape directly, we use $\Phi_{\mathcal{T}} C$ as input to the network.

\subsection{Real-world Data and View Dependence}
\begin{table}[tb]
\small
\caption{Hyperparameter table: Real-world data and view dependence.}
\label{tab:hyperparameters_real_world}
\begin{center}
\renewcommand{\arraystretch}{1.5}
\begin{tabular}{p{3.2cm} p{7cm}}
    \toprule
    Hyperparameter &
    Value \\
    
    \midrule
    \texttt{optimizer} & Adam with default parameters ($\beta_1 = 0.9$, $\beta_2 = 0.999$, $\epsilon= 10^{-8}$) \\
    \texttt{learning rate} & 0.0002 \\
    \texttt{learning rate schedule} & plateau with $\text{factor} {=} 0.1$, $\text{patience} {=} 10$, $\text{threshold} {=} 0.0001$ \\
    \texttt{batch size} & 16384 \\
    \texttt{image size} & $2848 {\times} 4272$ \\
    \texttt{random seed} & 0 \\
    \texttt{eigenfunctions} & 1 to 4096 where 0 is the constant eigenfunction \\
    \texttt{epochs} & 500 \\
    
    \bottomrule
\end{tabular}
\end{center}
\end{table}

For experiments with real-world data, we select the objects detergent and cinnamon toast crunch from the BigBIRD dataset \cite{DBLP:conf/icra/SinghSNAA14}.
The dataset provides images from different directions, foreground-background segmentation  masks, camera calibration, and a reconstructed mesh of the geometry of each object.
The provided meshes and the object masks do not align well, which can cause color bleeding from the background into the reconstruction.
Hence, we improve the masks using intensity thresholding and morphological operations. Specifically, potential background pixels are identified based on their intensity due to the mostly white background. They are removed if they are close to the boundary of the initial mask. Finally, a small margin of the mask is eroded to limit the number of false positive mask pixels. This process could be replaced by a more advanced method for a large-scale experiment on real-world data.
For both objects, we train our method on 60 evenly-spaced views from a 360 degree perspective. The view used for qualitative comparison is centered between two training views.
The model for this experiment implements the view-dependent network architecture visualized in \autoref{fig:intrinsic_network_arch} and uses the hyperparameters shown in \autoref{tab:hyperparameters_real_world}.
We will release the preprocessing code and the training data used in this experiment along with the final publication.

\section{Initialization Dependence of Intrinsic Neural Fields} \label{sup:sec:init_dependence}

\begin{table}[tb]
\notsotiny
\caption{Initialization dependence for intrinsic neural fields. We retrain our method with different seeds on the texture reconstruction experiment from \refmainpaper{Sec.~5.1}{sec:app:texture_representation}. The results show that intrinsic neural fields are not overly dependent on the initialization. Additionally, the results presented in \refmainpaper{Sec.~5.1}{sec:app:texture_representation} with seed 0 are not cherry picked.}
\label{tab:intrinsic_init_dependence}
\begin{center}
\renewcommand{\arraystretch}{1.4}
\setlength{\tabcolsep}{2.9pt}

\begin{tabular}{p{1cm} p{1cm} c c c c c c c c c c c a}

\toprule

&
Seed &
0 &
1 &
2 &
3 &
4 &
5 &
6 &
7 &
8 &
9 &
Avg. &
Std. \\

\midrule

\multirow{3}{*}{cat} & PSNR $\uparrow$ & 34.82                          & 34.73                          & 34.80                          & 34.81                          & 34.85                          & 34.95                          & 34.87                          & 35.00                          & 34.77                          & \textBF{35.01}                 & 34.86 & 0.262\% \\
& DSSIM $\downarrow$ &  0.095                          & 0.096                          & 0.096                          & 0.096                          & 0.093                          & 0.094                          & 0.093                          & 0.092                          & 0.096                          & \textBF{0.091}                 & 0.094 & 1.846\% \\
& LPIPS $\downarrow$ & 0.153                          & 0.155                          & 0.156                          & 0.158                          & 0.153                          & 0.158                          & 0.154                          & \textBF{0.151}                 & 0.156                          & 0.156                          & 0.155 & 1.305\% \\

\midrule

\multirow{3}{*}{human} & PSNR $\uparrow$ & 32.48                          & 32.43                          & 32.35                          & 32.47                          & 32.47                          & \textBF{32.56}                 & 32.50                          & 32.50                          & 32.42                          & 32.42                          & 32.46 & 0.171\% \\
& DSSIM $\downarrow$ & 0.121                          & 0.122                          & 0.124                          & 0.121                          & 0.121                          & \textBF{0.120}                 & 0.121                          & 0.121                          & 0.121                          & 0.122                          & 0.121 & 0.843\% \\
& LPIPS $\downarrow$ & 0.306                          & 0.309                          & 0.305                          & 0.309                          & 0.305                          & 0.306                          & 0.306                          & \textBF{0.301}                 & 0.303                          & 0.305                          & 0.306 & 0.763\% \\

\bottomrule

\end{tabular}

\end{center}
\end{table}

In order to test the dependence of our method on the initialization, we repeat the experiment from \refmainpaper{Sec. 5.1}{sec:app:texture_representation} with different seeds.
The results can be found in \autoref{tab:intrinsic_init_dependence}.
Our method has a relative standard deviation of roughly 1\% for DSSIM and LPIPS and only about $0.2$\% for PSNR, which shows that intrinsic neural fields are not overly initialization-dependent.

\section{Theory} \label{sup:sec:theory}
In this section, we provide further details regarding the theory of intrinsic neural fields. In \autoref{fig:suppl_kernel_stat} we demonstrate that the composed neural tangent kernel (NTK) can be non-stationary for a general 2-manifold. In \autoref{sec:theory_thm_man1} the proof for \refmainpaper{Theorem~1}{thm:stationarity} is given specifically for 1-manifolds. Finally, in \autoref{sec:theory_ntk} we provide further theoretical results regarding the NTK, specifically \autoref{thm:ntk_of_normed}, which was used in the proof of \refmainpaper{Theorem~1}{thm:stationarity}.

\begin{figure}[tb]
\centering
\begin{subfigure}[b]{0.31\textwidth}
\centering
\includegraphics[width=\textwidth]{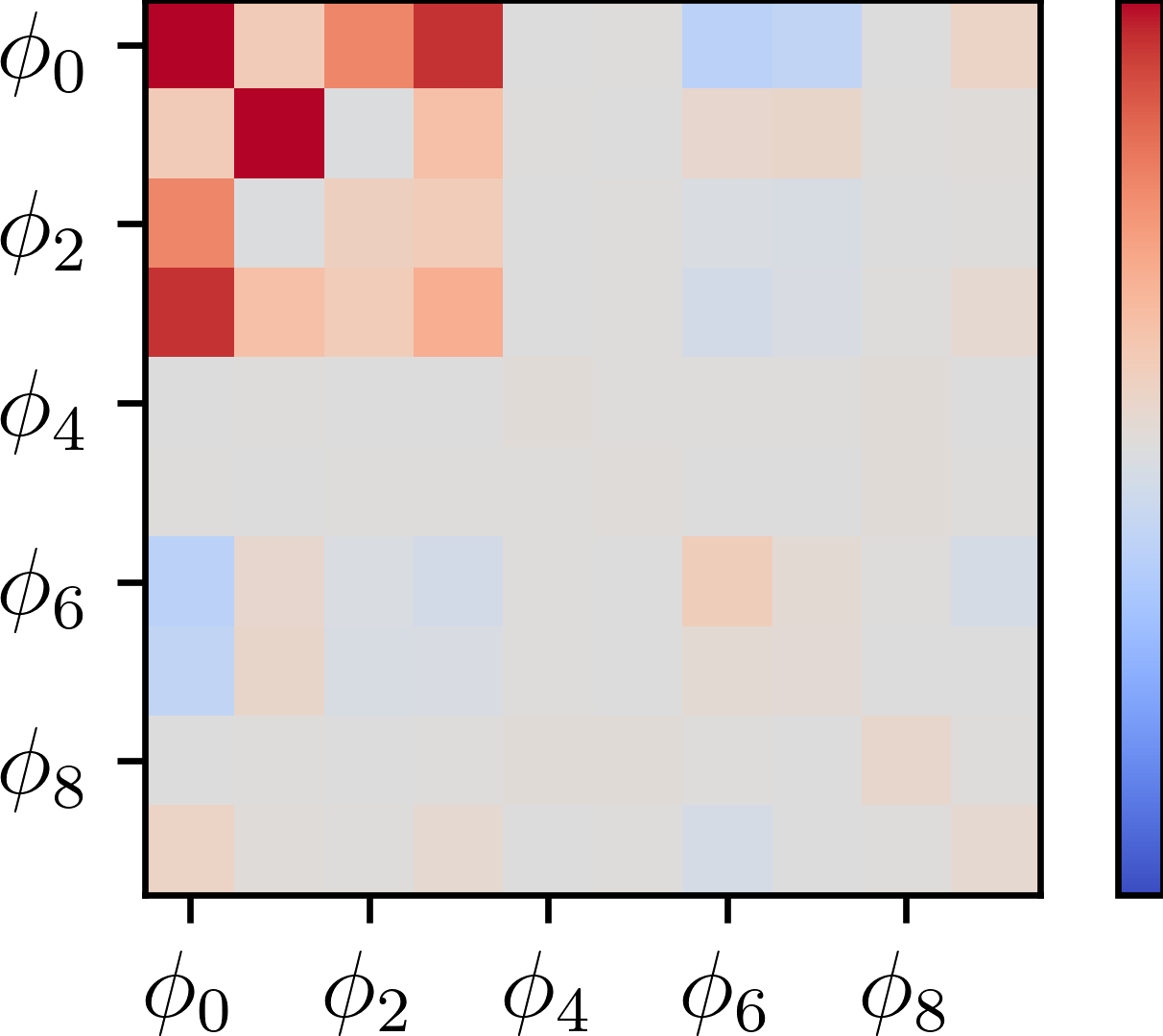}
\caption{\vspace{-0.2cm}XY}
\label{fig:suppl_kernel_stat_xy}
\end{subfigure}
\hfill
\begin{subfigure}[b]{0.31\textwidth}
\centering
\includegraphics[width=\textwidth]{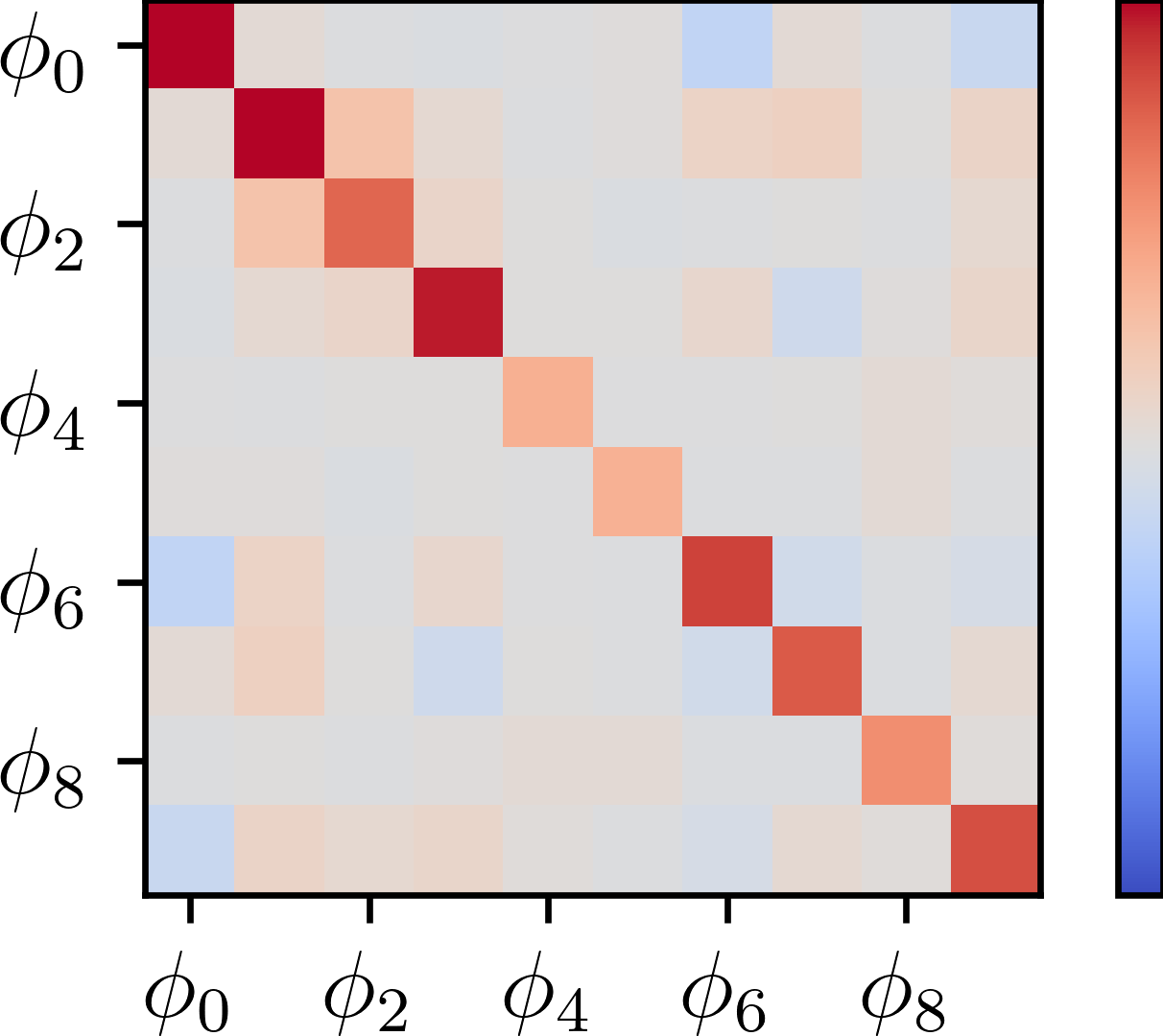}
\caption{\vspace{-0.2cm}RFF~\cite{DBLP:conf/nips/TancikSMFRSRBN20}}
\label{fig:suppl_kernel_stat_rff}
\end{subfigure}
\hfill
\begin{subfigure}[b]{0.31\textwidth}
\centering
\includegraphics[width=\textwidth]{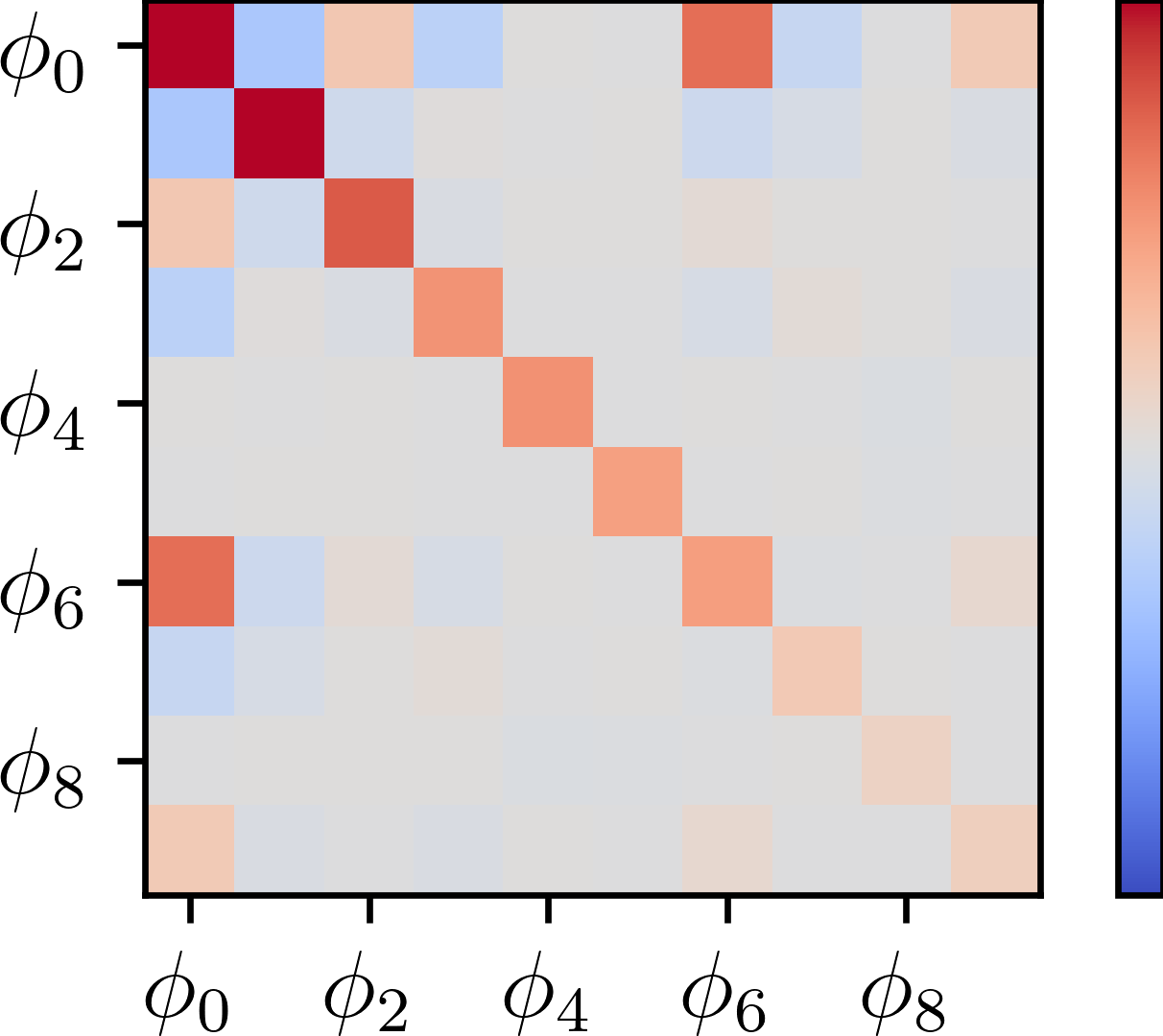}
\caption{\vspace{-0.2cm}Ours}
\label{fig:suppl_kernel_stat_intr}
\end{subfigure}
\caption{Non-stationarity of NTKs on a general 2-manifold. We investigate the non-stationarity of the neural tangent kernels (NTKs) on a general 2-manifold, namely the cat shown in \refmainpaper{Fig.~3}{fig:cat_3d_kernels}. Each small square in the image shows the coefficient $c_{ij}$ when projecting the kernel onto the LBO basis: ${c_{ij} = \int_{\man} \int_{\man} \ef_i(\pp) k(\pp, \qq) \ef_j(\qq) \dif\pp \dif \qq}$. The integral is approximated numerically as the area-weighted sum over the vertices. A stationary kernel as defined in \refmainpaper{Eqn.~4}{def:stationary} would have entries only along the main diagonal $c_{ij} = c_i \delta_{ij}$. The composed NTK is non-stationary for all features. We do not consider this a shortcoming of the proposed intrinsic neural fields because the NTK adapts to the intrinsic geometry of the underlying manifold as we show in  \refmainpaper{Fig.~3}{fig:cat_3d_kernels}.}
\label{fig:suppl_kernel_stat}
\end{figure}

\subsection{Theorem 1 on 1-Manifolds} \label{sec:theory_thm_man1}
The proof for \refmainpaper{Thm.~1}{thm:stationarity} only considered n-spheres. 
Here, we will give a short proof why it extends to general closed compact 1-manifolds. 
The proof depends only on properties of the spherical harmonics which are equivalent to the LBO eigenfunctions of closed 1-manifolds. 

Since the proof in the main paper is for n-spheres, it also holds for circles of arbitrary radius which are 1-spheres. 
Notice that all closed compact 1-manifolds $\mathcal{N}$ are isometric to the circle with radius equal to the circumference of $\mathcal{N}$. 
This is quite obvious if one considers that the geodesic distance between two points on a curve is simply the arc length between them, the geodesic distance on a closed 1-manifold is the minimum of both possible arc lengths. 
Therefore, the geodesic distances of any 1-manifolds with fixed circumference $r$ is invariant to its extrinsic embedding, and all 1-manifold with circumference $r$ are isometric to each other. 

The spherical harmonics are the eigenfunctions of the Laplace-Beltrami operator and those are invariant under isometries. Therefore, the spherical harmonics and all their properties transfer to general closed compact 1-manifolds and the proof still holds.

\subsection{Neural Tangent Kernel} \label{sec:theory_ntk}
In this section, we prove \autoref{thm:ntk_of_normed} that was used in the proof of \refmainpaper{Thm.~1}{thm:stationarity}. Additionally, we briefly discuss the positive definiteness of the NTK. As in the main paper, we consider the same setting as Jacot et al.~\cite{DBLP:conf/nips/JacotHG18}.

\begin{lemma} \label{thm:ntk_of_normed}
Let $\kntk: \R^n \times \R^n \to \R$ be the neural tangent kernel for a multilayer perceptron (MLP) with non-linearity $\sigma$. If the inputs are restricted to a \textbf{scaled} hypersphere $\| x \| = r$ then there holds
\begin{equation}
    \kntk(x, x') = \hntk(\inner{x, x'}) \,,
\end{equation}
for a scalar function $\hntk: \R \to \R$. 
\end{lemma}
\begin{proof}
For this proof we adopt the notation of Jacot et al.~\cite{DBLP:conf/nips/JacotHG18} to simplify following both papers simultaneously. We give a detailed proof, as this also gives good insight into the NTK. Let all requirements be equivalent to the ones used for \cite[Prop.~1]{DBLP:conf/nips/JacotHG18}. Jacot et al.\ show that the neural network Gaussian process (NNGP) has covariance $\Sigma^{(L)}$ defined recursively
\begin{align}
    \Sigma^{(1)} (x, x') &= \frac{1}{n_0} x\T x' + \beta^2 \label{eqn:ntk_lemma_sigma1} \\
    \Sigma^{(L+1)} (x, x') &=  \expec_{(u, v) \sim N(0, \Lambda^{(L)}(x, x'))} [\sigma(u) \sigma(v)] + \beta^2 \label{eqn:ntk_lemma_expec1} \\[1.1ex]
    \Lambda^{(L)}(x, x') &= \begin{pmatrix}
\Sigma^{(L)}(x, x) & \Sigma^{(L)}(x, x') \\
\Sigma^{(L)}(x', x) & \Sigma^{(L)}(x', x')
\end{pmatrix} \, ,
\end{align}
where $\sigma: \R \to \R$ is the non-linearity of the network and $\beta$ is related to the bias of the network. We use $u = f(x)$ and $v=f(x')$ instead of the Gaussian process notation used in \cite{DBLP:conf/nips/JacotHG18}. We prove by induction that $\Sigma^{(L)} (x, x')$ depends only on $x\T x'$, which also implies that $\Sigma^{(L)} (x, x)$ does not depend on $x$ because $x\T x = r^2$. For $L=1$ this follows directly from the definition. Assume now that for $L$ we have that $\Sigma^{(L)} (x, x')$ depends only on $x\T x'$. It directly follows that $\Lambda^{(L)}(x, x')$ and thus $\Sigma^{(L+1)} (x, x')$ also only depend on $x\T x'$, which is the induction step.

Given the NNGP kernel, the neural tangent kernel (NTK) is given by Theorem~1 from Jacot et al:
\begin{align}
    \Theta_{\infty}^{(1)} (x, x') &=  \Sigma^{(1)} (x, x') \\
    \Theta_{\infty}^{(L+1)} (x, x') &=  \Theta_{\infty}^{(L)} (x, x') \dot{\Sigma}^{(L+1)} (x, x') +  \Sigma^{(L+1)} (x, x') \\
    \dot{\Sigma}^{(L+1)} (x, x') &= \expec_{(u, v) \sim N(0, \Lambda^{(L)}(x, x'))} [\dot{\sigma}(u) \dot{\sigma}(v)] + \beta^2 \,, \label{eqn:ntk_lemma_expec2}
\end{align}
where $\dot{\sigma}$ is the derivative of the non-linearity. By a similar induction argument to above we obtain that $\Theta_{\infty}^{(L)} (x, x')$ only depends on $x\T x'$ and hence that $\Theta_{\infty}^{(L)} (x, x)$ does not depend on $x$. In the notation of our \autoref{thm:ntk_of_normed} this means that $\kntk$ only depends on  $\inner{x, x'}$ and can thus be written as $\hntk(\inner{x, x'})$ for a scalar function $\hntk: \R \to \R$. \qed
\end{proof}

\pseudoparagraph{Positive-definiteness of the NTK}
In the proof of \refmainpaper{Theorem~1}{thm:stationarity}, we used the fact that the NTK is positive definite as shown by \cite[Prop.~2]{DBLP:conf/nips/JacotHG18}. Their proposition is stated for $\| x \| = 1$ and the extension to $\| x \| = r$ requires only slight changes, which we will detail in the following. In the third step of Jacot et al.'s proof when doing the change of variables to arrive at their Eqn.~1 the following changes
\begin{equation}
    \expec_{(X, Y) \sim N(0, \tilde{\Sigma})} [\sigma(X) \sigma(Y)] + \beta^2
    = \hat{\mu}\left( \frac{n_0 \beta^2 + x\T x'}{n_0 \beta^2 + \textcolor{red}{r^2}}\right) + \beta^2 \,,
\end{equation}
where $\hat{\mu}: [-1, 1] \to \R$ is the dual in the sense of \cite[Lem.~2]{DBLP:conf/nips/JacotHG18} of the function $\mu: \R \to \R$ defined by $\mu(x) = \sigma\left( x \sqrt{\textcolor{red}{r^2}/n_0 + \beta^2} \right)$. Finally, in step 5 of their proof
\begin{equation}
\nu(x\T x') = \nu(\textcolor{red}{r^2} \rho) = \beta^2 + \sum_{i=0}^\infty a_i \, \left(\frac{n_0 \beta^2 + \textcolor{red}{r^2} \rho}{n_0 \beta^2 + \textcolor{red}{r^2}} \right)^i \,.
\end{equation}

\end{document}